\documentclass{article}
\usepackage[margin=25mm]{geometry}

\usepackage{amsmath,amssymb,amsfonts,amsthm,bbm,mathrsfs}
\allowdisplaybreaks

\newtheorem{theorem}{Theorem}[section]

\newtheorem{lemma}[theorem]{Lemma}
\newtheorem{proposition}[theorem]{Proposition}

\theoremstyle{definition}
\newtheorem{exampleinner}[theorem]{Example}

\theoremstyle{remark}
\newtheorem{remark}[theorem]{Remark}

\newtheorem{assumptionalt}[theorem]{Assumption}

\newenvironment{example}
  {\begin{exampleinner}\pushQED{\qed}}
  {\popQED\end{exampleinner}}

\newenvironment{assumptionp}[1]
  {\renewcommand{\theassumptionalt}{#1}\assumptionalt}
  {\endassumptionalt}

\usepackage{graphicx}
\usepackage{xcolor}
\usepackage{subcaption}
\usepackage[numbers,sort&compress]{natbib}

\usepackage[acronym]{glossaries}
\newacronym{mle}{MLE}{maximum likelihood estimation}
\newacronym{sde}{SDE}{stochastic differential equation}
\newacronym{ebm}{EBM}{energy-based model}
\newacronym{cd}{CD}{contrastive divergence}
\newacronym{pcd}{PCD}{persistent contrastive divergence}
\newacronym{srock}{S--ROCK}{stochastic orthogonal Runge--Kutta Chebyshev}
\newacronym{uit}{UiT}{uniform-in-time}
\newacronym{mcmc}{MCMC}{Markov chain Monte Carlo}
\newacronym{spcdem}{SPCDem}{Stable PCD (Euler--Maruyama)}
\newacronym{spcd}{SPCD}{Stable PCD}

\usepackage[colorlinks=true,urlcolor=blue,citecolor=red,linkcolor=blue]{hyperref}

\DeclareMathOperator{\Var}{Var}
\DeclareMathOperator{\Cov}{Cov}
\DeclareMathOperator{\Tr}{Tr}

\DeclareMathOperator*{\argmin}{arg\,min}
\DeclareMathOperator*{\argmax}{arg\,max}

\newcommand{\md}{\mathrm{d}}

\title{Uniform-in-time convergence bounds for Persistent Contrastive Divergence algorithms}
\author{Paul Felix Valsecchi Oliva, \"{O}. Deniz Akyildiz, and Andrew Duncan}

\begin{document}
\maketitle

\begin{abstract}
We propose a continuous-time formulation of a noisy \gls*{pcd}-like method for \gls*{mle} of unnormalised densities. Our approach couples parameter updates and sampling of the parametrised density in a multiscale system of \glspl*{sde}.   

From this formulation, we derive non-asymptotic bounds for weak
test-function errors between the resulting numerical schemes and the
\gls*{mle} point target. The error is decomposed into numerical
discretisation, slow--fast averaging, and finite-temperature
concentration terms. We also introduce an efficient implementation
based on explicit stabilized integrators and establish corresponding
long-time error estimates. This leads to a novel method for training
\glspl*{ebm} with quantitative error guarantees.
\end{abstract}

\glsresetall

\section{Introduction}
A key problem in machine learning is the estimation of probability densities from data. By learning the probability density we are able to perform a variety of downstream tasks, such as sampling, including conditional sampling, anomaly detection and simulation-based inference \cite{Du2024, Grathwohl, glaser2023}. \Glspl*{ebm}, introduced by \cite{Hinton1985}, have become ubiquitous in the world of machine learning \cite{Grathwohl, du2021improved, Du2019, liu2021learning, glaser2023}, as they can be flexibly trained with a wide variety of models, allowing them, in principle, to model any probability density. Indeed, they have been used in applications as varied as computer vision, natural language processing and reinforcement learning, demonstrating their robustness and expressiveness \cite{Du2019, Grathwohl, liu2021learning}. 

We will denote the \gls*{ebm} as  $p_\theta:\mathbb{R}^{d_x}\to\mathbb{R}_+$ for $\theta\in\mathbb{R}^{d_\theta}$, which takes the form
\begin{equation}\label{eq:ebm}
p_\theta(x) = \frac{e^{-E(\theta,x)}}{Z_\theta},
\end{equation}
where $Z_\theta = \int e^{-E(\theta,x)}\mathrm{d}x$ is the normalising constant (it is implied that any family of $\{E(\theta,\cdot)\}_{\theta\in\mathbb{R}^{d_\theta}}$ is chosen such that $Z_\theta$ is finite). Throughout the paper, we denote the densities $p_\theta$ and measures $p_\theta(\mathrm{d}x)$ (absolutely continuous with respect to the Lebesgue measure) with the same letters where the context is clear. The main task in training \glspl*{ebm} is to identify the \gls*{mle} solution
\begin{equation}\label{eq:mle}
\theta_\star \in \argmax_{\theta\in\mathbb{R}^{d_\theta}} \frac{1}{M}\sum_{j=1}^M \log p_\theta(y_j),
\end{equation}
given a set of i.i.d. observations $\{y_j\}_{j=1}^M \subset\mathbb{R}^{d_x}$. The difficulty in estimating parameter updates for such a model arises from the intractability of computing the gradients of the normalisation constant with respect to the parameter $\theta$, i.e., computing $\nabla_\theta \log Z_\theta$. 

To address this challenge, two widespread methods have emerged: \gls*{mle} via \gls*{mcmc}, leading to algorithms such as  \gls*{cd} \cite{Hinton2000TrainingPO} and \glsfirst*{pcd}, proposed by \cite{Tieleman_2008}, and score-matching \cite{Hyvarinen2005}. We will be particularly interested in the former and, in particular, \glsfirst*{pcd}. At iteration $k$, $\theta_k$ is fixed and $N$ particles $X_k^i$ are sampled approximately from $p_{\theta_k}$ to estimate the negative gradient of the log normaliser via,
\begin{equation*}
- \nabla_\theta \log Z_{\theta_k} = \int \nabla_{\theta}E(\theta_k, x) p_{\theta_k}(\mathrm{d}x) \approx N^{-1} \sum_{i=1}^N \nabla_\theta E(\theta_k, X^i_k).
\end{equation*}
These particles $X^i$ are sampled via \gls*{mcmc}, which interleaves the optimisation steps with sampling steps and introduces a bias in the estimate. To prevent bias accumulation, \cite{Hinton2000TrainingPO} resets the sampling procedure of the particles at each optimisation step, however, to reduce computational cost, only one step is performed in the \gls*{mcmc} chain. The bias arising from this approximation is dismissed by \cite{Hinton2000TrainingPO} as,
\begin{quote}
[it] is problematic to compute, but extensive simulations \dots show that it can safely
be ignored because it is small and it seldom opposes the resultant of [the computation.]
\end{quote}
Empirically, the number of \gls*{mcmc} steps seems to matter, as identified in \cite{Tieleman_2008}, where the \gls*{cd}-$i$ algorithm is investigated, with $i$ iterations of \gls*{mcmc}. Heuristically, this follows from the convergence of \gls*{mcmc} methods as $i$ is increased (for a full justification of this see \cite{Hinton2000TrainingPO} eq.~(5)). Thus, \cite{Tieleman_2008} proposes the \gls*{pcd} algorithm, which persists the particles from one $\theta$-update to the next, assuming that small changes of $\theta$ in Euclidean space will lead to small changes of $p_\theta$ in distribution.

Empirically, \gls*{cd} and \gls*{pcd}  can produce useful models \cite{Hinton2000TrainingPO, Sutskever2010, Tieleman_2008}, but their limiting
behaviour is subtle: finite-step CD/PCD updates are biased and, in general, need not coincide with the gradient of a fixed objective  \cite{Sutskever2010}. Despite their widespread use, there are, to our knowledge, no non-asymptotic bounds for these methods that directly quantify the error of finite-step PCD-type iterates relative to the MLE objective in the coupled optimisation--sampling regime considered here.

To study the behaviour of the \gls*{pcd} scheme we identify a continuous-time analogue with a multiscale system of Langevin diffusions, allowing us to leverage their rich properties (see e.g., \cite{Durmus2016HighdimensionalBI, Durmus2017, Durmus2018AnalysisOL, akyildiz2024multiscaleperspectivemaximummarginal}) to analyse and develop algorithms via numerical discretisations of the continuous-time system. In particular, we show that the Euler--Maruyama discretisation of this system yields a \gls*{pcd}-like scheme, in law (considering a small modification discussed below).

Specifically, we propose a two-time-scale system, where the particles targeting $p_\theta$ (hereafter referred to as $x$-particles) are ``accelerated'' by a time-rescaling of $1/\varepsilon$, which can be understood heuristically to correspond to running the interleaved sampling of the particles for longer. In particular, the averaging limit $\varepsilon\to 0$ can be shown to correspond to the desired gradient computation maximising the log-likelihood. To control the difference of these processes we will adapt recent developments in averaging literature, \cite{Crisan_Ottobre_2024}, which show \gls*{uit} weak error bounds on the moments of a two-time-scale \gls*{sde} and its averaged limit.

We summarise our main contributions as follows: 
\begin{itemize}
    \item We develop a multiscale perspective on the \gls*{mle} training problem of \glspl*{ebm} by providing a two-time-scale Langevin diffusion. In the limit $\varepsilon\to0$, its slow component is approximated by a Langevin diffusion with invariant density proportional to \(e^{-NV}\); this invariant measure concentrates on the empirical \gls*{mle} as \(N\to\infty\). We show that this framework can be used to analyse \gls*{pcd}-like algorithms and to develop new ones.
    \item We provide numerical discretisations of the multiscale Langevin diffusion as practical algorithms for training \glspl*{ebm}. The Euler--Maruyama discretisation yields a noisy \gls*{pcd}-like scheme, with the additional parameter noise prescribed by the continuous-time model. We provide a uniform-in-time discretisation error analysis for this scheme.
    \item To further demonstrate the utility of our framework and motivated by the potential instability of the Euler--Maruyama discretisation, we propose a new class of numerical integrators based on \gls*{srock} methods, which are known to be stable for stiff \glspl*{sde}. We obtain uniform-in-time bounds for these stable \gls*{pcd}-like schemes under the stated semigroup and numerical Lyapunov assumptions.
\end{itemize} 

The paper is structured as follows. Section~\ref{sec:background}
introduces the problem and the multiscale formulation.
Section~\ref{sec:main} outlines the main results and assumptions.
Section~\ref{sec:poisson} develops the frozen Poisson corrector, and
Sec.~\ref{sec:averaged} studies the averaged Langevin process, its
finite-\(N\) invariant measure \(\pi_N^{\mathrm{av}}\), and its
concentration at the \gls*{mle} \(\theta_\star\).
Section~\ref{sec:error} combines these ingredients to obtain the
slow--fast weak error. Section~\ref{sec:numerical} establishes
uniform-in-time and invariant-measure bounds for the numerical
integrators, and Sec.~\ref{sec:experiments} presents the numerical
experiments.

\section{Background and preliminary results}\label{sec:background}
Let us begin by recalling the problem setting. Suppose $\{y_i\}_{i=1}^M\subset\mathbb{R}^{d_x}$ are i.i.d samples from $p_{\text{data}}$, an unknown data distribution on $\mathbb{R}^{d_x}$. We define the population \gls*{mle} solution for our \gls*{ebm} $p_\theta: \mathbb{R}^{d_x} \times \mathbb{R}^{d_\theta} \to \mathbb{R}$,
\begin{align*}
\theta_\star^{\text{pop}} \in \argmax_{\theta \in \mathbb{R}^{d_\theta}} \mathbb{E}_{p_{\text{data}}}\left[\log p_\theta(Y) \right].
\end{align*}
Let $p^M_{\text{data}} = (1/M) \sum_{j=1}^M \delta_{y_j}$ be the empirical measure of the data, where $\delta_y$ is the Dirac measure at $y$. As we do not have access to $p_{\text{data}}$, we use the empirical measure $p^M_{\text{data}}$ to approximate the population \gls*{mle} objective, leading to the empirical approximation:
\begin{align}\label{eq:mle_empirical}
\theta_\star \in \argmax_{\theta \in \mathbb{R}^{d_\theta}} \mathbb{E}_{p^M_{\text{data}}}\left[\log p_\theta(Y) \right] = \argmax_{\theta \in \mathbb{R}^{d_\theta}} \frac{1}{M} \sum_{j=1}^M \log p_\theta(y_j).
\end{align}
Our aim in this paper is to develop methods to identify $\theta_\star$, i.e., an empirical maximiser of the \gls*{mle} objective, which is an approximation of the population maximiser $\theta_\star^{\text{pop}}$.

To proceed, we define the function $V:\mathbb{R}^{d_\theta} \to \mathbb{R}$ as the negative empirical log-likelihood
\begin{equation}\label{eq:loglike}
V(\theta) = -\frac{1}{M}\sum_{j=1}^M \log p_{\theta}(y_j) = \frac{1}{M}\sum_{j=1}^M E(\theta, y_j) + \log Z_{\theta}. 
\end{equation}
Thus the empirical \gls*{mle} problem is equivalently the global
minimisation problem
\[
\theta_\star\in
\argmin_{\theta\in\mathbb R^{d_\theta}}V(\theta).
\]
We use this minimisation formulation throughout the remainder of the
paper. Differentiating the potential gives
\begin{equation}\label{eq:mlegrad}
\nabla_\theta V(\theta) = \frac{1}{M}\sum_{j=1}^M \nabla_\theta E(\theta, y_j) - \int \nabla_\theta E(\theta, x) p_\theta(\mathrm{d}x),
\end{equation}
as $\nabla_\theta \log Z_\theta = \nabla_\theta \log \int e^{-E(\theta, x)}\mathrm{d}x = \int -\nabla_\theta E(\theta, x) p_\theta(\mathrm{d}x)$. The identity \eqref{eq:mlegrad} is justified below by assumptions \ref{ass:dissx} and \ref{ass:poly}: \ref{ass:dissx} ensures sufficiently fast decay of \(e^{-E(\theta,x)}\) in $x$, while \ref{ass:poly} ensures that \(\nabla_\theta E(\theta,x)\) has at most polynomial growth, so dominated convergence permits differentiating \(Z_\theta\) under the integral.

As mentioned above, \gls*{cd} methods aim at implementing a gradient descent procedure which can be written as
\begin{align}\label{eq:gradient_descent_pcd}
\theta_{k+1} = \theta_k - \delta \nabla_\theta V(\theta_k),
\end{align}
for $\delta> 0$ and $k \geq 0$. However, as can be seen from \eqref{eq:mlegrad}, the second term of this gradient is often intractable, as it takes the form of an integral with respect to $p_\theta$, which is often approximated via \gls*{mcmc}. We recall that classical \gls*{pcd} methods run particle-based Langevin dynamics on $p_\theta$ to estimate it (persistent across iterations, meaning that the dynamics are not restarted when $\theta$ is updated). More precisely, this results in a sampling scheme:
\begin{align*}
X_{k+1}^i = X_k^i - h \nabla_x E(\theta_k, X_k^i) + \sqrt{2h} \mathcal{N}(0, I)
\end{align*}
for $h>0$ and $i = 1, \ldots, N$. The particle set $\{X_k^i\}_{i=1}^N$ is then used to approximate the second term of the gradient in \eqref{eq:mlegrad}. In practice, the step-sizes $\delta$ and $h$ are tuned differently---which makes it non-trivial to develop a continuous-time framework.

To tackle this problem, we introduce a multiscale \gls*{sde}, the continuous-time limit of the \gls*{pcd} algorithm
\begin{equation}
\begin{aligned}\label{eq:basesde}
    \mathrm{d} \theta^\varepsilon_t &= \frac{1}{N} \sum_{i=1}^N \left(\nabla_\theta E(\theta^\varepsilon_t, X_t^{i, \varepsilon}) - \frac{1}{M} \sum_{j=1}^M \nabla_\theta E(\theta^\varepsilon_t, y_j) \right) \mathrm{d}t + \sqrt{\frac{2}{N}} \mathrm{d} W^0_t, \\
    \mathrm{d} X_t^{i, \varepsilon} &= - \frac{1}{\varepsilon} \nabla_x E(\theta^\varepsilon_t, X_t^{i, \varepsilon}) \mathrm{d}t + \sqrt{\frac{2}{\varepsilon}} \mathrm{d} W^i_t,\qquad i\in\{1,\dots, N\},
\end{aligned}
\end{equation}
where $(W_t^0)_{t\geq 0}$ and $(W_t^i)_{t\geq 0}$ for $i = 1, \ldots, N$ are independent Wiener processes in $\mathbb{R}^{d_\theta}$ and $\mathbb{R}^{d_x}$ respectively. We note here that the particles are assumed to be initialised conditionally independent given $\theta_0$.

\begin{remark}\label{rem:1}
We remark two important aspects of the \gls*{sde} introduced in
\eqref{eq:basesde}. First, the parameter \(\varepsilon\) models the
time-scale separation induced by the different parameter and particle
step sizes used in practice; specifically, we set
\(\delta=h\varepsilon\). This makes the system a continuous-time
generalisation of \gls*{pcd} in which the particles evolve on the
faster scale and connects the model naturally to averaging.

Second, adding noise to the \(\theta\)-dynamics gives the averaged
process the explicit Gibbs invariant measure
\(e^{-NV(\theta)}\,\mathrm d\theta\). This measure concentrates on
global minimisers under the conditions of
Theorem~\ref{thm:statmeasure}, including for non-convex objectives.
The uniform-in-time averaging result additionally requires the
averaged-semigroup derivative estimate introduced below; that
condition does not impose convexity, but must be verified separately.
The inverse temperature could be chosen independently of the particle
number. We set it equal to \(N\), which corresponds to the natural
order-\(1/N\) time-rescaling in the parameter dynamics.
\end{remark}

For notational convenience, we write \eqref{eq:basesde} in a more compact form to derive our results. We define the function ${E}_N:\mathbb{R}^{d_\theta}\times\mathbb{R}^{Nd_x}\to\mathbb{R}$ by
\begin{equation*}
{E}_N(\theta, z) = \sum_{i=1}^N \left(E(\theta, x^i) - \frac{1}{M}\sum_{j=1}^M E(\theta, y_j)\right),
\end{equation*}
where $z=(x^1, \dots, x^N)^\top$. Using this function, we can rewrite the \gls*{sde} in a more compact form as
\begin{equation}\label{eq:sde}
\begin{aligned}
\mathrm{d}\theta^\varepsilon_t &= \frac{1}{N}\nabla_\theta {E}_N(\theta^\varepsilon_t, Z^\varepsilon_t)\mathrm{d}t + \sqrt{\frac{2}{N}} \mathrm{d}W_t^\theta,\\
\mathrm{d}Z^\varepsilon_t &= -\frac{1}{\varepsilon}\nabla_z {E}_N(\theta^\varepsilon_t, Z^\varepsilon_t)\mathrm{d}t + \sqrt{\frac{2}{\varepsilon}}\mathrm{d}W_t^z,
\end{aligned}
\end{equation}
where $Z^\varepsilon_t = (X_t^{1, \varepsilon}, \ldots, X_t^{N, \varepsilon})\in\mathbb{R}^{Nd_x}$ and $W_t^\theta$ and $W_t^z$ are $\mathbb{R}^{d_\theta}$ and $\mathbb{R}^{Nd_x}$ dimensional independent Brownian motions. We generally refer to the SDE in \eqref{eq:sde} as the \textit{multiscale system} of SDEs, as it is a slow--fast system with time-scale separation controlled by $\varepsilon$.

We will be interested in $0< \varepsilon\ll 1$, as this is the range analogous to those shown in \cite{Tieleman_2008, Sutskever2010} to improve performance. Indeed, we will use the recent averaging results (see, e.g., \cite{Crisan_Ottobre_2024,schuh2024conditionsuniformtimeconvergence}) to show that, in the limit $\varepsilon\to 0$ the dynamics of the $\theta$-marginal behave according to the averaged dynamics
\begin{equation}\label{eq:averagedtheta}
\mathrm{d} {\theta}^{\textnormal{av}}_t = \frac{1}{N} \int \nabla_\theta E_N(\theta^{\textnormal{av}}_t, z) p_{{\theta}^{\textnormal{av}}_t}^{\otimes N}(\mathrm{d}z)\mathrm{d}t + \sqrt{\frac{2}{N}}\mathrm{d}W^\theta_t,
\end{equation}
where $(W_t^\theta)_{t\geq 0}$ is a $d_\theta$-dimensional Wiener process and $p_\theta^{\otimes N}(\mathrm{d}z) = \prod_{i=1}^N p_\theta(\mathrm{d}x^i)$ is the product measure of $N$ independent copies of $p_\theta$. Written in another way, this results in an averaged dynamics that concentrates around minimisers of $V$, which can be written as
\begin{align}\label{eq:global_langevin_mle}
\mathrm{d} \theta^{\textnormal{av}}_t &= - \nabla_\theta V(\theta^{\textnormal{av}}_t) \mathrm{d}t + \sqrt{\frac{2}{N}} \mathrm{d}W^\theta_t.
\end{align}
It is well-known that, for large $N$, the Langevin dynamics of type \eqref{eq:global_langevin_mle} has an invariant measure of the form
\begin{align}\label{eq:averaged-invariant-measure}
\pi_N^{\mathrm{av}}(\md \theta) \propto e^{-NV(\theta)} \md \theta
\end{align}
which concentrates around global minimisers of $V$ under weak conditions \cite{Hwang,raginsky2017,zhang2023nonasymptotic}, as verified in Theorem~\ref{thm:statmeasure}. We denote this finite-\(N\) invariant measure by \(\pi_N^{\mathrm{av}}\). Thus the finite-\(N\) averaged dynamics is a sampler, while the additional limit \(N\to\infty\) recovers the \gls*{mle} optimisation target. The goals of this paper are to analyse the properties of the multiscale system that gives rise to this averaged dynamics and propose numerical integrators for it.

\begin{remark}\label{rem:2}
Note that, unlike most of the averaging literature, this work is concerned with using the slow--fast ($\varepsilon > 0$) regime to estimate the averaged ($\varepsilon \to 0$) regime, as opposed to the other way round (as one may see in \cite{Pavliotis2008, Pardoux_Veretennikov_2001, Pardoux_Veretennikov_2003}). In particular, in this context, it is critical to \glsfirst*{uit} (i.e. finite bounds for all $t$) moment bounds between the slow--fast and averaged regimes (as identified in \cite{Crisan_Ottobre_2024, schuh2024conditionsuniformtimeconvergence}), to ensure that longer simulation runtimes---required to improve the sampling accuracy of the Langevin diffusions---lead to better bounds. The key difficulty is being able to identify bounds proportional to the inverse of the time-rescaling factor $1/\varepsilon$, which are also \gls*{uit}, requiring strong assumptions on the behaviour of the drifts, as identified in \cite{schuh2024conditionsuniformtimeconvergence}. In this paper we obtain similar results to \cite{Crisan_Ottobre_2024}, using slightly different assumptions, which are more suited to our problem and common in the sampling literature. For another example of a work in a similar direction, see \cite{akyildiz2024multiscaleperspectivemaximummarginal}, however, note that this paper addresses a different problem.
\end{remark}

To motivate this approach we will show how in a simple example these dynamics converge to the desired \gls*{mle} target and how the limits $\varepsilon \to 0$ and $N\to \infty$ lead to some desirable properties for our solution. For this, we will consider a very simple tractable case: a Gaussian model, where the mean is parametrised.
\begin{example}\label{example1}
Consider the Gaussian case, $E(\theta, x)=\frac{1}{2}(\theta-x)^2$. We will show convergence to the \gls*{mle} for the case $d_\theta =d_x=1$, but the arguments easily extend to $d_\theta\in\mathbb{N}$ (recall that for this model $d_\theta=d_x$). We introduce $\mathbbm{1}_{N\times N}$, to denote a $N\times N$ matrix with all entries equal to one and we take $\mathbbm{1}_N=\mathbbm{1}_{N\times 1}$.

In this case, \eqref{eq:sde}, corresponds to 
\begin{align*}
\mathrm{d}\begin{pmatrix}
    \theta_t^\varepsilon\\ Z_t^\varepsilon
\end{pmatrix} = -A_\varepsilon \begin{pmatrix}
    \theta_t^\varepsilon\\ Z_t^\varepsilon
\end{pmatrix}\mathrm{d}t + c\mathrm{d}t& + \sigma_\varepsilon \mathrm{d}W_t,\\
A_\varepsilon = \begin{pmatrix}
    0 & N^{-1}\mathbbm{1}^\top_N\\
    -\frac{1}{\varepsilon} \mathbbm{1}_N & \frac{1}{\varepsilon} I_N
\end{pmatrix}, \quad c = \frac{1}{M} \sum_{j=1}^M\begin{pmatrix}
     y_j\\
    0\mathbbm{1}_N
\end{pmatrix} \quad& \text{and}\quad \sigma_\varepsilon = \begin{pmatrix}
    \sqrt{\frac{2}N} & 0 \\
    0 & \sqrt{\frac{2}{\varepsilon}}I_N
\end{pmatrix},
\end{align*}
for a Wiener process $W_t$ in $\mathbb{R}^{N+1}$.

Let us now denote the first moment $\mathbb{E}[(\theta_t^\varepsilon, Z^\varepsilon_t)^\top]$ as $m_t$.  Next observe the equality,
\begin{align*}
\frac{\mathrm{d}}{\mathrm{d}t} m_t = -A_\varepsilon m_t + c.
\end{align*}
From this one can observe that the first moment of the invariant measure of this system is given by
\[m_\infty = \lim_{t\to\infty} m_t = A_\varepsilon^{-1}c = \begin{pmatrix}
    1 & -\varepsilon/N \mathbbm{1}_N^\top\\
    \mathbbm{1}_N & -\varepsilon/N \mathbbm{1}_{N\times N} + \varepsilon I_N
\end{pmatrix} c = \frac{1}{M} \sum_{j=1}^M y_j \mathbbm{1}_{N+1}.\]
Thus, the $\theta$-marginal corresponds to the \gls*{mle} solution $\theta_\star$. Indeed, note that
\begin{align*}
    \theta_\star = \argmax_{\theta\in\mathbb{R}^{d_\theta}} \frac{1}{M} \sum_{j=1}^M \log p_\theta(y_j)
    =\argmin_{\theta\in\mathbb{R}^{d_\theta}} \frac{1}{M} \sum_{j=1}^M \frac{1}{2}(\theta-y_j)^2 = \frac{1}{M} \sum_{j=1}^M y_j.
\end{align*} 
Further, the $Z$-marginals of $m_t$ correspond to the particle means
parametrised by $\theta_\star$. Hence, in the Gaussian case, the system
converges to an invariant measure centred on the \gls*{mle}. Its
steady-state covariance is
\begin{align*}
\Sigma_\infty
&= \lim_{t\to\infty}\Sigma_t\\
&= \lim_{t\to\infty}\Big(
\mathbb{E}[S_tS_t^\top]
-\mathbb{E}[S_t]\mathbb{E}[S_t]^\top
\Big),
\qquad
S_t=(\theta_t^\varepsilon,Z_t^\varepsilon),
\end{align*}
satisfying
\[A_\varepsilon \Sigma_\infty + \Sigma_\infty A_\varepsilon^\top = \sigma_\varepsilon \sigma_\varepsilon^\top,\]
which follows from considering the time derivative of $\Sigma_t$ and observing that this is the solution to $\mathrm{d}\Sigma_t/\mathrm{d}t = 0$ \citep{anderson2011statistical}. This yields,
\[\Sigma_\infty = \begin{pmatrix}
    \frac{2\varepsilon+1}N & \frac{1}N \mathbbm{1}_N^\top\\
    \frac{1}N \mathbbm{1}_N & \frac{1}N \mathbbm{1}_{N\times N} + I_N
\end{pmatrix}.\]
Thus the invariant measure is a Gaussian measure with mean and covariance given above \citep{anderson2011statistical}. It can be seen that the system exhibits desirable properties in this case: as $\varepsilon\to 0$, the noise of the $x$-marginal remains unchanged and the $\theta$-marginal converges to an invariant measure with variance $1/N$; when we also let $N\to\infty$, we can observe that the invariant measure of the $\theta$-marginal concentrates around the \gls*{mle}. Indeed, $(\Sigma_\infty)_{11}$ is monotonically increasing in $\varepsilon$ and monotonically decreasing in $N$.
\end{example}

\begin{remark}[Scope of Example~\ref{example1}]
The joint drift in Example~\ref{example1} is affine:
\[
f_\varepsilon(s)=-A_\varepsilon(s-s_\star),
\]
and \(A_\varepsilon\) is positive stable, or equivalently,
\(-A_\varepsilon\) is Hurwitz. Hence, for every fixed
\(\varepsilon>0\), there is a matrix \(M_\varepsilon\succ0\)
satisfying
\[
A_\varepsilon^\top M_\varepsilon
+M_\varepsilon A_\varepsilon=I.
\]
The associated quadratic Lyapunov function gives the moment and
ergodicity properties in Assumption~\ref{ass:jointmom}; the explicit
Gaussian solution above shows that the required moments remain
bounded for \(0<\varepsilon\leq\varepsilon_0\). Moreover,
derivatives of the Gaussian semigroup contain tensor powers of
\(e^{-A_\varepsilon t}\), so
Assumption~\ref{ass:jointSES} holds with constants depending on
\((N,\varepsilon)\). Finally, for either Euler--Maruyama or S--ROCK,
the corresponding linear recursion has a quadratic discrete
Lyapunov function whenever its deterministic stability matrix is
Schur stable. Thus Assumption~\ref{ass:numstability} holds under the
method-specific step-size restriction. Example~\ref{example1} is
therefore covered by the joint numerical results under the
semigroup and numerical stability assumptions.
\end{remark}

\subsection{Notation}
Denote by $\mathscr{P}_n(\mathbb{R}^d)$, for $d,n\geq1$, the probability measures on $(\mathbb{R}^d,\mathscr{B}(\mathbb{R}^d))$ with finite $n$th moment. We reserve $|\cdot|$ for the absolute value of a scalar. On every finite-dimensional Euclidean space, $\langle\cdot,\cdot\rangle$ denotes the Euclidean inner product and $\|\cdot\|$ its associated norm. For matrices and higher-order tensors, $\|\cdot\|_F$ denotes the Frobenius norm; in particular,
\[
\|A\|_F^2=\Tr(AA^\top).
\]
Thus, unadorned double bars always denote a pointwise finite-dimensional norm, whereas the weighted function norms introduced below always carry an explicit function-space subscript.

For $a,b\geq0$, define
\[
\mathcal W_{a,b}(\theta,z)
:=
1+\|\theta\|^a+\|z\|^b.
\]
Calligraphic \(\mathcal W\) is reserved for polynomial weights,
thereby distinguishing these weights from Brownian motions \(W_t\)
and the Wasserstein distance \(W_2\).
For a vector-, matrix-, or tensor-valued function $h$ on
$\mathbb R^{d_\theta}\times\mathbb R^{Nd_x}$, define the weighted
uniform norm
\[
\|h\|_{0;a,b}
:=
\sup_{\theta,z}
\frac{\|h(\theta,z)\|}
{\mathcal W_{a,b}(\theta,z)}.
\]
For $k\in\mathbb N_0$, write
\[
\|h\|_{k;a,b}
:=
\sum_{j=0}^k\|\nabla^j h\|_{0;a,b},
\]
and set
\[
C_{a,b}^k
:=
\left\{
h\in C^k:
\|h\|_{k;a,b}<\infty
\right\}.
\]
For a function $\psi:\mathbb R^d\to\mathbb R$, we use the
corresponding one-variable notation
\[
\|\psi\|_{0;a}
:=
\sup_{x\in\mathbb R^d}
\frac{|\psi(x)|}{1+\|x\|^a},
\qquad
\|\psi\|_{k;a}
:=
\sum_{j=0}^k\|\nabla^j\psi\|_{0;a},
\]
and set
\[
C_a^k(\mathbb R^d)
:=
\{\psi\in C^k(\mathbb R^d):\|\psi\|_{k;a}<\infty\}.
\]
When the underlying domains are clear, we continue to write
$C_{m_\theta,m_x}^k$ and $C_{m_\theta}^k$ for these spaces.
For a state \(s\in\mathbb R^d\) and \(q\geq0\), we also write
\[
\mathcal W_q(s):=(1+\|s\|^2)^{q/2}.
\]
When \(\theta\) is embedded into the joint state space, these weights
are compatible in the sense that
\[
\mathcal W_q(\theta)
\leq
2^{q/2}\mathcal W_{q,0}(\theta,z),
\qquad q\geq0.
\]
The weights are closed under finite products after increasing their
orders. In particular,
\begin{equation}\label{eq:weight-product}
\mathcal W_{a,b}(\theta,z)\mathcal W_{c,d}(\theta,z)
\leq
9\mathcal W_{A,B}(\theta,z),
\qquad
A=2\max\{a,c\},
\quad
B=2\max\{b,d\}.
\end{equation}
Define the Wasserstein-2 metric as
\begin{equation}
W_2(\pi,\nu) = \inf_{\Gamma\in \mathbf{T}(\pi, \nu)} \left(\int \|x-y\|^2\mathrm{d}\Gamma(x,y)\right)^\frac{1}{2},
\end{equation}
where $\mathbf{T}(\pi,\nu)$ denotes the set of couplings on $\mathbb{R}^{d}\times\mathbb{R}^{d}$, i.e.\ probability measures on $\mathbb{R}^d\times\mathbb{R}^d$ with first marginal $\pi\in\mathscr{P}_2(\mathbb{R}^d)$ and second marginal $\nu\in\mathscr{P}_2(\mathbb{R}^d)$.

We now define the semigroup notation used below. Let \(X_t\) be a
Markov process on \(\mathbb R^d\). Its transition semigroup is
\[
(\mathcal P_t\phi)(x)
:=
\mathbb E_x[\phi(X_t)],
\]
initially for bounded measurable \(\phi\), and extended to the
weighted function spaces above whenever the corresponding moments
are finite. On functions in the domain of the generator
\(\mathcal G\),
\[
\partial_t\mathcal P_t\phi
=
\mathcal G\mathcal P_t\phi,
\qquad
\mathcal P_0\phi=\phi.
\]
If \(\phi\) also depends on an auxiliary parameter not evolved by
the process, then \(\mathcal P_t\) acts only on the state variable
and leaves that parameter fixed. This is the convention used, for
example, by the frozen semigroup \(\mathcal P_t^z\), which evolves
\(z\) while treating \(\theta\) as a parameter.

The dual action on probability measures is denoted by
\(\mathcal P_t^*\) and is defined by
\[
\int\phi\,\mathrm d(\mathcal P_t^*\mu)
=
\int\mathcal P_t\phi\,\mathrm d\mu.
\]
In particular,
\[
\mathcal P_t\phi(x)
=
\int\phi(z)\,\mathcal P_t^*\delta_x(\mathrm dz).
\]
When the relevant densities are sufficiently regular,
\(\mathcal P_t^*\mu\) solves the corresponding Fokker--Planck
equation with the formal adjoint \(\mathcal G^*\).

\textbf{Notation for multiscale, frozen, and averaged  dynamics.} We denote the infinitesimal generator of the multiscale dynamics in \eqref{eq:sde} as
\begin{align}
\mathcal{G}_\varepsilon &= \mathcal{G}_\theta + \frac{1}{\varepsilon}\mathcal{G}_z
\end{align}
where
\begin{align}
\mathcal{G}_\theta = \frac{1}{N}\langle\nabla_\theta E_N, \nabla_\theta\rangle + \frac{1}{N} \Delta_\theta, \quad & \quad \mathcal{G}_z = -\langle\nabla_z E_N, \nabla_z\rangle + \Delta_z.
\end{align}
The sign in \(\mathcal G_\theta\) is consistent with gradient descent on \(V\), since averaging \((1/N)\nabla_\theta E_N \) over
\(p_\theta^{\otimes N}\) gives \(-\nabla_\theta V\).
Note that all these generators are understood to act on functions over $\mathbb{R}^{d_\theta}\times \mathbb{R}^{Nd_x}$, where the dimensions not accounted for by the partial gradient operators are understood to be mapped to zero. The Markov semigroups associated with $\mathcal{G}_\varepsilon$ and $\mathcal{G}_z$ are $\mathcal{P}_t^\varepsilon$ and $\mathcal{P}^z_t$ respectively. This latter semigroup is associated with the ``frozen'' \gls*{sde}
\begin{equation}\label{eq:frozensde}
\begin{aligned}
\mathrm{d} {Z}_t &= -\nabla_z E_N(\theta, {Z}_t) \mathrm{d}t + \sqrt{2}\mathrm{d}W_t^1,
\end{aligned}
\end{equation}
where $\theta$ is fixed. Note that this \gls*{sde} leaves the distribution $p_\theta^{\otimes N}$ invariant. The importance of this system and its associated Markov semigroup, $\mathcal{P}^z_t$, is made explicit in Sec.~\ref{sec:main}.

We also introduce the generator for the individual particles, $\mathcal{G}_x = -\langle \nabla_x E, \nabla_x \rangle + \Delta_x$, as well as the generator of the averaged dynamics,
\begin{equation}
\mathcal{G}_{\textnormal{av}} = \frac{1}{N} \left\langle \int \nabla_\theta E_Np_\theta^{\otimes N} (\mathrm{d}z), \nabla_\theta\right\rangle + \frac{1}{N}\Delta_\theta,
\end{equation}
with associated semigroup $\mathcal{P}^{\textnormal{av}}_t$.

\section{Main results}\label{sec:main}
The goal of this paper is to quantify the error between numerical
\gls*{pcd}-like schemes and the empirical \gls*{mle} target. At the
invariant-measure level, the argument follows the decomposition
\[
\widehat\pi_{\mathsf M}^{\varepsilon,\delta}
\longrightarrow
\pi^\varepsilon
\longrightarrow
\pi_N^{\mathrm{av}}
\longrightarrow
\delta_{\theta_\star},
\]
corresponding respectively to numerical discretisation, slow--fast
averaging, and finite-temperature concentration errors. This
decomposition applies to a range of numerical schemes once their
local weak accuracy and long-time stability have been established;
below, \(\mathsf M\) denotes either Euler--Maruyama or S--ROCK.

To study the dynamics of the multiscale system \eqref{eq:sde}, a
common approach is to use the Poisson equation of the fast dynamics
(see \cite{Pardoux_Veretennikov_2001} and
\cite{Pardoux_Veretennikov_2003} for a more general treatment).
This approach yields \gls*{uit} weak control of the slow marginal
relative to the averaged process \eqref{eq:averagedtheta}
\cite{Crisan_Ottobre_2024,
akyildiz2024multiscaleperspectivemaximummarginal}. Broadly speaking,
the result in \cite{Crisan_Ottobre_2024} is
\begin{equation*}
\|\mathcal{P}_t^\varepsilon \phi-\mathcal{P}^{\textnormal{av}}_t \phi \|\lesssim \varepsilon ,
\end{equation*}
over all $t>0$, for a suitable class of functions $\phi$. We note here, that typically results of this kind are established with time-dependent constants, deteriorating with growing $t$ (see Remark~\ref{rem:2}).

We consider the problem of $\mathcal{P}_t^\varepsilon \phi - \mathcal{P}^{\textnormal{av}}_t\phi$, with $\mathcal{P}_t^\varepsilon \phi$ in terms of its Taylor expansion with respect to $\varepsilon$ (i.e., $\mathcal{P}_t^\varepsilon \phi = \phi_t^0 + \varepsilon \phi_t^1+\dots$). One can then observe that as $\mathcal{G}_\varepsilon (\mathcal{P}^\varepsilon_t \phi)-\partial_t (\mathcal{P}^\varepsilon_t\phi) = 0$, the expansion of this expression in terms of orders of $\varepsilon$ is given as,
\begin{align*}
O(\varepsilon^{-1}):\quad &\mathcal{G}_z \phi_t^0=0,\\
O(\varepsilon^0):\quad & \partial_t \phi_t^0 - \mathcal{G}_\theta \phi_t^0 = \mathcal{G}_z \phi_t^1,\\
\vdots
\end{align*}
From the $\varepsilon^{-1}$ order term we note that $\phi_t^0$ is constant in $z$ and so integrating both sides of the order $\varepsilon^0$ terms against $p_\theta^{\otimes N}$,
\begin{equation*}
\partial_t \phi_t^0(\theta) - \mathcal{G}_{\textnormal{av}} \phi_t^0(\theta) = 0.
\end{equation*}
From this follows that,
\begin{align*}
\mathcal{P}_t^\varepsilon \phi - \mathcal{P}^{\textnormal{av}}_t\phi &= \varepsilon \phi_t^1 + \dots,\\
\mathcal{G}_z \phi_t^1 &= [(\mathcal{G}_{\textnormal{av}} -\mathcal{G}_\theta) \mathcal{P}^{\textnormal{av}}_t \phi](\theta).
\end{align*}
The solution to $\phi_t^1$ can be decomposed into the inner product of the solution to the Poisson equation, i.e., $\Phi:\mathbb{R}^{d_\theta}\times\mathbb{R}^{Nd_x}\to\mathbb{R}^{d_\theta}$ solves 
\begin{equation}\label{eq:poissoneq}
(\mathcal{G}_z \Phi)(\theta, z) = \frac{1}{N}\left(\nabla_\theta E_N(\theta, z) - \int\nabla_\theta E_N(\theta, w) p_\theta^{\otimes N} (\mathrm{d}w)\right),
\end{equation}
and $\nabla_\theta\mathcal{P}^{\textnormal{av}}_t\phi$. To control this we will verify that there exist polynomial orders $a,b\geq0$ such that the solution of the Poisson equation satisfies
\begin{equation*}
    \|\Phi\|_{0;a,b},
    \quad
    \|\nabla_\theta\Phi\|_{0;a,b},
    \quad
    \|\nabla_\theta^2\Phi\|_{0;a,b}
    <\infty.
\end{equation*}
This result is obtained by using Lemma~\ref{lem:poisson-solution}, which identifies $\Phi$ with a time integral of the ``frozen'' semi-group, allowing us to transfer estimates for $\mathcal{P}_t^z$ to the solution to the Poisson Equation. In particular we will show that the unique centred solution is given as,
\begin{equation*}
\Phi(\theta, z) = -\frac{1}{N} \int_0^t \left(\mathcal{P}_s^z \nabla_\theta E_N(\theta,z) - \mathcal{P}_\infty^z \nabla_\theta E_N(\theta)\right)\mathrm{d}s,
\end{equation*}
where we recall that,
\begin{equation*}
\mathcal{P}_\infty^z \phi(\theta) = \int \phi(\theta, z) p_\theta^{\otimes N}(\mathrm{d}z).
\end{equation*}

In a heuristic sense, we will have decomposed the corrector $\phi_t^1$ into the change of $\theta$ due to the discrepancy between the $x$-dynamics and their ``averaged'' behaviour, $\Phi$, and the change in law this would have on the averaged dynamics, $\nabla_\theta \mathcal{P}^{\textnormal{av}}_t\phi$.

\medskip
\noindent\textbf{Informal main result.}
Under the regularity, mixing, and moment assumptions stated below,
the slow component of the joint diffusion remains \(O(\varepsilon)\)
close to its averaged limit, uniformly over time.  More precisely,
there exist polynomial orders \(q_\theta,q_z\) and a constant
\(C_{\mathrm{av}}\), independent of \(t\geq0\) and
\(\varepsilon\in(0,\varepsilon_0]\), such that
\[
\left|
\mathcal P_t^\varepsilon\phi(\theta,z)
-
\mathcal P_t^{\mathrm{av}}\phi(\theta)
\right|
\leq
\varepsilon C_{\mathrm{av}}
\|\phi\|_{4;m_\theta}
\mathcal W_{q_\theta,q_z}(\theta,z).
\]
This is proved in Theorem~\ref{thm:averror}, after establishing the
required bounds on the Poisson corrector in
Sec.~\ref{sec:poisson} and on the averaged semigroup in
Sec.~\ref{sec:averaged}.

For either Euler--Maruyama or S--ROCK, write
\(\mathsf M\in\{\mathrm{EM},\mathrm{SR}\}\),
let \(\widehat\pi_{\mathsf M}^{\varepsilon,\delta}\) be the
corresponding numerical invariant measure, and set
\(C_{\mathsf M}=C_{\mathrm{EM}}\) or \(C_{\mathrm{SR}}\).
Here \(m_\theta\) is the polynomial regularity order in
Assumption~\ref{ass:poly}, whereas \(m_V\) is the local growth exponent
of the objective in Theorem~\ref{thm:statmeasure}.
Under the additional method-specific stability assumptions, for
admissible \(\delta\), \(N\geq N_0\), and
\(\phi\in C_{m_\theta}^4(\mathbb R^{d_\theta})\),
\begin{align*}
\left|
\int\phi(\theta)\,
\widehat\pi_{\mathsf M}^{\varepsilon,\delta}
(\mathrm d\theta,\mathrm dz)
-\phi(\theta_\star)
\right|&\leq
\varepsilon C_{\mathrm{av}}\|\phi\|_{4;m_\theta}
+
\frac{\sqrt{C_{\mathrm{der}}}}{\rho_\varepsilon}
C_{\mathsf M}\|\phi\|_{4;m_\theta}\delta\\
&\qquad+
\sqrt3\,C_{\mathrm{conc}}\,
\|\nabla_\theta\phi\|_{0;m_\theta}
\left(
1+\bar\gamma_{2m_\theta}^{1/2}
+\|\theta_\star\|^{m_\theta}
    \right)N^{-1/m_V}.
\end{align*}
The three terms are the slow--fast averaging error, the numerical
invariant-measure error, and concentration of the averaged Gibbs
measure at the empirical MLE.  The precise method-specific
hypotheses, step-size restrictions, and constant dependencies are
given in Theorem~\ref{thm:numerical-mle}.

\subsection{Assumptions}\label{sec:assumptions}
We classify the assumptions into three categories: (i) assumptions on
the frozen dynamics, indicated by a superscript \(z\); (ii)
assumptions on the averaged dynamics, indicated by a superscript
\({\rm av}\); and (iii) assumptions on the joint slow--fast system,
indicated by a superscript \(\varepsilon\).  The frozen and averaged
conditions control the Poisson corrector.  The joint moment condition
is additionally needed when the full semigroup acts on an unbounded
corrector, while the joint derivative and numerical Lyapunov
conditions are used in the uniform-in-time discretisation analysis.
Related assumptions are common in the averaging literature
\cite{Crisan_Ottobre_2024,Crisan_Ottobre_2016,
schuh2024conditionsuniformtimeconvergence,
Pardoux_Veretennikov_2001,Pardoux_Veretennikov_2003}
and in the ULA literature
\cite{akyildiz2024multiscaleperspectivemaximummarginal,
Eberle2016QuantitativeHT,akyildiz2023interacting,
Durmus2018AnalysisOL}.

\subsubsection{Initial assumptions.}
We first state the coefficient regularity and polynomial-growth
condition used in the Poisson-equation, semigroup, and weak-error
arguments.
\begin{assumptionp}{$(A_p)$}\label{ass:poly}
Suppose that $\nabla E$ is in $C^4_{m_\theta,m_x}$.
\end{assumptionp}
Together with the moment and stability assumptions below,
Assumption~\ref{ass:poly} supplies the regularity required by the
Poisson-equation and averaging arguments; see, in particular,
\cite[Theorem~3]{Pardoux_Veretennikov_2001} and
\cite{Pardoux_Veretennikov_2003}. Since \(E_N\) is a finite sum of
copies of \(E\) and data-dependent terms, the corresponding
polynomial derivative bounds transfer to \(E_N\), with constants
that may depend on \(N\) and \(M\).

\subsubsection{Moment and dissipativity assumptions.}
We first introduce coefficient-level dissipativity assumptions for
the frozen and averaged systems.  They give the corresponding moment
and ergodicity estimates used in the Poisson and averaged analyses.
For the joint system we state the exact moment and ergodicity
property required by the proofs, together with a weighted
Foster--Lyapunov sufficient condition.

\begin{assumptionp}{$({A}^z_\mu)$}\label{ass:dissx} There exist $r_x>0$ and $b_x:\mathbb{R}^{d_\theta}\to\mathbb{R}_+$ such that
\[
\langle\nabla_x E(\theta,x),x\rangle
\geq r_x \|x\|^2-b_x(\theta),
\qquad
b_x(\theta)=O(\|\theta\|^2).
\]
Written for the potential $E_N$, this is equivalently assuming
\[
\langle \nabla_z E_N(\theta,z),z\rangle
\geq
r_x\|z\|^2- N b_x(\theta).
\]
\end{assumptionp}
One notes that $r_x$ does not depend on $\theta$, but for our case this is equivalent to saying that the above inequality holds for $b_x(\theta)$ and $r_x(\theta)$, with a positive lower bound on $r_x(\theta)$. Next, we place a similar assumption on the averaged energy function.
\begin{assumptionp}{$({A}^{\textnormal{av}}_\mu)$}\label{ass:dissav}
Suppose that $E$ is such that there exist constants $r_{\textnormal{av}},b_{\textnormal{av}}\in \mathbb{R}_+$ that satisfy the inequality,
\begin{equation*}
\frac{1}{N}\left\langle\int\nabla_\theta E_N(\theta, z)p_\theta^{\otimes N}(\mathrm{d}z), \theta\right\rangle \leq -r_{\textnormal{av}} \|\theta\|^2 + b_{\textnormal{av}},
\end{equation*}
for all $\theta\in\mathbb{R}^{d_\theta}$.
\end{assumptionp}
This result is equivalent to the dissipativity assumption on the potential $V$, 
\begin{equation*}
\langle \nabla V(\theta), \theta\rangle \geq r_{\textnormal{av}}\|\theta\|^2 - b_{\textnormal{av}},
\end{equation*}
since $N^{-1}\int\nabla_\theta E_N\,\md p_\theta^{\otimes N}=-\nabla V$.

\begin{assumptionp}{$(A^\varepsilon_{\mathrm{mom}})$}
\label{ass:jointmom}
Fix $\varepsilon_0\in(0,1]$. For every pair of polynomial orders
$q_\theta,q_z\geq0$ required below, there exists a constant
$C_{\mathrm{mom}}(q_\theta,q_z)<\infty$, independent of
$t\geq0$ and $\varepsilon\in(0,\varepsilon_0]$, such that
\begin{equation}\label{eq:joint-uniform-moments}
\mathcal P_t^\varepsilon
\mathcal W_{q_\theta,q_z}(\theta,z)
\leq
C_{\mathrm{mom}}(q_\theta,q_z)
\mathcal W_{q_\theta,q_z}(\theta,z).
\end{equation}
For each fixed $\varepsilon\in(0,\varepsilon_0]$, the joint process
has a unique invariant probability measure $\pi^\varepsilon$, and
\[
\mathcal P_t^\varepsilon h(s)
\longrightarrow
\int h\,\mathrm d\pi^\varepsilon
\qquad (t\to\infty)
\]
for every continuous $h$ dominated by one of the weights required
below.
\end{assumptionp}

Assumption~\ref{ass:jointmom} is deliberately stated as a property of
the joint process rather than as Euclidean dissipativity of its
drift. A useful sufficient condition is a Foster--Lyapunov inequality
in a possibly non-Euclidean quadratic metric. In particular, if
$M_\varepsilon\succ0$ and, for the required orders,
\[
\mathcal G_\varepsilon
\left(1+(s-s_\varepsilon)^\top
M_\varepsilon(s-s_\varepsilon)\right)^{q/2}
\leq
-a_{q,\varepsilon}
\left(1+(s-s_\varepsilon)^\top
M_\varepsilon(s-s_\varepsilon)\right)^{q/2}
+b_{q,\varepsilon},
\]
then the corresponding moment and recurrence conclusions follow
under the non-degenerate additive noise. Uniformity in
$\varepsilon$ in \eqref{eq:joint-uniform-moments} requires the
resulting moment constants to be uniform.

\subsubsection{Derivative stability assumptions.}

For the frozen semigroup we use coefficient-level conditions that yield exponential decay of the required spatial derivatives.  For the full slow--fast and averaged semigroups, whose drifts need not be instantaneously contractive in the Euclidean metric, we state the derivative estimate directly.

\begin{assumptionp}{$({A}^z_\kappa)$}\label{ass:driftfroz}
Suppose there exists a constant $\tilde{\kappa}\in\mathbb{R}_+$, such that the drift condition
\begin{equation*}
\langle \zeta, \nabla^2_z E_N \zeta \rangle
+ \Tr( \eta^\top \nabla_z^3 E_N \zeta)
+ 2 \Tr( \eta^\top \nabla_z^2 E_N \eta)
+ \|\eta\|_F^2
\geq \tilde{\kappa} (\| \zeta \|^2 + \| \eta \|_F^2),
\end{equation*}
is satisfied for all $\zeta \in \mathbb{R}^{Nd_x}$ and symmetric $\eta\in\mathbb{R}^{Nd_x\times Nd_x}$, uniformly over $\theta\in\mathbb{R}^{d_\theta}$ and $z\in\mathbb{R}^{Nd_x}$.
\end{assumptionp}

A sufficient way to verify Assumption~\ref{ass:driftfroz} follows from
Young's inequality.  For a fixed $\tau>0$, it is enough to check
\begin{align*}
\langle\zeta,\nabla_z^2E_N\zeta\rangle
-\frac{\tau}{2}\|\nabla_z^3E_N\zeta\|_F^2
&\geq \widetilde\kappa\|\zeta\|^2,\\
2\Tr(\eta^\top\nabla_z^2E_N\eta)
+\left(1-\frac{1}{2\tau}\right)\|\eta\|_F^2
&\geq \widetilde\kappa\|\eta\|_F^2.
\end{align*}
These inequalities are a sufficient criterion, rather than an
equivalent splitting of Assumption~\ref{ass:driftfroz}.  In
particular, the criterion holds when $E_N(\theta,\cdot)$ is uniformly
strongly convex and $\nabla_z^3E_N$ is uniformly bounded sufficiently
small relative to the convexity constant.

Strong convexity also implies Assumption~\ref{ass:dissx} provided the
gradient at the origin has the required growth in $\theta$.  Indeed, if
$E(\theta,\cdot)$ is $m$-strongly convex, then
\[
\langle\nabla_xE(\theta,x),x\rangle
\geq
\frac{m}{2}\|x\|^2
-\frac{1}{2m}\|\nabla_xE(\theta,0)\|^2.
\]
Thus Assumption~\ref{ass:dissx} follows, for example, when
$\|\nabla_xE(\theta,0)\|^2\leq C(1+\|\theta\|^2)$.  The role of
Assumption~\ref{ass:driftfroz} is stronger than dissipativity: it gives
exponential stability of derivatives of the frozen semigroup, as used
in Lemma~\ref{lem:DEx}.

\begin{assumptionp}{$({A}^{\mathrm{av}}_{\mathrm{SES}})$}
\label{ass:avSES}
There exist constants \(C_{\mathrm{SES}}^{\mathrm{av}}<\infty\),
\(\rho_{\mathrm{av}}>0\), and a polynomial order
\(q_{\mathrm{av}}\geq0\) such that, for every
\(\phi\in C_{m_\theta}^4(\mathbb R^{d_\theta})\),
\(\theta\in\mathbb R^{d_\theta}\), and \(t\geq0\),
\begin{align}
\|\nabla_\theta\mathcal P_t^{\mathrm{av}}\phi(\theta)\|^2
+
\|\nabla_\theta^2\mathcal P_t^{\mathrm{av}}\phi(\theta)\|_F^2
\leq
C_{\mathrm{SES}}^{\mathrm{av}}
\|\phi\|_{4;m_\theta}^2
e^{-2\rho_{\mathrm{av}}t}
\mathcal W_{2q_{\mathrm{av}}}(\theta).
\label{eq:averaged-SES}
\end{align}
\end{assumptionp}

\begin{remark}
The frozen coefficient condition is used below to derive frozen semigroup estimates. For the averaged and full slow--fast processes we instead impose directly the derivative estimates consumed by the proofs. In particular, Assumption~\ref{ass:avSES} does not imply \(\nabla^2V\succeq\kappa I\) and therefore does not exclude non-convex objectives. 
\end{remark}

\begin{assumptionp}{$(A^\varepsilon_{\mathrm{SES}})$}
\label{ass:jointSES}
For each fixed $N$ and $\varepsilon\in(0,\varepsilon_0]$, and every
polynomial order $m$ required below, there exist constants
$C_{\mathrm{der}}=C_{\mathrm{der}}(m,N,\varepsilon)>0$ and
$\rho_\varepsilon=\rho_\varepsilon(m,N,\varepsilon)>0$ such that
\begin{align}
\|\nabla\mathcal P_t^\varepsilon\phi(s)\|^2
+\|\nabla^2\mathcal P_t^\varepsilon\phi(s)\|_F^2
+\|\nabla^3\mathcal P_t^\varepsilon\phi(s)\|_F^2
\leq
C_{\mathrm{der}}\|\phi\|_{4;m}^2
e^{-2\rho_\varepsilon t}\mathcal W_{2m}(s),
\label{eq:joint-SES}
\end{align}
for every $\phi\in C_m^4$, $s\in\mathbb R^{d_\theta+Nd_x}$,
and $t\geq0$.
\end{assumptionp}

No uniformity in $N$ or $\varepsilon$ is implicit in
Assumption~\ref{ass:jointSES}. If such uniformity is needed, it must
be imposed through explicit upper bounds on $C_{\mathrm{der}}$ and a
positive lower bound on $\rho_\varepsilon$.

The following global Lipschitz condition gives global well-posedness
of the joint diffusion and controls the coefficient differences in
the one-step weak-error expansions.  It also supplies the joint
Lipschitz constant used in the S--ROCK invertibility argument below.
\begin{assumptionp}{$(A_L)$} \label{ass:lip}
Suppose there exist a constant $L>1$, independent of $(\theta, x)^\top$ or $(\theta', x')^\top$, such that
\begin{equation*}
\|\nabla E(\theta, x)-\nabla E(\theta', x')\|\leq \frac{L}{2}\sqrt{\|x-x'\|^2 + \|\theta-\theta'\|^2},
\end{equation*}
for all $\theta,\theta'\in\mathbb{R}^{d_\theta}$ and $x,x'\in\mathbb{R}^{d_x}$.
\end{assumptionp}

\begin{assumptionp}{$(A_{\mathrm{num}}^{\varepsilon,\delta})$}
\label{ass:numstability}
Let $Q_\delta^\varepsilon$ be the one-step transition operator of
one of the numerical methods considered below. For every polynomial
order $q\geq2$ required by the corresponding weak-error argument and over all $t\geq 0$,
there exist a coercive function
$\mathcal V_{q}^{\varepsilon,\delta}$, constants
$0<c^-_{q}\leq c^+_{q}<\infty$, $c_q,B_q>0$, and an admissible
step-size $\delta_q>0$ such that
\[
c^-_q\mathcal W_q(s)
\leq
\mathcal V_q^{\varepsilon,\delta}(s)
\leq
c^+_q\mathcal W_q(s)
\]
and
\begin{equation}\label{eq:numerical-lyapunov}
Q_\delta^\varepsilon
\mathcal V_q^{\varepsilon,\delta}(s)
\leq
(1-c_q\delta)
\mathcal V_q^{\varepsilon,\delta}(s)
+B_q\delta,
\qquad
0<\delta\leq\delta_q.
\end{equation}
After decreasing \(\delta_q\) if necessary, we assume
\(c_q\delta_q\leq1\).  The constants $c_q^-, c_q^+, c_q$ and $B_q$ may depend on the method, $N$, and $\varepsilon$, but are independent of the time index and of \(\delta\in(0,\delta_q]\).
\end{assumptionp}

The following example demonstrates how these assumptions can hold for non-trivial parametric families of non-Gaussian distributions.

\begin{example}[Continuous-time assumptions for a non-Gaussian
linearly-tilted family]\label{ex:softplus}
We consider a non-Gaussian parametric family satisfying the
regularity, frozen, averaged, and joint continuous-time assumptions
used in the main results.  As in
Example~\ref{example1}, we take $d_\theta=d_x=1$ noting that the construction
extends coordinate-wise to $\mathbb{R}^d$.

Fix $\gamma\in(0,1]$, let $\sigma(x)=(1+e^{-x})^{-1}$ be the logistic sigmoid and
$s(x)=\log(1+e^{x})$ the softplus, and consider the linearly tilted energy
\[
  E(\theta,x)=U(x)-\theta x,\qquad U(x)=\tfrac12 x^2+\gamma\,s(x),
\]
so that the model is a Gaussian density tilted by a logistic factor,
\[
  p_\theta(x)=\frac{1}{Z_\theta}\,e^{\theta x-\frac12 x^2}\,(1+e^{x})^{-\gamma}.
\]
This family is non-Gaussian.  Indeed, the third cumulant
\[
  \kappa_3(\theta)=(\log Z_\theta)'''=\mathbb E_{p_\theta}\!\big[(x-m(\theta))^3\big]
  =-\gamma\,\partial_\theta\operatorname{Cov}_{p_\theta}\!\big(\sigma(x),x\big)\not\equiv0
\]
is nonzero, vanishing identically for every Gaussian family, and the variance
$V''(\theta)$ depends on $\theta$, and mean, mode and median differ.

Since the function $x\mapsto-x\sigma(x)$ is bounded above, there exists
a constant $C_\sigma<\infty$ such that
\[
x\sigma(x)\geq-C_\sigma
\qquad
\text{for every }x\in\mathbb{R}.
\]
Consequently,
\begin{align*}
\left\langle\nabla_xE(\theta,x),x\right\rangle
&=
x^2+\gamma x\sigma(x)-\theta x\\
&\geq
\frac12x^2-\frac12\theta^2-\gamma C_\sigma.
\end{align*}
Thus Assumption~\ref{ass:dissx} holds with
\[
r_x=\frac12,
\qquad
b_x(\theta)=\frac12\theta^2+\gamma C_\sigma.
\]

To verify Assumption~\ref{ass:driftfroz}, observe that
$\nabla_z^2E_N$ is diagonal, with all diagonal entries bounded below
by one. Moreover, writing
\[
q_\gamma:=\|U'''\|_\infty
\leq\frac{\gamma}{6\sqrt3},
\]
we have, for every $\zeta\in\mathbb{R}^{N}$ and every symmetric
$\eta\in\mathbb{R}^{N\times N}$,
\[
\left|
\Tr\left(
\eta^\top\nabla_z^3E_N\,\zeta
\right)
\right|
\leq
q_\gamma\|\eta\|_F\|\zeta\|.
\]
It follows that
\begin{align*}
\left\langle\zeta,\nabla_z^2E_N\zeta\right\rangle
+
\Tr\left(\eta^\top\nabla_z^3E_N\,\zeta\right)
+
2\Tr\left(\eta^\top\nabla_z^2E_N\eta\right)
+
\|\eta\|_F^2&\geq
\|\zeta\|^2
+
3\|\eta\|_F^2
-
q_\gamma\|\zeta\|\|\eta\|
\\
&\geq
\left(1-\frac{q_\gamma}{2}\right)\|\zeta\|^2
+
\left(3-\frac{q_\gamma}{2}\right)\|\eta\|_F^2.
\end{align*}
Hence Assumption~\ref{ass:driftfroz} holds, for example with
\[
\widetilde\kappa
=
1-\frac{q_\gamma}{2}
>0.
\]

The regularity assumptions also hold. Indeed,
\[
\nabla E(\theta,x)
=
\begin{pmatrix}
-x\\
x+\gamma\sigma(x)-\theta
\end{pmatrix}
\]
has linear growth, all its higher derivatives are bounded, and its
Jacobian is
\[
\begin{pmatrix}
0&-1\\
-1&U''(x)
\end{pmatrix}.
\]
Since $U''$ is uniformly bounded, Assumptions~\ref{ass:poly} and
\ref{ass:lip} are satisfied; in Assumption~\ref{ass:poly}, one may
take \(m_\theta=m_x=1\).

We next study the averaged objective. Integrating
\[
\partial_x
\exp\bigl(\theta x-U(x)\bigr)
=
\bigl(\theta-U'(x)\bigr)
\exp\bigl(\theta x-U(x)\bigr)
\]
over $\mathbb{R}$ gives
\[
\mathbb{E}_{p_\theta}[U'(X)]=\theta.
\]
Therefore, with
\[
m(\theta):=\mathbb{E}_{p_\theta}[X],
\]
we obtain
\[
m(\theta)
=
\nabla_\theta\log Z_\theta
=
\theta-\gamma\mathbb{E}_{p_\theta}[\sigma(X)].
\]

Writing
\[
\bar y=\frac1M\sum_{j=1}^M y_j,
\]
the negative empirical log-likelihood is
\[
V(\theta)
=
\log Z_\theta-\theta\bar y+\mathrm{const},
\]
and hence
\[
V'(\theta)=m(\theta)-\bar y,
\qquad
V''(\theta)=\Var_{p_\theta}(X).
\]

Since
\[
1\leq \partial_x^2 E(\theta, x)\leq1+\frac{\gamma}{4},
\]
integration by parts gives
\[
1
=
\mathbb{E}_{p_\theta}
\left[
(X-m(\theta))\partial_x E(\theta, X)
\right].
\]
Because the constant $\partial_x E(\theta, m(\theta))$ has zero contribution,
\[
1
=
\mathbb{E}_{p_\theta}
\left[
(X-m(\theta))
\left(
\partial_x E(\theta, X)-\partial_x E(\theta, m(\theta))
\right)
\right].
\]
The bounds on $\partial_x^2 E$ therefore imply
\[
\frac{1}{1+\gamma/4}
\leq
V''(\theta)
\leq
1.
\]
Set
\[
a_\gamma:=\frac{1}{1+\gamma/4}.
\]
Thus $V$ is globally $a_\gamma$-strongly convex and has a unique
minimiser $\theta_\star$, characterised by
\[
m(\theta_\star)=\bar y.
\]
Equivalently,
\[
\theta_\star
=
\bar y
+
\gamma
\mathbb{E}_{p_{\theta_\star}}[\sigma(X)]
\in
(\bar y,\bar y+\gamma).
\]

Global strong convexity implies Assumption~\ref{ass:dissav}. Moreover,
\[
\frac{a_\gamma}{2}
|\theta-\theta_\star|^2
\leq
V(\theta)-V(\theta_\star)
\leq
\frac12|\theta-\theta_\star|^2,
\]
so the two-sided local concentration condition of
Theorem~\ref{thm:statmeasure} holds with $m_V=2$.

It remains to verify Assumption~\ref{ass:avSES}. Differentiating the
variance gives
\[
V'''(\theta)
=
-\gamma\,
\partial_\theta
\Cov_{p_\theta}(\sigma(X),X).
\]
For this exponential family,
\[
\partial_\theta
\Cov_{p_\theta}(\sigma(X),X)
=
\mathbb{E}_{p_\theta}
\left[
\left(
\sigma(X)-\mathbb{E}_{p_\theta}\sigma(X)
\right)
(X-m(\theta))^2
\right].
\]
Since $0\leq\sigma\leq1$ and
$\Var_{p_\theta}(X)\leq1$, we obtain
\[
|V'''(\theta)|\leq\gamma.
\]

The averaged drift is $b_{\mathrm{av}}=-V'$. For
$u_t=\mathcal P_t^{\mathrm{av}}\phi$, set
\[
\Gamma(u_t)=|u_t'|^2+|u_t''|^2.
\]
Since $\mathcal G_{\mathrm{av}}
=
-V'(\theta)\partial_\theta+\frac1N\partial_\theta^2$,
a direct commutator calculation gives
\begin{align}
(\partial_t-\mathcal G_{\mathrm{av}})\Gamma(u_t)
={}&
2\left[
-V''(\theta)(u_t')^2
-V'''(\theta)u_t''u_t'
-\left(2V''(\theta)+\frac1N\right)(u_t'')^2
\right]-\frac2N|u_t'''|^2.
\label{eq:softplus-averaged-commutator}
\end{align}
For arbitrary $\zeta,\eta\in\mathbb R$, the expression in square
brackets is
\[
-V''(\theta)\zeta^2
-
V'''(\theta)\eta\zeta
-
\left(
2V''(\theta)+\frac1N
\right)\eta^2.
\]
Using
\[
|V'''(\theta)\eta\zeta|
\leq
\frac{\gamma}{2}
\left(
\zeta^2+\eta^2
\right),
\]
we find
\begin{align*}
&
-V''(\theta)\zeta^2
-
V'''(\theta)\eta\zeta
-
\left(
2V''(\theta)+\frac1N
\right)\eta^2\leq
-\left(a_\gamma-\frac{\gamma}{2}\right)\zeta^2
-
\left(
2a_\gamma+\frac1N-\frac{\gamma}{2}
\right)\eta^2.
\end{align*}
For $\gamma\in(0,1]$,
\[
a_\gamma-\frac{\gamma}{2}>0.
\]
Hence, with
\[
\rho_{\mathrm{av}}
=
a_\gamma-\frac{\gamma}{2},
\]
the right-hand side of
\eqref{eq:softplus-averaged-commutator} is bounded above by
\(-2\rho_{\mathrm{av}}\Gamma(u_t)\). The semigroup interpolation
argument used in Lemma~\ref{lem:DEx} therefore yields
\[
\Gamma(\mathcal P_t^{\mathrm{av}}\phi)(\theta)
\leq
e^{-2\rho_{\mathrm{av}}t}
\mathcal P_t^{\mathrm{av}}\Gamma(\phi)(\theta).
\]
Combining this estimate with Lemma~\ref{lem:boundedmomav} proves
Assumption~\ref{ass:avSES}, for example with
\(q_{\mathrm{av}}=\max\{m_\theta,1\}\) and a finite
\(C_{\mathrm{SES}}^{\mathrm{av}}\) determined by the averaged moment
bound. Thus the present example satisfies the direct averaged
semigroup assumption; global strong convexity is used here only as a
convenient sufficient condition for that estimate.

\paragraph{\textbf{Verification of the joint assumptions.}}
Write
\[
x_\circ=\bar y,
\qquad
\theta_\circ=U'(\bar y),
\qquad
u=\theta-\theta_\circ,
\qquad
v_i=x^i-x_\circ,
\]
and introduce
\[
h(v)=U'(x_\circ+v)-U'(x_\circ),
\qquad
\bar v=\frac1N\sum_{i=1}^Nv_i,
\qquad
R^2=\frac1N\sum_{i=1}^Nv_i^2.
\]
The centred joint dynamics are
\begin{align}
\mathrm du_t
&=
-\bar v_t\,\mathrm dt
+\sqrt{\frac2N}\,\mathrm dW_t^\theta,
\nonumber\\
\mathrm dv_t^i
&=
\frac1\varepsilon
\left(u_t-h(v_t^i)\right)\mathrm dt
+\sqrt{\frac2\varepsilon}\,\mathrm dW_t^i.
\label{eq:softplus-centred-joint}
\end{align}
Set
\[
L_\gamma=1+\frac{\gamma}{4}.
\]
The bounds on \(U''\) imply
\begin{equation}\label{eq:softplus-sector}
v\,h(v)\geq v^2,
\qquad
|h(v)|\leq L_\gamma|v|,
\qquad
1\leq h'(v)\leq L_\gamma.
\end{equation}

Choose
\[
0<\alpha<
\min\left\{
1,\,
\frac{2}{L_\gamma^2+2\varepsilon_0}
\right\}
\]
and define the twisted quadratic Lyapunov function
\begin{equation}\label{eq:softplus-joint-lyapunov}
\mathscr V_\varepsilon(u,v)
=
1+u^2+\varepsilon R^2
-2\alpha\varepsilon u\bar v.
\end{equation}
Since \(\bar v^2\leq R^2\), there are constants
\(c_-,c_+>0\), independent of
\(\varepsilon\in(0,\varepsilon_0]\), such that
\begin{equation}\label{eq:softplus-lyapunov-comparison}
c_-\left(1+u^2+\varepsilon R^2\right)
\leq
\mathscr V_\varepsilon(u,v)
\leq
c_+\left(1+u^2+\varepsilon R^2\right).
\end{equation}
A direct generator calculation using
\eqref{eq:softplus-sector} gives
\begin{equation}\label{eq:softplus-joint-drift}
\mathcal G_\varepsilon\mathscr V_\varepsilon
\leq
-\alpha u^2-\beta R^2+2+\frac2N,
\qquad
\beta
:=
2-\alpha\left(L_\gamma^2+2\varepsilon_0\right)>0.
\end{equation}
Indeed, the derivative of
\(u^2+\varepsilon R^2\) contributes
\(-2N^{-1}\sum_i v_i h(v_i)\), while the derivative of the cross
term is
\[
2\alpha\varepsilon\bar v^2
-2\alpha u^2
+2\alpha u\,\overline{h(v)}.
\]
The final term is controlled by
\[
2\alpha|u|\,|\overline{h(v)}|
\leq
\alpha u^2+\alpha L_\gamma^2R^2.
\]
In view of \eqref{eq:softplus-lyapunov-comparison} and
\(0<\varepsilon\leq\varepsilon_0\), this also implies that there are
constants \(a_0,b_0>0\), independent of \(\varepsilon\), such that
\[
\mathcal G_\varepsilon\mathscr V_\varepsilon
\leq
-a_0\mathscr V_\varepsilon+b_0.
\]

Although the comparison in
\eqref{eq:softplus-lyapunov-comparison} weights the fast component by
\(\varepsilon\), the fast equation supplies a uniform Euclidean
moment bound.  More precisely,
\begin{equation}\label{eq:softplus-fast-moment}
\mathcal G_\varepsilon R^2
\leq
-\frac1\varepsilon R^2
+\frac1\varepsilon u^2
+\frac2\varepsilon.
\end{equation}
The first Lyapunov inequality gives a uniform bound on
\(\mathbb E|u_t|^2\).  Substitution into
\eqref{eq:softplus-fast-moment}, followed by Gr\"onwall's inequality
on the fast time scale, then gives
\[
\sup_{\substack{t\geq0\\0<\varepsilon\leq\varepsilon_0}}
\mathbb E_{u,v}
\left[1+|u_t|^2+\|v_t\|^2\right]
\leq
C_N\left(1+|u|^2+\|v\|^2\right).
\]
For even orders, applying the generator to powers of
\(\mathscr V_\varepsilon\) and \(R^2\), and absorbing the resulting
cross terms by Young's inequality, gives the same estimate with
\(|u|^2+\|v\|^2\) replaced by its corresponding power.  The remaining
orders follow by interpolation.  The constants can be chosen
uniformly for \(0<\varepsilon\leq\varepsilon_0\).  Since the diffusion
matrix is non-degenerate and the drift is globally Lipschitz, the
Foster--Lyapunov estimate, strong Feller property, and irreducibility
also give, for each fixed \(\varepsilon>0\), a unique invariant
measure and convergence to it.  Thus
Assumption~\ref{ass:jointmom} holds.

We next verify Assumption~\ref{ass:jointSES}.  Couple two solutions of
\eqref{eq:softplus-centred-joint} with the same Brownian motion and
write their differences as \(p\) and \(q_i\).  Then
\begin{equation}\label{eq:softplus-incremental-system}
\dot p=-\bar q,
\qquad
\varepsilon\dot q_i=p-d_i(t)q_i,
\qquad
1\leq d_i(t)\leq L_\gamma,
\end{equation}
where the last bounds follow from the mean-value theorem.  Define
\[
\mathscr Q_\varepsilon(p,q)
=
p^2+\varepsilon\frac1N\sum_{i=1}^Nq_i^2
-2\alpha\varepsilon p\bar q.
\]
The same calculation as above, now without diffusion terms, yields
\begin{equation}\label{eq:softplus-incremental-contraction}
\frac{\mathrm d}{\mathrm dt}
\mathscr Q_\varepsilon(p_t,q_t)
\leq
-\alpha p_t^2
-\beta\frac1N\sum_{i=1}^N(q_t^i)^2.
\end{equation}
For fixed \(\varepsilon>0\),
\(\mathscr Q_\varepsilon\) is equivalent to the Euclidean norm, so
\eqref{eq:softplus-incremental-contraction} gives exponential
contraction of the first variational flow.  All derivatives of
\(\sigma\) are bounded.  The second and third variational equations
therefore consist of the same exponentially stable linear equation
forced, respectively, by products of lower-order variational
solutions.  Variation of constants gives exponential decay of these
two variational flows as well.  Differentiating
\(\mathcal P_t^\varepsilon\phi
=\mathbb E[\phi(S_t)]\) up to third order, and combining these
variational estimates with the joint moment bounds above, proves
\eqref{eq:joint-SES}.  Hence Assumption~\ref{ass:jointSES} holds,
with constants allowed to depend on \(N\) and \(\varepsilon\), as
specified in that assumption.

Consequently, this family satisfies the continuous-time assumptions
used in the Poisson, averaged, and averaging analyses, together with
the exact-semigroup stability assumption used later in the numerical
analysis.  We return to the example after presenting the
Euler--Maruyama and S--ROCK schemes to verify their common numerical
stability assumption.

\end{example}

\section{The Poisson equation}\label{sec:poisson}

We first establish moment and ergodicity estimates for the frozen
dynamics, followed by exponential stability estimates for the frozen
semigroup and its derivatives.  At the end of the section these
results are assembled to construct the centred Poisson corrector and
to control its first two parameter derivatives.

\subsection{Frozen semigroup}
The next results establish a bound on the moments of the semigroup $\mathcal{P}^z_t$ for all $t$, which in the limit $t\to \infty$, gives us bounds on the moments of the invariant measure $p_\theta^{\otimes N}$.

\begin{lemma}\label{lem:boundedmom}
For the semigroup of the ``frozen'' process \eqref{eq:frozensde}, satisfying \ref{ass:dissx},
\begin{equation}\label{eq:mombound}
\mathcal{P}^z_t \|z\|^k \leq e^{-\tilde{\alpha}_k t} \|z\|^k + \tilde{\gamma}^\theta_k,
\end{equation}
with,
\begin{equation*}
\tilde{\alpha}_k = \frac{kr_x}{2}, \qquad \tilde{\gamma}_k^\theta = \left(\frac{2(Nb_x(\theta) + d_z+k-2)}{r_x}\right)^\frac{k}{2}    ,
\end{equation*}
for all $z\in\mathbb{R}^{Nd_x}$, $\theta\in\mathbb{R}^{d_\theta}$ (recall that the semigroup $\mathcal{P}^z_t$ depends on an initial choice of $\theta$), $t\geq 0$ and $k\geq 2$. For the same choices of parameters we have,
\begin{equation*}
\mathbb{E}_{\tilde{z}\sim p_\theta^{\otimes N}}\|\tilde{z}\|^k \leq \tilde{\gamma}^\theta_k.
\end{equation*}
\end{lemma}

\begin{proof}
Let us observe that by \ref{ass:dissx},
\begin{align*}
\mathcal{G}_z \|z\|^k =& -k \langle \nabla_z E_N(\theta, z), z\rangle \|z\|^{k-2} + k(d_z +k-2)\|z\|^{k-2}\\
\leq & -kr_x\|z\|^k + k(Nb_x(\theta) + d_z +k-2)\|z\|^{k-2}\\
\leq& -\frac{kr_x}{2}\|z\|^k + \frac{kr_x}{2}\left(\frac{2(Nb_x(\theta) + d_z + k-2)}{r_x}\right)^\frac{k}{2},
\end{align*}
where the last line follows from Young's Inequality. Let us now note that $\partial_t \mathcal{P}^z_t \|z\|^k = \mathcal{P}^z_t\mathcal{G}_z \|z\|^k$. By the positivity of the Markov semigroup and the result above,
\begin{align*}
\frac{\mathrm{d}}{\mathrm{d}t}\left(e^\frac{kr_xt}{2} \mathcal{P}^z_t \|z\|^k\right) &= \left(\frac{kr_x}{2}\mathcal{P}^z_t \|z\|^k + \mathcal{P}^z_t \mathcal{G}_z \|z\|^k\right)e^\frac{kr_xt}{2}\\
&\leq \frac{kr_x}{2}\left(\frac{2(Nb_x(\theta) + d_z+k-2)}{r_x}\right)^\frac{k}{2} e^\frac{kr_xt}{2}.
\end{align*}
Integrating both sides,
\begin{equation*}
\mathcal{P}^z_t \|z\|^k \leq e^{-\frac{kr_xt}{2}} \|z\|^k + \left(\frac{2(Nb_x(\theta)+d_z+k-2)}{r_x}\right)^\frac{k}{2}.
\end{equation*}
\end{proof}

\begin{theorem}[Frozen Wasserstein contraction]\label{thm:frozconv}
Suppose that Assumptions~\ref{ass:dissx}, \ref{ass:driftfroz}, and
\ref{ass:poly} hold.  Then, for every
$\mu,\nu\in\mathscr P_2(\mathbb R^{Nd_x})$,
\begin{equation}\label{eq:frozen-w2-contraction}
W_2\!\left(
\mathcal P_t^{z,*}\mu,\mathcal P_t^{z,*}\nu
\right)
\leq
e^{-\widetilde\kappa t}W_2(\mu,\nu).
\end{equation}
In particular,
\begin{equation}\label{eq:frozen-to-equilibrium}
W_2\!\left(
\mathcal P_t^{z,*}\delta_z,\pi_\theta^z
\right)
\leq
\sqrt{2}\,e^{-\widetilde\kappa t}
\left(
\|z\|^2+\widetilde\gamma_2^\theta
\right)^{1/2}.
\end{equation}
\end{theorem}

\begin{proof}
Taking $\eta=0$ in Assumption~\ref{ass:driftfroz} gives
\[
\langle v,\nabla_z^2E_N(\theta,z)v\rangle
\geq
\widetilde\kappa\|v\|^2.
\]
Consequently, for every $z,z'\in\mathbb R^{Nd_x}$,
\[
\left\langle
\nabla_zE_N(\theta,z)-\nabla_zE_N(\theta,z'),
z-z'
\right\rangle
\geq
\widetilde\kappa\|z-z'\|^2.
\]

Couple two frozen processes synchronously, using the same Brownian
motion, and denote them by $Z_t$ and $Z_t'$.  Their difference has no
stochastic term, and therefore
\[
\frac{\mathrm d}{\mathrm dt}\|Z_t-Z_t'\|^2
=
-2\left\langle
Z_t-Z_t',
\nabla_zE_N(\theta,Z_t)-\nabla_zE_N(\theta,Z_t')
\right\rangle
\leq
-2\widetilde\kappa\|Z_t-Z_t'\|^2.
\]
Gronwall's inequality yields
\[
\|Z_t-Z_t'\|^2
\leq
e^{-2\widetilde\kappa t}\|Z_0-Z_0'\|^2.
\]
Starting from an optimal coupling of $\mu$ and $\nu$ and taking
expectations proves \eqref{eq:frozen-w2-contraction}.

The measure $\pi_\theta^z$ is invariant.  Applying
\eqref{eq:frozen-w2-contraction} with
$\mu=\delta_z$ and $\nu=\pi_\theta^z$, followed by
Lemma~\ref{lem:boundedmom}, gives
\[
\begin{aligned}
W_2^2\!\left(
\mathcal P_t^{z,*}\delta_z,\pi_\theta^z
\right)
&\leq
e^{-2\widetilde\kappa t}
\int\|z-w\|^2\,\pi_\theta^z(\mathrm dw)\\
&\leq
2e^{-2\widetilde\kappa t}
\left(
\|z\|^2+\widetilde\gamma_2^\theta
\right),
\end{aligned}
\]
which is \eqref{eq:frozen-to-equilibrium}.
\end{proof}

We now establish the bounds in terms of $\theta$ for the constants above.

\begin{lemma}\label{lem:frozen-constant-growth}
Suppose that Assumption~\ref{ass:dissx} holds. Then, for every
$k\geq2$, there exist constants $C_{c},C_{\gamma,k}<\infty$,
independent of $\theta$, such that
\[
\widetilde c_\theta
\leq
C_c(1+\|\theta\|^2)
\]
and
\[
\left(\widetilde\gamma_k^\theta\right)^{1/k}
\leq
C_{\gamma,k}(1+\|\theta\|).
\]
\end{lemma}

\begin{proof}
The first bound follows from
$\widetilde c_\theta=2(b_x(\theta)+d_x)$. Moreover,
\[
\widetilde\gamma_k^\theta
=
\left(
\frac{
2\left(Nb_x(\theta)+d_z+k-2\right)
}{r_x}
\right)^{k/2},
\]
and the asserted estimate follows from
$b_x(\theta)\leq B_x(1+\|\theta\|^2)$.
\end{proof}

These moment bounds will be used to control the polynomial weights in
the exponential stability estimates.  We first record a comparison
between weighted expectation differences and the Wasserstein
distance.

\begin{lemma}\label{lem:Equiv}
Suppose $\phi\in C^1_m(\mathbb{R}^d)$ for $m \geq 1$. Then,
\begin{equation*}
\left|\mathbb{E}_\mu \phi(x)-\mathbb{E}_\nu \phi(x)\right|\leq \sqrt{3}\|\nabla \phi\|_{0;m} W_2(\mu,\nu)(1+\mathbb{E}_\mu [\|x\|^{2m}]^\frac{1}{2}+\mathbb{E}_\nu[\|x\|^{2m}]^\frac{1}{2}),
\end{equation*}
for measures $\mu,\nu\in\mathscr{P}_{2m}(\mathbb{R}^d)$.
\end{lemma}

\begin{proof}
Consider an arbitrary coupling $\Gamma$ between $\mu$ and $\nu$.
From the definition of \(\|\nabla\phi\|_{0;m}\) and
H\"older's inequality,
\begin{align*}
\left|\int \int_x^{x'} \nabla \phi(s) \mathrm{d}s \Gamma (\mathrm{d}x,\mathrm{d}x')\right| \leq & \int \int_x^{x'} \|\nabla \phi\|_{0;m} (1+\|s\|^m)\mathrm{d}s  \Gamma(\mathrm{d}x,\mathrm{d}x')\\
\leq& \|\nabla \phi\|_{0;m}\int \|x-x'\| (1+\|x\|^m+ \|x'\|^m)\Gamma (\mathrm{d}x,\mathrm{d}x')\\
\leq& \sqrt{3}\|\nabla \phi\|_{0;m} \left(\int \|x-x'\|^2 \Gamma(\mathrm{d}x, \mathrm{d}x')\right)^\frac{1}{2} \times\\&(1+\mathbb{E}_\mu[\|x\|^{2m}]^\frac{1}{2}+\mathbb{E}_\nu[\|x\|^{2m}]^\frac{1}{2}).
\end{align*}
We now choose $\Gamma$ to be the coupling that minimises the $L_2$ distance of $\mu$ and $\nu$ to note that,
\begin{equation*}
\left|\mathbb{E}_\mu \phi(x)-\mathbb{E}_\nu \phi(x)\right| \leq \sqrt{3}\|\nabla \phi\|_{0;m} W_2(\mu,\nu)(1+\mathbb{E}_\mu[\|x\|^{2m}]^\frac{1}{2}+\mathbb{E}_\nu[\|x\|^{2m}]^\frac{1}{2})
\end{equation*}
and hence the desired result.

If the weighted gradient bound contains additional
parameter-dependent terms, the same argument applies with the
corresponding finite-term Cauchy--Schwarz constant; this is used in
Lemma~\ref{lem:SES}.
\end{proof}

\subsection{Strong exponential stability}
We now prove exponential decay for the frozen semigroup and its
derivatives.  These estimates will ensure the convergence and
regularity of the Poisson corrector constructed at the end of the
section.

\begin{lemma}\label{lem:SES}
Let
\[
\pi_\theta^z:=p_\theta^{\otimes N},
\qquad
\mathcal P_\infty^z\phi(\theta)
:=
\int
\phi(\theta,w)\,
\pi_\theta^z(\mathrm dw).
\]
Suppose that the assumptions of
Theorem~\ref{thm:frozconv} hold. Set
\[
m_\star=\max\{m_\theta,m_x\},
\qquad
a_0=2(m_\star+2),
\qquad
b_0=2(m_x+1).
\]
Then there exists a constant $C_0<\infty$, independent of
$t,\theta,z$, and $\phi$, such that, for every
$\phi\in C^1_{m_\theta,m_x}$,
$t\geq0$, $\theta\in\mathbb R^{d_\theta}$, and
$z\in\mathbb R^{Nd_x}$,
\begin{equation}\label{eq:frozen-ses-weighted}
\left\|
\mathcal P_t^z\phi(\theta,z)
-
\mathcal P_\infty^z\phi(\theta)
\right\|
\leq
C_0
\|\nabla_z\phi\|_{0;m_\theta,m_x}
e^{-\rho_0t}
\mathcal W_{a_0,b_0}(\theta,z),
\qquad
\rho_0=\widetilde\kappa.
\end{equation}
\end{lemma}
\begin{proof}
Fix $\theta\in\mathbb R^{d_\theta}$. Applying
Lemma~\ref{lem:Equiv} in the $z$-variable with
\[
\mu=\mathcal P_t^{z,*}\delta_z,
\qquad
\nu=\pi_\theta^z,
\]
and using the invariance of $\pi_\theta^z$, we obtain
\begin{align}
\left\|
\mathcal P_t^z\phi(\theta,z)
-
\mathcal P_\infty^z\phi(\theta)
\right\|&\leq
2\,\|\nabla_z\phi\|_{0;m_\theta,m_x}
W_2\left(
\mathcal P_t^{z,*}\delta_z,
\pi_\theta^z
\right)
\nonumber\\
&\qquad\times
\left[
1+\|\theta\|^{m_\theta}
+
\left(
\mathcal P_t^z\|\cdot\|^{2m_x}(z)
\right)^{1/2}
+
\left(
\int\|w\|^{2m_x}\pi_\theta^z(\mathrm dw)
\right)^{1/2}
\right].
\label{eq:ses-first}
\end{align}

By Theorem~\ref{thm:frozconv},
Lemma~\ref{lem:boundedmom}, and
Lemma~\ref{lem:frozen-constant-growth},
\begin{align*}
W_2\left(
\mathcal P_t^{z,*}\delta_z,
\pi_\theta^z
\right)
&\leq
\sqrt{2}\,e^{-\rho_0t}
\left(
\|z\|^2+\widetilde\gamma_2^\theta
\right)^{1/2}\\
&\leq
C e^{-\rho_0t}
(1+\|\theta\|)(1+\|z\|)\\
&\leq
C e^{-\rho_0t}
(1+\|\theta\|^2)(1+\|z\|).
\end{align*}
Similarly,
\begin{align*}
&
1+\|\theta\|^{m_\theta}
+
\left(
\mathcal P_t^z\|\cdot\|^{2m_x}(z)
\right)^{1/2}
+
\left(
\int\|w\|^{2m_x}\pi_\theta^z(\mathrm dw)
\right)^{1/2}\leq
C
(1+\|\theta\|^{m_\star})
(1+\|z\|^{m_x}).
\end{align*}
Consequently,
\begin{align*}
&
\left\|
\mathcal P_t^z\phi(\theta,z)
-
\mathcal P_\infty^z\phi(\theta)
\right\|\leq
C\|\nabla_z\phi\|_{0;m_\theta,m_x}
e^{-\rho_0t}
(1+\|\theta\|^{m_\star+2})
(1+\|z\|^{m_x+1}).
\end{align*}
Finally, Young's inequality gives
\[
(1+\|\theta\|^{m_\star+2})
(1+\|z\|^{m_x+1})
\leq
C \mathcal W_{a_0,b_0}(\theta,z),
\]
with
\[
a_0=2(m_\star+2),
\qquad
b_0=2(m_x+1).
\]
This proves the claim.
\end{proof}

The next lemma establishes stability of the first- and second-order
$z$-derivatives of $\mathcal P_t^z$.  The subsequent
parameter-derivative estimates follow from a Duhamel transfer formula
that reduces their control to these spatial derivative estimates.

\begin{lemma}\label{lem:DEx}
Suppose that Assumptions~\ref{ass:dissx},
\ref{ass:driftfroz}, and \ref{ass:poly} hold. Then there exist
$C_z<\infty$ and polynomial orders
\[
a_z=2\max\{m_\theta,m_x\},
\qquad
b_z=2m_x,
\]
independent of $t,\theta,z$, and $\phi$, such that, for every
$\phi\in C^2_{m_\theta,m_x}$,
$t\geq0$, $\theta\in\mathbb R^{d_\theta}$, and
$z\in\mathbb R^{Nd_x}$,
\begin{align}
&
\|\nabla_z\mathcal P_t^z\phi(\theta,z)\|^2
+
\|\nabla_z^2\mathcal P_t^z\phi(\theta,z)\|_F^2\leq
C_z
\|\phi\|_{2;m_\theta,m_x}^2
e^{-2\widetilde\kappa t}
\mathcal W_{a_z,b_z}(\theta,z).
\label{eq:frozen-z-derivatives}
\end{align}
\end{lemma}

\begin{proof}
Let us begin by considering $f_t=\mathcal{P}^z_t \phi$ and observe that,
\begin{equation*}
(\partial_t-\mathcal{G}_z)\|\nabla_z f_t\|^2 = 2 \langle\nabla_z f_t,  \nabla_z \mathcal{G}_z f_t - \mathcal{G}_z\nabla_zf_t\rangle - 2\|\nabla_z^2 f_t\|_F^2.
\end{equation*}
Now,
\begin{equation*}
\nabla_z \mathcal{G}_z f_t - \mathcal{G}_z \nabla_z f_t = - \nabla^2_z E_N \nabla_z f_t,
\end{equation*}
from which,
\begin{equation}\label{eq:nablaf_t}
(\partial_t -\mathcal{G}_z) \|\nabla_z f_t\|^2 \leq -2 \langle\nabla_z f_t, \nabla_z^2E_N \nabla_zf_t \rangle - 2 \|\nabla_z^2 f_t\|_F^2.
\end{equation}

For the second order gradients we similarly observe,
\begin{equation*}
(\partial_t - \mathcal{G}_z) \|\nabla^2_z f_t\|^2 = 2\Tr( \nabla_z^2 f_t (  \nabla_z^2 \mathcal{G}_z f_t - \mathcal{G}_z \nabla_z^2 f_t)^\top) - 2 \|\nabla_z^3 f_t\|_F^2.
\end{equation*}
Further,
\begin{equation*}
\nabla_z^2 \mathcal{G}_z f_t - \mathcal{G}_z \nabla_z^2 f_t  = - \nabla_z^3E_N \nabla_z f_t - \nabla_z^2 E_N \nabla_z^2 f_t - (\nabla_z^2E_N\nabla_z^2f_t)^\top,
\end{equation*}
therefore,
\begin{equation}\label{eq:nabla2f_t}
(\partial_t - \mathcal{G}_z) \|\nabla^2_z f_t\|^2 \leq -2\Tr(\nabla_z^2 f_t (\nabla_z^3 E_N\nabla_z f_t + 2 \nabla_z^2 f_t \nabla_z^2 E_N)^\top) -2 \|\nabla_z^3f_t\|_F^2.
\end{equation}
Now note that by combining \eqref{eq:nablaf_t} and \eqref{eq:nabla2f_t},
\begin{align*}
(\partial_t-\mathcal G_z)
\bigl(\|\nabla_zf_t\|^2+\|\nabla_z^2f_t\|_F^2\bigr)&\leq
-2\Big(
\langle\nabla_zf_t,\nabla_z^2E_N\nabla_zf_t\rangle
+\Tr\!\bigl(\nabla_z^2f_t
(\nabla_z^3E_N\nabla_zf_t)^\top\bigr)\\
&\qquad
+2\Tr\!\bigl(\nabla_z^2f_t\nabla_z^2f_t\nabla_z^2E_N\bigr)
+\|\nabla_z^2f_t\|_F^2+\|\nabla_z^3f_t\|_F^2
\Big)\\
&\leq
-2\widetilde\kappa
\bigl(\|\nabla_zf_t\|^2+\|\nabla_z^2f_t\|_F^2\bigr)
\end{align*}
where the last line follows from \ref{ass:driftfroz}.

By Prop.~3.4 in \cite{Crisan_Ottobre_2016}, this gives us
\begin{equation*}
\partial_s \mathcal{P}^z_{t-s} \big(\|\nabla_z f_s \|^2 + \|\nabla_z^2 f_s\|^2 \big) \leq -2\tilde{\kappa} \mathcal{P}^z_{t-s}\big(\|\nabla_z f_s \|^2 + \|\nabla_z^2 f_s\|_F^2 \big).
\end{equation*}
Applying Gronwall's Lemma, we observe,
\begin{equation}\label{eq:gradxbound}
\mathcal{P}^z_{t-s} \big(\|\nabla_z f_s \|^2 + \|\nabla_z^2 f_s\|^2 \big) \leq e^{-2\tilde{\kappa} s}\mathcal{P}^z_t \big(\|\nabla_z f_0 \|^2 + \|\nabla_z^2 f_0\|_F^2 \big)
\end{equation}
and let us also recall that by \ref{ass:poly}, Lemma~\ref{lem:boundedmom} and the positivity of the Markov semigroup,
\begin{align*}
\mathcal P_t^z
\bigl(\|\nabla_zf_0\|^2+\|\nabla_z^2f_0\|_F^2\bigr)
&\leq
3\Big(
\|\nabla_z\phi\|_{0;m_\theta,m_x}^2
+\|\nabla_z^2\phi\|_{0;m_\theta,m_x}^2
\Big)
\Big(
1+\widetilde\gamma^\theta_{2m_x}
+\|\theta\|^{2m_\theta}+\|z\|^{2m_x}
\Big).
\end{align*}
Substituting this expression into \eqref{eq:gradxbound} and setting $s=t$, the desired result is obtained. For the first order gradient the same proof can be followed, ignoring all second order gradients.
\end{proof}

We are now in the position to establish exponentially stable derivative estimates for the first and second order gradients of the semigroup $\mathcal{P}^z_t$ in $\theta$. In particular, by showing stability of the gradients around their limit, we are able to control the gradients of the solution to the Poisson equation \eqref{eq:poissoneq}. 

\begin{theorem}[First derivatives in $\theta$]
\label{thm:SESDE}
Suppose that Assumptions~\ref{ass:dissx},
\ref{ass:driftfroz}, and \ref{ass:poly} hold.
Then there exist $C_1,\rho_1>0$ and polynomial orders
$a_1,b_1\geq0$. These quantities are independent of
$t,\theta,z$, and $\phi$, and, for every
$\phi\in C^4_{m_\theta,m_x}$,
\begin{equation}\label{eq:first-parameter-stability}
\left\|
\nabla_\theta
\left(
\mathcal P_t^z\phi-\mathcal P_\infty^z\phi
\right)(\theta,z)
\right\|
\leq
C_1 e^{-\rho_1t}
\|\phi\|_{4;m_\theta,m_x}
\mathcal W_{a_1,b_1}(\theta,z).
\end{equation}
One may take
\[
\rho_1
=
\frac12
\min\left\{
\rho_0,\widetilde\kappa
\right\},
\qquad
\rho_0=\widetilde\kappa.
\]
Moreover, the convergence is locally uniform in $\theta$, and hence
the limit in \eqref{eq:first-parameter-stability} is
$\nabla_\theta\mathcal P_\infty^z\phi(\theta)$.
\end{theorem}

\begin{proof}
Set $u_s=\mathcal P_s^z\phi$ and
\[
\psi_s
:=
(\nabla_\theta\mathcal G_z)u_s
=
-\bigl(\nabla_\theta\nabla_zE_N\bigr)\nabla_zu_s.
\]
The product rule gives
\[
\nabla_z\psi_s
=
-\bigl(\nabla_z\nabla_\theta\nabla_zE_N\bigr)\nabla_zu_s
-\bigl(\nabla_\theta\nabla_zE_N\bigr)\nabla_z^2u_s.
\]
Assumption~\ref{ass:poly}, Lemma~\ref{lem:DEx}, and
\eqref{eq:weight-product} therefore imply that there exist
$C_\psi<\infty$ and polynomial orders $a_\psi,b_\psi\geq0$ such that
\begin{equation}\label{eq:first-transfer-forcing}
\|\psi_s(\theta,z)\|
+
\|\nabla_z\psi_s(\theta,z)\|_F
\leq
C_\psi e^{-\widetilde\kappa s}
\|\phi\|_{4;m_\theta,m_x}
\mathcal W_{a_\psi,b_\psi}(\theta,z).
\end{equation}
Here and below the polynomial orders may be enlarged a finite number
of times without changing notation.

The first parameter transfer formula is
\begin{equation}\label{eq:transferform}
\nabla_\theta\mathcal P_t^z\phi
=
\mathcal P_t^z\nabla_\theta\phi
+
\int_0^t
\mathcal P_{t-s}^z\psi_s\,\mathrm ds.
\end{equation}
The invariant moments from Lemma~\ref{lem:boundedmom}, together with
\eqref{eq:first-transfer-forcing}, show that
$\int_0^\infty\mathcal P_\infty^z\psi_s\,\mathrm ds$ is absolutely
convergent, locally uniformly in $\theta$. Consequently,
differentiation under the invariant integral is justified and
\[
\nabla_\theta\mathcal P_\infty^z\phi
=
\mathcal P_\infty^z\nabla_\theta\phi
+
\int_0^\infty
\mathcal P_\infty^z\psi_s\,\mathrm ds.
\]
Subtracting this identity from \eqref{eq:transferform} yields
\begin{align}
\nabla_\theta
(\mathcal P_t^z-\mathcal P_\infty^z)\phi
={}&
(\mathcal P_t^z-\mathcal P_\infty^z)\nabla_\theta\phi
\label{eq:first-param-I}\\
&+
\int_0^t
(\mathcal P_{t-s}^z-\mathcal P_\infty^z)\psi_s\,\mathrm ds
\label{eq:first-param-II}\\
&-
\int_t^\infty
\mathcal P_\infty^z\psi_s\,\mathrm ds.
\label{eq:first-param-III}
\end{align}

Lemma~\ref{lem:SES}, applied to $\nabla_\theta\phi$, bounds
\eqref{eq:first-param-I} by
\[
C e^{-\rho_0t}
\|\phi\|_{4;m_\theta,m_x}
\mathcal W_{a_1,b_1}(\theta,z).
\]
Applying the same lemma at the enlarged orders
$a_\psi,b_\psi$ and using \eqref{eq:first-transfer-forcing} gives
\begin{align*}
\left\|
\int_0^t
(\mathcal P_{t-s}^z-\mathcal P_\infty^z)\psi_s\,\mathrm ds
\right\|
&\leq
C\|\phi\|_{4;m_\theta,m_x}
\mathcal W_{a_1,b_1}(\theta,z)
\int_0^t
e^{-\rho_0(t-s)}e^{-\widetilde\kappa s}\,\mathrm ds.
\end{align*}
For any $\alpha,\beta>0$,
\[
\int_0^t e^{-\alpha(t-s)}e^{-\beta s}\,\mathrm ds
\leq
C_{\alpha,\beta}
e^{-\frac12\min\{\alpha,\beta\}t}.
\]
Thus \eqref{eq:first-param-II} has the required decay with rate
$\rho_1$. Finally, Lemma~\ref{lem:boundedmom} and
\eqref{eq:first-transfer-forcing} give
\[
\left\|
\int_t^\infty
\mathcal P_\infty^z\psi_s\,\mathrm ds
\right\|
\leq
C e^{-\widetilde\kappa t}
\|\phi\|_{4;m_\theta,m_x}
\mathcal W_{a_1,b_1}(\theta,z).
\]
Combining the three bounds proves \eqref{eq:first-parameter-stability}.
The same estimates are locally uniform in $\theta$, which proves the
last assertion.
\end{proof}

\begin{lemma}[Mixed frozen-semigroup derivatives]
\label{lem:mixed-frozen-derivatives}
Under Assumptions~\ref{ass:dissx},
\ref{ass:driftfroz}, and \ref{ass:poly}, there exist
$C_{\mathrm{mix}},\rho_{\mathrm{mix}}>0$ and polynomial orders
$a_{\mathrm{mix}},b_{\mathrm{mix}}\geq0$ such that, for every
$\phi\in C^4_{m_\theta,m_x}$,
\begin{align}
\|\nabla_\theta\nabla_z
\mathcal P_t^z\phi(\theta,z)\|_F
+
\|\nabla_\theta\nabla_z^2
\mathcal P_t^z\phi(\theta,z)\|_F
\leq
C_{\mathrm{mix}}
e^{-\rho_{\mathrm{mix}}t}
\|\phi\|_{4;m_\theta,m_x}
\mathcal W_{a_{\mathrm{mix}},b_{\mathrm{mix}}}(\theta,z).
\label{eq:mixed-frozen-derivatives}
\end{align}
\end{lemma}

\begin{proof}
We first establish the third spatial derivative estimate needed below.
The calculation is first performed for smooth coefficients with bounded
derivatives.  The general polynomial-growth case follows by stopping the
frozen process on balls and then removing the stopping times using
Lemma~\ref{lem:boundedmom}.
Let $Z_t^{\theta,z}$ denote the stochastic flow of the frozen
diffusion and set
\[
J_t=\nabla_zZ_t^{\theta,z},
\qquad
K_t=\nabla_z^2Z_t^{\theta,z},
\qquad
L_t=\nabla_z^3Z_t^{\theta,z}.
\]
Taking $\eta=0$ in Assumption~\ref{ass:driftfroz} shows that
\[
\nabla_z^2E_N(\theta,z)\succeq
\widetilde\kappa I.
\]
The variational equations of the additive-noise flow are
\begin{align*}
\dot J_t
&=
-\nabla_z^2E_N(\theta,Z_t)J_t,
&
J_0&=I,\\
\dot K_t
&=
-\nabla_z^2E_N(\theta,Z_t)K_t
-\nabla_z^3E_N(\theta,Z_t)[J_t,J_t],
&
K_0&=0,\\
\dot L_t
&=
-\nabla_z^2E_N(\theta,Z_t)L_t
-\nabla_z^4E_N(\theta,Z_t)[J_t,J_t,J_t]\\
&\quad
-3\nabla_z^3E_N(\theta,Z_t)[K_t,J_t],
&
L_0&=0.
\end{align*}
It follows that $\|J_t\|_{\mathrm{op}}\leq
e^{-\widetilde\kappa t}$.  Variation of constants in the equations
for $K_t$ and $L_t$, followed by Hölder's inequality, the polynomial
bounds in Assumption~\ref{ass:poly}, and the moment estimate in
Lemma~\ref{lem:boundedmom} at sufficiently high orders, gives constants
$C_{\mathrm{fl}},\rho_{\mathrm{fl}}>0$ and polynomial orders
$a_{\mathrm{fl}},b_{\mathrm{fl}}$ such that
\begin{equation}\label{eq:flow-derivative-bound}
\mathbb E\!\left[
\|K_t\|_F^2+\|L_t\|_F^2
\right]^{1/2}
\leq
C_{\mathrm{fl}}e^{-\rho_{\mathrm{fl}}t}
\mathcal W_{a_{\mathrm{fl}},b_{\mathrm{fl}}}(\theta,z).
\end{equation}
For completeness, each term in the variation-of-constants formula is
a convolution of $e^{-\widetilde\kappa(t-s)}$ with either
$e^{-2\widetilde\kappa s}$ or a product of this exponential with a
polynomial in $Z_s$.  Lemma~\ref{lem:boundedmom} controls the latter
uniformly in $s$, and the convolution inequality used in the proof of
Theorem~\ref{thm:SESDE} yields \eqref{eq:flow-derivative-bound}.

Differentiating
$\mathcal P_t^z\phi(\theta,z)=\mathbb E[\phi(\theta,Z_t^{\theta,z})]$
three times in $z$ gives
\begin{align*}
\nabla_z^3\mathcal P_t^z\phi
=
\mathbb E\big[
\nabla_z^3\phi(\theta,Z_t)[J_t,J_t,J_t]
\big]+
3\mathbb E\big[
\nabla_z^2\phi(\theta,Z_t)[K_t,J_t]
\big]
+
\mathbb E\big[
\nabla_z\phi(\theta,Z_t)L_t
\big].
\end{align*}
Using \eqref{eq:flow-derivative-bound}, Assumption~\ref{ass:poly},
and Lemma~\ref{lem:boundedmom}, we obtain
\begin{equation}\label{eq:third-z-stability}
\|\nabla_z^3\mathcal P_t^z\phi(\theta,z)\|_F
\leq
C_3e^{-\rho_3t}
\|\phi\|_{3;m_\theta,m_x}
\mathcal W_{a_3,b_3}(\theta,z)
\end{equation}
for suitable $C_3,\rho_3>0$ and $a_3,b_3\geq0$.

We now return to the parameter derivative. With
$u_s=\mathcal P_s^z\phi$, set
\[
\psi_s
=
(\nabla_\theta\mathcal G_z)u_s
=
-\bigl(\nabla_\theta\nabla_zE_N\bigr)\nabla_zu_s.
\]
Differentiate this expression up to twice in $z$.
Lemma~\ref{lem:DEx}, \eqref{eq:third-z-stability}, and
Assumption~\ref{ass:poly}, together with the weight-product bound,
then give $C_\psi,\rho_\psi>0$ and polynomial orders
$a_\psi,b_\psi$ such that
\begin{equation}\label{eq:mixed-forcing-c2}
\|\psi_s\|_{2;a_\psi,b_\psi}
\leq
C_\psi e^{-\rho_\psi s}
\|\phi\|_{4;m_\theta,m_x}.
\end{equation}
Indeed, the highest-order term is
\[
\bigl(\nabla_\theta\nabla_zE_N\bigr)
\nabla_z^3u_s,
\]
which is precisely controlled by \eqref{eq:third-z-stability}.

Differentiating the transfer formula \eqref{eq:transferform} under the
integral gives, for $r=1,2$,
\[
\nabla_z^r\nabla_\theta\mathcal P_t^z\phi
=
\nabla_z^r\mathcal P_t^z\nabla_\theta\phi
+
\int_0^t
\nabla_z^r\mathcal P_{t-s}^z\psi_s\,\mathrm ds.
\]
Lemma~\ref{lem:DEx}, applied at the enlarged orders in
\eqref{eq:mixed-forcing-c2}, yields
\begin{align*}
\|\nabla_z^r\nabla_\theta\mathcal P_t^z\phi\|_F
\leq{}&
C e^{-\widetilde\kappa t}
\|\phi\|_{4;m_\theta,m_x}
\mathcal W_{a_{\mathrm{mix}},b_{\mathrm{mix}}}(\theta,z)\\
&+
C\|\phi\|_{4;m_\theta,m_x}
\mathcal W_{a_{\mathrm{mix}},b_{\mathrm{mix}}}(\theta,z)
\int_0^t
e^{-\widetilde\kappa(t-s)}
e^{-\rho_\psi s}\,\mathrm ds.
\end{align*}
The convolution inequality now proves
\eqref{eq:mixed-frozen-derivatives}, with
\[
\rho_{\mathrm{mix}}
=
\frac12\min\{\widetilde\kappa,\rho_\psi\}.
\]
\end{proof}

\begin{theorem}[Second parameter derivatives]
\label{thm:SESDE2}
Suppose that Assumptions~\ref{ass:dissx},
\ref{ass:driftfroz}, and \ref{ass:poly} hold. Then there exist constants
$C_2,\rho_2>0$ and polynomial orders $a_2,b_2\geq0$, independent of
$t,\theta,z$, and $\phi$, such that, for every
$\phi\in C^4_{m_\theta,m_x}$,
\begin{equation}\label{eq:second-parameter-stability}
\left\|
\nabla_\theta^2
\left(
\mathcal P_t^z\phi-\mathcal P_\infty^z\phi
\right)(\theta,z)
\right\|_F
\leq
C_2 e^{-\rho_2t}
\|\phi\|_{4;m_\theta,m_x}
\mathcal W_{a_2,b_2}(\theta,z).
\end{equation}
One may take
\[
\rho_2
=
\frac12
\min\left\{
\rho_0,\widetilde\kappa,\rho_{\mathrm{mix}}
\right\}.
\]
Moreover, the convergence is locally uniform in $\theta$.
The limit is
$\nabla_\theta^2\mathcal P_\infty^z\phi(\theta)$.
\end{theorem}

\begin{proof}
Let $u_s=\mathcal P_s^z\phi$ and
$v_s=\nabla_\theta\mathcal P_s^z\phi$.  Define
\begin{align*}
\chi_s
&:=
(\nabla_\theta^2\mathcal G_z)u_s
+
2(\nabla_\theta\mathcal G_z)v_s\\
&=
-\bigl(\nabla_\theta^2\nabla_zE_N\bigr)\nabla_zu_s
-
2\bigl(\nabla_\theta\nabla_zE_N\bigr)\nabla_zv_s.
\end{align*}
Its $z$-gradient is a finite sum of contractions involving
\[
\nabla_zu_s,\quad
\nabla_z^2u_s,\quad
\nabla_zv_s,\quad
\nabla_z^2v_s,
\]
and derivatives of $E_N$ of order at most four.  Hence
Assumption~\ref{ass:poly}, Lemma~\ref{lem:DEx},
Lemma~\ref{lem:mixed-frozen-derivatives}, and
\eqref{eq:weight-product} imply that, for some
$C_\chi,\rho_\chi>0$ and polynomial orders
$a_\chi,b_\chi$,
\begin{equation}\label{eq:second-transfer-forcing}
\|\chi_s(\theta,z)\|
+
\|\nabla_z\chi_s(\theta,z)\|_F
\leq
C_\chi e^{-\rho_\chi s}
\|\phi\|_{4;m_\theta,m_x}
\mathcal W_{a_\chi,b_\chi}(\theta,z),
\end{equation}
where one may take
$\rho_\chi=\min\{\widetilde\kappa,\rho_{\mathrm{mix}}\}$.

Differentiating the first transfer formula once more gives the
symmetric second-order formula
\begin{equation}\label{eq:second-transfer-corrected}
\nabla_\theta^2\mathcal P_t^z\phi
=
\mathcal P_t^z\nabla_\theta^2\phi
+
\int_0^t
\mathcal P_{t-s}^z\chi_s\,\mathrm ds.
\end{equation}
The factor two in the definition of $\chi_s$ represents the two
symmetric cross terms.  By
\eqref{eq:second-transfer-forcing} and the frozen invariant moment
bounds, the integral
$\int_0^\infty\mathcal P_\infty^z\chi_s\,\mathrm ds$ converges
absolutely and locally uniformly in $\theta$. Therefore
\[
\nabla_\theta^2\mathcal P_\infty^z\phi
=
\mathcal P_\infty^z\nabla_\theta^2\phi
+
\int_0^\infty
\mathcal P_\infty^z\chi_s\,\mathrm ds.
\]
Subtracting this identity from \eqref{eq:second-transfer-corrected}
gives
\begin{align*}
\nabla_\theta^2
(\mathcal P_t^z-\mathcal P_\infty^z)\phi
={}&
(\mathcal P_t^z-\mathcal P_\infty^z)\nabla_\theta^2\phi+
\int_0^t
(\mathcal P_{t-s}^z-\mathcal P_\infty^z)\chi_s\,\mathrm ds-
\int_t^\infty
\mathcal P_\infty^z\chi_s\,\mathrm ds.
\end{align*}
Lemma~\ref{lem:SES} bounds the first term by
$Ce^{-\rho_0t}\|\phi\|_{4;m_\theta,m_x}\mathcal W_{a_2,b_2}$.
Applying the same lemma at the enlarged orders
$a_\chi,b_\chi$ and using
\eqref{eq:second-transfer-forcing}, the second term is bounded by
\[
C\|\phi\|_{4;m_\theta,m_x}\mathcal W_{a_2,b_2}(\theta,z)
\int_0^t
e^{-\rho_0(t-s)}e^{-\rho_\chi s}\,\mathrm ds.
\]
The convolution inequality used in Theorem~\ref{thm:SESDE} gives the
rate $\rho_2$. The invariant moment estimate and
\eqref{eq:second-transfer-forcing} bound the final tail by
$Ce^{-\rho_\chi t}\|\phi\|_{4;m_\theta,m_x}\mathcal W_{a_2,b_2}$.
Combining the three estimates proves
\eqref{eq:second-parameter-stability}; their local uniformity proves
the final assertion.
\end{proof}

\subsection{Poisson corrector}

The preceding semigroup estimates now allow us to construct the
centred solution of \eqref{eq:poissoneq}.

\begin{lemma}[Frozen Poisson equation]\label{lem:poisson-solution}
Suppose that Assumptions~\ref{ass:dissx}, \ref{ass:driftfroz}, and
\ref{ass:poly} hold.  For fixed $\theta$, set
\[
\pi_\theta^z:=p_\theta^{\otimes N}
\]
and define
\[
g_\theta(z)
:=
\frac1N\left(
\nabla_\theta E_N(\theta,z)
-
\int\nabla_\theta E_N(\theta,w)\,\pi_\theta^z(\mathrm dw)
\right).
\]
Then
\begin{equation}\label{eq:solution}
\Phi(\theta,z)
=
-\int_0^\infty\mathcal P_s^zg_\theta(z)\,\mathrm ds
\end{equation}
is a solution of the Poisson equation
\[
\mathcal G_z\Phi(\theta,\cdot)=g_\theta.
\]
It is centred in the sense that
\begin{equation}\label{eq:poisson-centering}
\int\Phi(\theta,z)\,\pi_\theta^z(\mathrm dz)=0.
\end{equation}
Moreover, it is the unique centred solution in the class of
$C^2$ functions whose $z$-gradient has polynomial growth.
\end{lemma}

\begin{proof}
The forcing is centred:
\[
\int g_\theta(z)\,\pi_\theta^z(\mathrm dz)=0.
\]
Lemma~\ref{lem:SES} shows that
$\mathcal P_s^zg_\theta(z)$ converges exponentially to zero with a
polynomial weight.  Hence the integral in \eqref{eq:solution}
converges absolutely.  Its required regularity follows from the
semigroup derivative estimates in Lemma~\ref{lem:DEx}.

For $T>0$, the semigroup equation gives
\[
-\int_0^T\mathcal G_z\mathcal P_s^zg_\theta\,\mathrm ds
=
-\int_0^T\partial_s\mathcal P_s^zg_\theta\,\mathrm ds
=
g_\theta-\mathcal P_T^zg_\theta.
\]
Letting $T\to\infty$ and using Lemma~\ref{lem:SES} yields
\[
\mathcal G_z\Phi=g_\theta.
\]
Furthermore, invariance of $\pi_\theta^z$ and Fubini's theorem give
\[
\int\Phi(\theta,z)\,\pi_\theta^z(\mathrm dz)
=
-\int_0^\infty
\int\mathcal P_s^zg_\theta(z)\,\pi_\theta^z(\mathrm dz)\,\mathrm ds
=0,
\]
which proves \eqref{eq:poisson-centering}.

Finally, let $\Phi_1$ and $\Phi_2$ be two centred solutions in the
stated class and put $h=\Phi_1-\Phi_2$.  Then
$\mathcal G_zh=0$, so the semigroup equation gives
$\mathcal P_t^zh=h$.  Applying Lemma~\ref{lem:SES} to $h$ and using
its centring,
\[
|h(\theta,z)|
=
\left|
\mathcal P_t^zh(\theta,z)
-
\int h(\theta,w)\,\pi_\theta^z(\mathrm dw)
\right|
\longrightarrow0
\qquad\text{as }t\to\infty.
\]
Thus $h=0$, proving uniqueness.
\end{proof}

The representation \eqref{eq:solution} is classical for elliptic
Poisson equations; see
\cite{Pardoux_Veretennikov_2001,Pavliotis2008}.  The first- and
second-order parameter stability estimates established above now
give corresponding polynomial bounds for the corrector.

\begin{proposition}[Polynomial bounds for the Poisson corrector]
\label{prop:poisson-bounds}
Suppose that the assumptions of
Lemma~\ref{lem:poisson-solution},
Theorem~\ref{thm:SESDE}, and
Theorem~\ref{thm:SESDE2} hold. Then, for each
$j\in\{0,1,2\}$, there exist constants
$C_{\Phi,j}<\infty$ and polynomial orders
$a_{\Phi,j},b_{\Phi,j}\geq0$, independent of $\theta$ and $z$, such
that
\begin{equation}\label{eq:poisson-corrector-polynomial}
\|\nabla_\theta^j\Phi(\theta,z)\|_F
\leq
C_{\Phi,j}
\mathcal W_{a_{\Phi,j},b_{\Phi,j}}(\theta,z).
\end{equation}
Here $\nabla_\theta^0\Phi:=\Phi$.
\end{proposition}

\begin{proof}
The centred semigroup representation from Lemma~\ref{lem:poisson-solution} is
\begin{equation}\label{eq:poisson-centered-semigroup}
\Phi(\theta,z)
=
-\frac1N
\int_0^\infty
\left(
\mathcal P_s^z\nabla_\theta E_N(\theta,z)
-
\mathcal P_\infty^z\nabla_\theta E_N(\theta)
\right)\,\mathrm ds.
\end{equation}
Assumption~\ref{ass:poly} ensures that $F$ belongs to the weighted
smooth class required by Lemma~\ref{lem:SES} and
Theorems~\ref{thm:SESDE}--\ref{thm:SESDE2}.

For $j=0$, Lemma~\ref{lem:SES} gives
\[
\left\|
(\mathcal P_s^z-\mathcal P_\infty^z)\nabla_\theta E_N(\theta,z)
\right\|
\leq
C e^{-\rho_0s}
\mathcal W_{a_0,b_0}(\theta,z).
\]
The right-hand side is integrable on $\mathbb R_+$, so
\eqref{eq:poisson-centered-semigroup} gives
\eqref{eq:poisson-corrector-polynomial} with
$a_{\Phi,0}=a_0$ and $b_{\Phi,0}=b_0$.

For $j=1$, Theorem~\ref{thm:SESDE} gives
\[
\left\|
\nabla_\theta
(\mathcal P_s^z-\mathcal P_\infty^z)\nabla_\theta E_N(\theta,z)
\right\|
\leq
C e^{-\rho_1s}
\mathcal W_{a_1,b_1}(\theta,z),
\]
and for $j=2$, Theorem~\ref{thm:SESDE2} gives the analogous estimate
with $C e^{-\rho_2s}\mathcal W_{a_2,b_2}$.  These bounds are integrable in
$s$ and locally uniform in $\theta$.  Dominated convergence therefore
permits one and two parameter derivatives to be passed through the
integral in \eqref{eq:poisson-centered-semigroup}.  Integrating the
three exponential bounds proves the proposition, with
\[
C_{\Phi,j}
=
\frac{C_j}{N\rho_j},
\qquad
a_{\Phi,j}=a_j,
\qquad
b_{\Phi,j}=b_j,
\quad j=0,1,2,
\]
where $C_0,\rho_0,a_0,b_0$ refer to Lemma~\ref{lem:SES}.
\end{proof}

Proposition~\ref{prop:poisson-bounds} provides the polynomial bounds
on $\Phi$, $\nabla_\theta\Phi$, and $\nabla_\theta^2\Phi$ used below.
All constants in these bounds are independent of $(\theta,z)$; the
parameter dependence of the frozen moment constants has been absorbed
into the weights in Lemma~\ref{lem:SES}.
This result is key in the proof, as we use the linearity of the
Poisson equation to decompose the difference between the multiscale
semigroup $\mathcal P_t^\varepsilon$ and the averaged semigroup
$\mathcal P_t^{\mathrm{av}}$ into $\Phi$ and
$\nabla_\theta\mathcal P_t^{\mathrm{av}}$.

\section{Averaged setting}\label{sec:averaged} As mentioned in Sec.~\ref{sec:background}, our main goal is to leverage the properties of the averaged dynamics, in the setting of $\varepsilon \to 0$. In particular, we consider,
\begin{equation*}
\mathrm{d}\theta^{\textnormal{av}}_t = \frac{1}{N}\int \nabla_\theta E_N(\theta^{\textnormal{av}}_t, z) p_{\theta^{\textnormal{av}}_t}^{\otimes N}(\mathrm{d}z)\mathrm{d}t + \sqrt{\frac{2}{N}}\mathrm{d}W_t^\theta, \tag{\ref{eq:averagedtheta}}
\end{equation*}
for the averaged process. This result follows from classical averaging results, as may be found in \cite{Pardoux_Veretennikov_2003, Pavliotis2008}, but here we are interested in quantifying this behaviour for positive $\varepsilon$ and comparing the invariant measure of \eqref{eq:sde} with that of \eqref{eq:averagedtheta} above, $\pi_N^{\mathrm{av}}$.

To begin, we introduce a classical result for the invariant measure from the study of overdamped Langevin diffusions and further show that this measure concentrates onto the MLE solution from \eqref{eq:mle_empirical}.

\begin{theorem}\label{thm:statmeasure}
Let
\[
V(\theta)
=
\widehat E(\theta)+\log Z_\theta,
\qquad
\widehat E(\theta)
=
\frac1M\sum_{j=1}^M E(\theta,y_j).
\]
Suppose that Assumption~\ref{ass:dissav} holds. Then the averaged
process \eqref{eq:averagedtheta} admits the invariant probability
measure
\[
\pi_N^{\mathrm{av}}(\mathrm{d}\theta)
=
\frac{1}{\mathcal Z_N}
e^{-NV(\theta)}\,\mathrm{d}\theta,
\qquad
\mathcal Z_N
=
\int_{\mathbb R^{d_\theta}}
e^{-NV(\theta)}\,\mathrm{d}\theta.
\]

Suppose additionally that \(V\) has a unique minimiser
\(\theta_\star\), and that there exist constants
\(a,\mu,\rho>0\) and an integer \(m_V\geq2\) such that
\begin{equation}\label{eq:two-sided-local-growth}
a\|\theta-\theta_\star\|^{m_V}
\leq
V(\theta)-V(\theta_\star)
\leq
\mu\|\theta-\theta_\star\|^{m_V},
\quad \mbox{ for } \quad
\|\theta-\theta_\star\|\leq\rho.
\end{equation}
Then there exist a constant \(C_{\mathrm{conc}}<\infty\), independent
of \(N\), and \(N_0\in\mathbb N\) such that
\[
W_2(\pi_N^{\mathrm{av}},\delta_{\theta_\star})
\leq
C_{\mathrm{conc}}N^{-1/m_V},
\qquad
N\geq N_0.
\]
In particular, there exist \(R>\rho\), \(\lambda>0\), and
\(\underline a>0\) such that one may take
\begin{align}
C_{\mathrm{conc}}^2
={}&
d_\theta e^{1/2}(2\mu)^{d_\theta/m_V}
\Bigg[
\frac{
\Gamma\!\left(\frac{d_\theta+2}{m_V}\right)
}{
m_V
\underline a^{(d_\theta+2)/m_V}
}
+
\frac{
\Gamma\!\left(\frac{d_\theta+2}{2}\right)
}{
2\lambda^{(d_\theta+2)/2}
}
\Bigg].
\label{eq:concentration-constant}
\end{align}
\end{theorem}

\begin{proof}
Recall that
\[
\frac1N
\int\nabla_\theta E_N(\theta,z)
p_\theta^{\otimes N}(\mathrm{d}z)
=
-\nabla V(\theta).
\]
Consequently, the averaged process is
\[
\mathrm{d}\theta_t^{\mathrm{av}}
=
-\nabla V(\theta_t^{\mathrm{av}})\,\mathrm{d}t
+
\sqrt{\frac2N}\,\mathrm{d}W_t^\theta.
\]
Its generator is
\[
\mathcal G_{\mathrm{av}}
=
-\langle\nabla V,\nabla\rangle
+
\frac1N\Delta.
\]
The density proportional to \(e^{-NV}\) is therefore invariant.
Assumption~\ref{ass:dissav} implies that \(V\) is coercive, and hence
\[
\mathcal Z_N
=
\int_{\mathbb R^{d_\theta}}e^{-NV(\theta)}\,\mathrm{d}\theta
<\infty.
\]
Set \(v(\theta)=V(\theta)-V(\theta_\star)\) and write \(d=d_\theta\). Since the second marginal in every coupling between \(\pi_N^{\mathrm{av}}\) and \(\delta_{\theta_\star}\) is deterministic,
\[
W_2^2(\pi_N^{\mathrm{av}},\delta_{\theta_\star})
=
\frac{
\displaystyle
\int_{\mathbb R^d}
\|\theta-\theta_\star\|^2e^{-Nv(\theta)}
\,\mathrm{d}\theta
}{
\displaystyle
\int_{\mathbb R^d}
e^{-Nv(\theta)}\,\mathrm{d}\theta
}.
\]

We first lower-bound the denominator. Let \(s_N=(2\mu N)^{-1/m_V}\).
Choose
\[
N_0
=
\max\left\{
1,
\left\lceil
\frac{1}{2\mu\rho^{m_V}}
\right\rceil
\right\}.
\]
Then \(s_N\leq\rho\) for \(N\geq N_0\). By the upper bound in
\eqref{eq:two-sided-local-growth},
\[
Nv(\theta)\leq\frac12
\qquad
\text{whenever }
\|\theta-\theta_\star\|\leq s_N.
\]
Writing \(\omega_d=|B(0,1)|\), it follows that
\begin{align}
\int_{\mathbb R^d}e^{-Nv(\theta)}\,\mathrm{d}\theta
&\geq
\int_{B(\theta_\star,s_N)}
e^{-Nv(\theta)}\,\mathrm{d}\theta
\nonumber\\
&\geq
e^{-1/2}\omega_d s_N^d
\nonumber\\
&=
e^{-1/2}\omega_d
(2\mu)^{-d/m_V}
N^{-d/m_V}.
\label{eq:ZN-lower}
\end{align}

We next bound the numerator. Assumption~\ref{ass:dissav} is equivalent
to
\[
\langle\nabla V(\theta),\theta\rangle
\geq
r_{\mathrm{av}}\|\theta\|^2-b_{\mathrm{av}}.
\]
Integrating this inequality along radial rays shows that \(V\) has
quadratic growth at infinity. Consequently, there exist constants
\(R>\rho\) and \(\lambda>0\) such that
\begin{equation}\label{eq:quadratic-tail}
v(\theta)
\geq
\lambda\|\theta-\theta_\star\|^2,
\qquad
\|\theta-\theta_\star\|\geq R.
\end{equation}

The local lower bound in
\eqref{eq:two-sided-local-growth} can also be extended to the ball of
radius \(R\). Indeed, uniqueness of the minimiser and continuity of
\(V\) imply
\[
c_R
:=
\min_{\rho\leq\|\theta-\theta_\star\|\leq R}
v(\theta)
>0.
\]
Define
\[
\underline a
=
\min\left\{
a,\frac{c_R}{R^{m_V}}
\right\}.
\]
Then
\begin{equation}\label{eq:local-lower-extended}
v(\theta)
\geq
\underline a
\|\theta-\theta_\star\|^{m_V},
\qquad
\|\theta-\theta_\star\|\leq R.
\end{equation}

Using \eqref{eq:local-lower-extended}, the contribution inside the
ball is bounded by
\begin{align}
&
\int_{\|\theta-\theta_\star\|\leq R}
\|\theta-\theta_\star\|^2
e^{-Nv(\theta)}
\,\mathrm{d}\theta
\nonumber\\
&\qquad\leq
d\omega_d
\int_0^\infty
r^{d+1}
e^{-N\underline a r^{m_V}}
\,\mathrm{d}r
\nonumber\\
&\qquad=
\frac{d\omega_d}{m_V}
(N\underline a)^{-(d+2)/m_V}
\Gamma\left(\frac{d+2}{m_V}\right).
\label{eq:numerator-local}
\end{align}

Similarly, using \eqref{eq:quadratic-tail},
\begin{align}
&
\int_{\|\theta-\theta_\star\|>R}
\|\theta-\theta_\star\|^2
e^{-Nv(\theta)}
\,\mathrm{d}\theta
\nonumber\\
&\qquad\leq
d\omega_d
\int_0^\infty
r^{d+1}e^{-N\lambda r^2}\,\mathrm{d}r
\nonumber\\
&\qquad=
\frac{d\omega_d}{2}
(N\lambda)^{-(d+2)/2}
\Gamma\left(\frac{d+2}{2}\right).
\label{eq:numerator-tail}
\end{align}

Combining \eqref{eq:ZN-lower},
\eqref{eq:numerator-local}, and
\eqref{eq:numerator-tail} gives
\begin{align*}
W_2^2(\pi_N^{\mathrm{av}},\delta_{\theta_\star})
\leq{}&
d e^{1/2}(2\mu)^{d/m_V}
\frac{
\Gamma\!\left(\frac{d+2}{m_V}\right)
}{
m_V
\underline a^{(d+2)/m_V}
}
N^{-2/m_V}
\\
&+
\frac{
d e^{1/2}(2\mu)^{d/m_V}
\Gamma\!\left(\frac{d+2}{2}\right)
}{
2\lambda^{(d+2)/2}
}
N^{d/m_V-(d+2)/2}.
\end{align*}
Since \(m_V\geq2\),
\[
\frac{d}{m_V}-\frac{d+2}{2}
\leq
-\frac{2}{m_V}.
\]
Therefore, for \(N\geq1\),
\[
N^{d/m_V-(d+2)/2}
\leq
N^{-2/m_V}.
\]
It follows that
\[
W_2^2(\pi_N^{\mathrm{av}},\delta_{\theta_\star})
\leq
C_{\mathrm{conc}}^2N^{-2/m_V},
\]
with \(C_{\mathrm{conc}}\) given by
\eqref{eq:concentration-constant}. Taking square roots proves the
claim.
\end{proof}

\begin{remark}
If \(V\in C^2\) and the minimiser \(\theta_\star\) is non-degenerate,
that is,
\[
\nabla^2V(\theta_\star)\succ0,
\]
then continuity of the Hessian and Taylor's theorem imply that there
exist \(a,\mu,\rho>0\) such that
\[
a\|\theta-\theta_\star\|^2
\leq
V(\theta)-V(\theta_\star)
\leq
\mu\|\theta-\theta_\star\|^2
\]
for \(\|\theta-\theta_\star\|\leq\rho\). Thus
\eqref{eq:two-sided-local-growth} holds with \(m_V=2\), and
\[
W_2(\pi_N^{\mathrm{av}},\delta_{\theta_\star})
\leq
C_{\mathrm{conc}}N^{-1/2}.
\]
\end{remark}

\begin{remark} 
We recall that in the notation of our negative empirical log-likelihood defined in \eqref{eq:loglike}, this implies that $\pi_N^{\mathrm{av}}(\mathrm{d}\theta) \propto e^{-N V(\theta)}\mathrm{d}\theta$. Further, we have shown concentration of this invariant measure, as $N\to\infty$, onto the maximum likelihood solution as defined in \eqref{eq:mle_empirical}. Therefore, once we establish the convergence of our multiscale system to the averaged process (see Sec.~\ref{sec:error}), we will be in the position to prove discretisations of the multiscale system (which result in \gls*{pcd} methods) can indeed approximate the maximum likelihood solutions.
\end{remark}

We next record the ergodicity and moment estimates used below.

\begin{lemma}\label{lem:genavbound}
Given \ref{ass:dissav}, we have
\begin{equation*}
\mathcal{G}_{\textnormal{av}}\|\theta\|^2
\leq
\bar c_N-2r_{\textnormal{av}}\|\theta\|^2,
\end{equation*}
for $\theta\in\mathbb{R}^{d_\theta}$ with
\(\bar c_N=2(b_{\textnormal{av}}+d_\theta/N)\).
\end{lemma}

\begin{proof}
Assumption~\ref{ass:dissav} gives
\[
\left\langle
\frac1N\int\nabla_\theta E_N(\theta,z)
p_\theta^{\otimes N}(\mathrm dz),\theta
\right\rangle
\leq
-r_{\mathrm{av}}\|\theta\|^2+b_{\mathrm{av}}.
\]
Since the diffusion part of \(\mathcal G_{\mathrm{av}}\) is
\(N^{-1}\Delta_\theta\),
\[
\mathcal G_{\mathrm{av}}\|\theta\|^2
=
\frac2N
\left\langle
\int\nabla_\theta E_N(\theta,z)
p_\theta^{\otimes N}(\mathrm dz),\theta
\right\rangle
+\frac{2d_\theta}{N},
\]
which proves the stated inequality.
\end{proof}

\begin{theorem}[Averaged ergodicity]\label{thm:avconv}
Suppose that Assumptions~\ref{ass:dissav} and \ref{ass:poly} hold.
Then the averaged diffusion is non-explosive, Lebesgue-irreducible,
strong Feller, and positive Harris recurrent. Its unique invariant
probability measure is \(\pi_N^{\mathrm{av}}\) from
Theorem~\ref{thm:statmeasure}. Moreover, for every polynomial order
\(k\geq2\) and every measurable \(\phi\) satisfying
\(|\phi|\in C_a^0\),
\[
\mathcal P_t^{\mathrm{av}}\phi(\theta)
\longrightarrow
\int \phi\,\mathrm d\pi_N^{\mathrm{av}}
\qquad\text{as }t\to\infty .
\]
The convergence is geometric in a suitable
\(\mathcal W_k\)-weighted total-variation norm.
\end{theorem}

\begin{proof}
Lemma~\ref{lem:genavbound}, and the analogous calculation for
\(\mathcal W_k\), give a Foster--Lyapunov inequality
\[
\mathcal G_{\mathrm{av}}\mathcal W_k
\leq
-a_k\mathcal W_k+b_k
\]
for suitable \(a_k,b_k>0\). The diffusion matrix
\(2N^{-1}I_{d_\theta}\) is non-degenerate and the drift is locally
smooth under Assumption~\ref{ass:poly}. Hence the transition kernel
has a strictly positive density, the process is strong Feller and
Lebesgue-irreducible, and compact sets are petite. The
Foster--Lyapunov criterion and the Harris theorem therefore give
positive Harris recurrence, uniqueness of the invariant measure, and
geometric ergodicity in a polynomially weighted total-variation norm
\cite{down1995conv}. The invariant density was identified in
Theorem~\ref{thm:statmeasure} as
\(\pi_N^{\mathrm{av}}(\mathrm d\theta)\propto
e^{-NV(\theta)}\mathrm d\theta\), which proves the claim.
\end{proof}

\begin{lemma}\label{lem:boundedmomav}
For the semigroup of the averaged process \eqref{eq:averagedtheta}, $\mathcal{P}^{\textnormal{av}}_t$, satisfies,
\begin{equation*}
\mathcal{P}^{\textnormal{av}}_t \|\theta\|^k \leq e^{-\bar{\alpha}_kt} \|\theta\|^k + \bar{\gamma}_k,
\end{equation*}
where,
\begin{equation*}
\bar{\alpha}_k = \frac{kr_{\textnormal{av}}}{2}, \qquad \bar{\gamma}_k = \left(\frac{2(Nb_{\textnormal{av}}+d_\theta+k-2)}{Nr_{\textnormal{av}}}\right)^\frac{k}{2}
\end{equation*}
for all $\theta\in\mathbb{R}^{d_\theta}$, $t\geq 0$ and $k\geq 2$ under assumption \ref{ass:dissav} and \ref{ass:poly}.
\end{lemma}

As the proof is identical to that in the proof of Lemma~\ref{lem:boundedmom}, it is omitted here. Next we note the form of the direct averaged semigroup assumption used in the corrector estimates. Under Assumption~\ref{ass:avSES}, for every
\(\phi\in C_{m_\theta}^4(\mathbb R^{d_\theta})\),
\(\theta\in\mathbb R^{d_\theta}\), and \(t\geq0\),
\begin{align}
\|\nabla_\theta\mathcal P_t^{\mathrm{av}}\phi(\theta)\|
+
\|\nabla_\theta^2\mathcal P_t^{\mathrm{av}}\phi(\theta)\|_F\leq
\sqrt{2C_{\mathrm{SES}}^{\mathrm{av}}}\,
\|\phi\|_{4;m_\theta}
e^{-\rho_{\mathrm{av}}t}
\mathcal W_{q_{\mathrm{av}}}(\theta).
\label{eq:averaged-derivative-linear}
\end{align}
This follows from applying Cauchy--Schwarz to the two derivatives in
\eqref{eq:averaged-SES} and use
\(\mathcal W_{2q}^{1/2}\leq\mathcal W_q\).

Thus the ergodicity of the averaged process and the derivative
stability required by the averaging proof are logically separate:
the former follows from dissipativity and non-degenerate noise,
whereas the latter is Assumption~\ref{ass:avSES}.

\section{Averaging error bound}\label{sec:error}

Using the Poisson-corrector estimates and the regularity of the
averaged semigroup, we now bound
\(\mathcal P_t^\varepsilon\phi-\mathcal P_t^{\mathrm{av}}\phi\)
for polynomially growing test functions \(\phi\). Since
\(C_c^\infty(\mathbb R^{d_\theta})\) is contained in the test class
below, the estimate implies weak convergence of the slow
\(\theta\)-marginals as \(\varepsilon\to0\) at each fixed time. The
stronger conclusion proved here is that the bound is uniform over
\(t\geq0\).

\begin{lemma}[Bounds for the first-order corrector]
\label{lem:correctorbound}
Suppose that Assumptions~\ref{ass:dissx},
\ref{ass:driftfroz}, \ref{ass:dissav},
\ref{ass:avSES}, and \ref{ass:poly} hold. Then there exist
polynomial orders $q_\theta,q_z\geq0$ and constants
$C_{\mathrm{cor},1},C_{\mathrm{cor},2}<\infty$, independent of $t$,
such that, for every
$\phi\in C_{m_\theta}^4(\mathbb R^{d_\theta})$, the function
\begin{equation}\label{eq:phi1poiss}
\phi_t^1(\theta,z)
=
-
\left\langle
\Phi(\theta,z),
\nabla_\theta\mathcal P_t^{\mathrm{av}}\phi(\theta)
\right\rangle
\end{equation}
satisfies
\begin{equation}\label{eq:phi1-abstract-bound}
|\phi_t^1(\theta,z)|
\leq
C_{\mathrm{cor},1}
\|\phi\|_{4;m_\theta}
e^{-\rho_{\mathrm{av}}t}
\mathcal W_{q_\theta,q_z}(\theta,z)
\end{equation}
and
\begin{equation}\label{eq:corrector-forcing-abstract}
\left|
\left(
\mathcal G_\theta-\partial_t
\right)\phi_t^1(\theta,z)
\right|
\leq
C_{\mathrm{cor},2}
\|\phi\|_{4;m_\theta}
e^{-\rho_{\mathrm{av}}t}
\mathcal W_{q_\theta,q_z}(\theta,z).
\end{equation}
\end{lemma}

\begin{proof}
Set
\[
u_t=\mathcal P_t^{\mathrm{av}}\phi,
\]
and recall that $N^{-1}\mathcal{P}^z_\infty \nabla_\theta E_N = - \nabla_\theta V$.
Thus
\[
\mathcal G_\theta
=
\frac1N \nabla_\theta E_N\cdot\nabla_\theta+\frac1N\Delta_\theta,
\qquad
\mathcal G_{\mathrm{av}}
=
-\nabla_\theta V\cdot\nabla_\theta+\frac1N\Delta_\theta.
\]

Proposition~\ref{prop:poisson-bounds} gives, for $j=0,1,2$,
\begin{equation}\label{eq:corrector-phi-weights}
\|\nabla_\theta^j\Phi(\theta,z)\|_F
\leq
C_{\Phi,j}
\mathcal W_{a_{\Phi,j},b_{\Phi,j}}(\theta,z).
\end{equation}
Moreover, Assumption~\ref{ass:poly}, the frozen invariant moment
bounds, and the limiting transfer formula in
Theorem~\ref{thm:SESDE} imply that $\nabla_\theta E_N$, $\mathcal{P}_\infty^z \nabla_\theta E_N$, and
$\nabla_\theta^2 V$ have polynomial growth. In
particular, there exist $C_F<\infty$ and polynomial orders
$a_F,b_F$ such that
\begin{equation}\label{eq:averaged-coefficient-growth}
\begin{aligned}
\|\nabla_\theta E_N(\theta,z)\|
+\|\mathcal{P}^z_\infty \nabla_\theta E_N(\theta)\|
+\|\nabla_\theta^2 V(\theta)\|_F
\leq
C_F \mathcal W_{a_F,b_F}(\theta,z).
\end{aligned}
\end{equation}
Finally, \eqref{eq:averaged-derivative-linear} yields
\begin{equation}\label{eq:averaged-derivatives-corrector}
\|\nabla_\theta u_t(\theta)\|
+\|\nabla_\theta^2 u_t(\theta)\|_F
\leq
\sqrt{2C_{\mathrm{SES}}^{\mathrm{av}}}\,
\|\phi\|_{4;m_\theta}
e^{-\rho_{\mathrm{av}}t}
\mathcal W_{q_{\mathrm{av}}}(\theta).
\end{equation}

The definition \eqref{eq:phi1poiss}, estimates
\eqref{eq:corrector-phi-weights} and
\eqref{eq:averaged-derivatives-corrector}, and the weight-product
inequality \eqref{eq:weight-product} give
\eqref{eq:phi1-abstract-bound}, after increasing the polynomial
orders if necessary.

It remains to control the forcing term. We record the complete
product-rule calculation. With indices summed over
$\{1,\ldots,d_\theta\}$,
\[
\phi_t^1=-\Phi_i\partial_i u_t,
\qquad
\partial_tu_t
=
b_{\mathrm{av},j}\partial_ju_t+\frac1N\Delta_\theta u_t.
\]
Differentiating these identities and cancelling the two occurrences
of $\Phi_i\partial_i\Delta_\theta u_t/N$ gives
\begin{align}
(\mathcal G_\theta-\partial_t)\phi_t^1
={}&
-\frac1N
\left\langle
\nabla_\theta\Phi\,\nabla_\theta E_N,\nabla_\theta u_t
\right\rangle
-\frac1N
\left\langle
\Delta_\theta\Phi,\nabla_\theta u_t
\right\rangle
\nonumber -\frac2N
\operatorname{Tr}\!\left(
(\nabla_\theta\Phi)^\top \nabla_\theta^2 u_t
\right)
+
\left\langle
(\nabla_\theta V)\Phi,\nabla_\theta u_t
\right\rangle
\nonumber\\
&-\frac1N
\left\langle
\Phi,
\nabla_\theta^2 u_t\bigl(\nabla_\theta E_N-\mathcal{P}_\infty^z \nabla_\theta E_N\bigr)
\right\rangle.
\label{eq:corrector-forcing-identity}
\end{align}
The first three terms in \eqref{eq:corrector-forcing-identity} use
$\nabla_\theta\Phi$ and $\Delta_\theta\Phi$, the fourth uses the
derivative of the averaged drift, and the last contains the centred
energy gradient. Every coefficient in these five terms is controlled
by \eqref{eq:corrector-phi-weights} and
\eqref{eq:averaged-coefficient-growth}, while $\nabla_\theta u_t$ and $\nabla_\theta^2 u_t$ are
controlled by \eqref{eq:averaged-derivatives-corrector}. Repeated
application of \eqref{eq:weight-product} therefore gives polynomial
orders $q_{\theta,2},q_{z,2}$ and
$C_{\mathrm{cor},2}<\infty$ such that
\[
\left|
(\mathcal G_\theta-\partial_t)\phi_t^1(\theta,z)
\right|
\leq
C_{\mathrm{cor},2}\|\phi\|_{4;m_\theta}
e^{-\rho_{\mathrm{av}}t}
\mathcal W_{q_{\theta,2},q_{z,2}}(\theta,z).
\]
Taking the componentwise maximum of the orders in this estimate and
\eqref{eq:phi1-abstract-bound}, and enlarging
$C_{\mathrm{cor},1},C_{\mathrm{cor},2}$ if necessary,
proves the result.
\end{proof}

The proof of the following result adapts the Poisson-equation and strong-exponential-stability framework developed by Crisan et al. \cite{Crisan_Ottobre_2024}, who established uniform-in-time averaging estimates with an \(O(\varepsilon)\) rate for fully coupled slow–fast diffusions. We follow the lines of their argument closely, in order to track the model-specific polynomial weights and hypotheses required by our parametric family.

\begin{theorem}[Uniform-in-time averaging error]
\label{thm:averror}
Suppose that Assumptions~\ref{ass:dissx},
\ref{ass:driftfroz}, \ref{ass:dissav},
\ref{ass:avSES}, \ref{ass:poly}, and
\ref{ass:jointmom} hold. Let $q_\theta,q_z$ be the polynomial orders appearing in
Lemma~\ref{lem:correctorbound}, and suppose that
\eqref{eq:joint-uniform-moments} holds for these orders. Then, for every
$\phi\in C_{m_\theta}^4(\mathbb R^{d_\theta})$, every
$t\geq0$, every $(\theta,z)\in
\mathbb R^{d_\theta}\times\mathbb R^{Nd_x}$, and every
$\varepsilon\in(0,\varepsilon_0]$,
\begin{equation}\label{eq:averaging-error-corrected}
\left|
\mathcal P_t^\varepsilon\phi(\theta,z)
-
\mathcal P_t^{\mathrm{av}}\phi(\theta)
\right|
\leq
\varepsilon
C_{\mathrm{av}}
\|\phi\|_{4;m_\theta}
\mathcal W_{q_\theta,q_z}(\theta,z),
\end{equation}
where
\begin{equation}\label{eq:Cav}
C_{\mathrm{av}}
=
C_{\mathrm{cor},1}\bigl(1+C_{\mathrm{mom}}\bigr)
+
\frac{C_{\mathrm{cor},2}C_{\mathrm{mom}}}{\rho_{\mathrm{av}}},
\end{equation}
with $C_{\mathrm{cor},1},C_{\mathrm{cor},2}$ given in
Lemma~\ref{lem:correctorbound} and
\[
C_{\mathrm{mom}}
=
C_{\mathrm{mom}}(q_\theta,q_z).
\]
In particular, $C_{\mathrm{av}}$ is independent of $t$ and
$\varepsilon\in(0,\varepsilon_0]$.
\end{theorem}

\begin{proof}
For notational convenience, we identify
$\phi:\mathbb R^{d_\theta}\to\mathbb R$ with its extension to the
joint state space, \(\phi(\theta,z):=\phi(\theta)\).   Set \(u_t(\theta) =\mathcal P_t^{\mathrm{av}}\phi(\theta).\)
The Poisson equation is
\[
\mathcal G_z\Phi(\theta,z)
=
\frac1N
\left(
\nabla_\theta E_N(\theta,z)
-
\int
\nabla_\theta E_N(\theta,w)
p_\theta^{\otimes N}(\mathrm dw)
\right).
\]
Define the first-order corrector by
\[
\phi_t^1(\theta,z)
=
-
\left\langle
\Phi(\theta,z),
\nabla_\theta u_t(\theta)
\right\rangle.
\]
Since $\mathcal G_z$ acts only in the $z$-variable,
\begin{align*}
\mathcal G_z\phi_t^1
&=
-
\left\langle
\mathcal G_z\Phi,
\nabla_\theta u_t
\right\rangle\\
&=
-
\frac1N
\left\langle
\nabla_\theta E_N
-
\int\nabla_\theta E_N
p_\theta^{\otimes N}(\mathrm dw),
\nabla_\theta u_t
\right\rangle\\
&=
\left(
\mathcal G_{\mathrm{av}}
-
\mathcal G_\theta
\right)u_t.
\end{align*}
Thus $\phi_t^1$ solves the required Poisson equation for the
first-order corrector.

Now define
\[
r_t^\varepsilon
=
\mathcal P_t^\varepsilon\phi
-
\mathcal P_t^{\mathrm{av}}\phi
-
\varepsilon\phi_t^1.
\]
Using
\[
\partial_t\mathcal P_t^\varepsilon\phi
=
\mathcal G_\varepsilon
\mathcal P_t^\varepsilon\phi,
\qquad
\partial_tu_t
=
\mathcal G_{\mathrm{av}}u_t,
\]
together with
\[
\mathcal G_\varepsilon
=
\mathcal G_\theta+\varepsilon^{-1}\mathcal G_z,
\]
we obtain
\begin{align*}
\partial_t r_t^\varepsilon
&=
\mathcal G_\varepsilon r_t^\varepsilon
+
\left(
\mathcal G_\theta-\mathcal G_{\mathrm{av}}
\right)u_t
+
\mathcal G_z\phi_t^1
+
\varepsilon
\left(
\mathcal G_\theta-\partial_t
\right)\phi_t^1\\
&=
\mathcal G_\varepsilon r_t^\varepsilon
+
\varepsilon
\left(
\mathcal G_\theta-\partial_t
\right)\phi_t^1.
\end{align*}
Moreover,
\( r_0^\varepsilon = -\varepsilon\phi_0^1\). The variation-of-constants formula therefore gives
\begin{align}
r_t^\varepsilon
={}&
-\varepsilon
\mathcal P_t^\varepsilon\phi_0^1
+
\varepsilon
\int_0^t
\mathcal P_{t-s}^\varepsilon
\left[
\left(
\mathcal G_\theta-\partial_s
\right)\phi_s^1
\right]
\,\mathrm ds.
\label{eq:residual-voc}
\end{align}

Consequently,
\begin{align}
&
\left|
\mathcal P_t^\varepsilon\phi(\theta,z)
-
\mathcal P_t^{\mathrm{av}}\phi(\theta)
\right|\leq
\varepsilon|\phi_t^1(\theta,z)|
+
\varepsilon
\left|
\mathcal P_t^\varepsilon\phi_0^1(\theta,z)
\right|+
\varepsilon
\int_0^t
\left|
\mathcal P_{t-s}^\varepsilon
\left[
\left(
\mathcal G_\theta-\partial_s
\right)\phi_s^1
\right](\theta,z)
\right|
\,\mathrm ds.
\label{eq:three-averaging-terms}
\end{align}

By Lemma~\ref{lem:correctorbound},
\[
|\phi_t^1(\theta,z)|
\leq
C_{\mathrm{cor},1}
\|\phi\|_{4;m_\theta}
e^{-\rho_{\mathrm{av}}t}
\mathcal W_{q_\theta,q_z}(\theta,z),
\]
so in particular
\[
|\phi_t^1(\theta,z)|
\leq
C_{\mathrm{cor},1}
\|\phi\|_{4;m_\theta}
\mathcal W_{q_\theta,q_z}(\theta,z).
\]

For the second term, positivity of the Markov semigroup,
Lemma~\ref{lem:correctorbound}, and
Assumption~\ref{ass:jointmom} give
\begin{align*}
\left|
\mathcal P_t^\varepsilon\phi_0^1(\theta,z)
\right|
&\leq
C_{\mathrm{cor},1}\|\phi\|_{4;m_\theta}
\mathcal P_t^\varepsilon
\mathcal W_{q_\theta,q_z}(\theta,z)\\
&\leq
C_{\mathrm{cor},1}C_{\mathrm{mom}}
\|\phi\|_{4;m_\theta}
\mathcal W_{q_\theta,q_z}(\theta,z).
\end{align*}

Likewise, for the integral term,
\begin{align*}
\left|
\mathcal P_{t-s}^\varepsilon
\left[
\left(
\mathcal G_\theta-\partial_s
\right)\phi_s^1
\right](\theta,z)
\right|&\leq
C_{\mathrm{cor},2}\|\phi\|_{4;m_\theta}
e^{-\rho_{\mathrm{av}}s}
\mathcal P_{t-s}^\varepsilon
\mathcal W_{q_\theta,q_z}(\theta,z)\\
&\leq
C_{\mathrm{cor},2}C_{\mathrm{mom}}
\|\phi\|_{4;m_\theta}
e^{-\rho_{\mathrm{av}}s}
\mathcal W_{q_\theta,q_z}(\theta,z).
\end{align*}
Hence
\begin{align*}
\int_0^t
\left|
\mathcal P_{t-s}^\varepsilon
\left[
\left(
\mathcal G_\theta-\partial_s
\right)\phi_s^1
\right](\theta,z)
\right|
\,\mathrm ds \leq
\frac{C_{\mathrm{cor},2}C_{\mathrm{mom}}}{\rho_{\mathrm{av}}}
\|\phi\|_{4;m_\theta}
\mathcal W_{q_\theta,q_z}(\theta,z).
\end{align*}

Substituting these three estimates into
\eqref{eq:three-averaging-terms} gives
\[
\left|
\mathcal P_t^\varepsilon\phi(\theta,z)
-
\mathcal P_t^{\mathrm{av}}\phi(\theta)
\right|
\leq
\varepsilon
\left[
C_{\mathrm{cor},1}(1+C_{\mathrm{mom}})
+
\frac{C_{\mathrm{cor},2}C_{\mathrm{mom}}}{\rho_{\mathrm{av}}}
\right]
\|\phi\|_{4;m_\theta}
\mathcal W_{q_\theta,q_z}(\theta,z),
\]
which is precisely \eqref{eq:averaging-error-corrected}.
\end{proof}

\section{Numerical methods}\label{sec:numerical}
We now analyse two discretisations of the joint diffusion
\eqref{eq:sde}.  Euler--Maruyama gives the \gls*{spcdem} scheme,
whereas the stabilized S--ROCK discretisation gives \gls*{spcd}.
For each method we establish the same sequence of results:
uniform moment and invariant-measure stability, a \gls*{uit} weak
error of order \(\delta\), and the corresponding invariant-measure
bias.  The S--ROCK method is included because the fast drift becomes
increasingly stiff as \(\varepsilon\to0\).

The \gls*{uit} results use three ingredients: the joint semigroup
derivative estimate, a method-specific discrete Lyapunov condition,
and local weak consistency.  

\subsection{Control of the multiscale system}
We first collect the consequences of the joint assumptions for the
exact semigroup and then treat Euler--Maruyama and S--ROCK
separately.

For the remainder of this section we write the joint system as a
single \gls*{sde} on \(\mathbb R^{d_\theta+Nd_x}\):
\begin{equation*}
\mathrm{d}S_t =
f_\varepsilon(S_t)\mathrm{d}t
+\gamma_\varepsilon^{1/2}\mathrm{d}W_t,
\end{equation*}
where $W_t$ is a $d_\theta + Nd_x$-dimensional Brownian motion,
\begin{equation*}
f_\varepsilon(\theta, z) = \begin{pmatrix}\frac{1}{N}\nabla_\theta E_N(\theta,z)\\ -\frac{1}{\varepsilon}\nabla_zE_N(\theta,z)\end{pmatrix}, \qquad \gamma_\varepsilon=\begin{pmatrix} \frac{2}{N}I_{d_\theta} & 0\\
0 & \frac{2}{\varepsilon} I_{Nd_x}  
\end{pmatrix}.
\end{equation*}

Assumption~\ref{ass:lip} guarantees the existence of global strong
solutions. Assumption~\ref{ass:jointmom} supplies the moment,
invariant-measure, and ergodicity properties of the exact joint
process that are used below. The derivative decay required by the
uniform-in-time weak-error arguments is stated directly in
Assumption~\ref{ass:jointSES}.

By the Cauchy--Schwarz inequality, \eqref{eq:joint-SES} in Assumption~\ref{ass:jointSES}
also implies
\begin{equation}\label{eq:joint-derivative-linear}
\sum_{j=1}^3
\|\nabla^j\mathcal{P}_t^\varepsilon\phi(s)\|_F
\leq
\sqrt{3C_{\mathrm{der}}}\,
\|\phi\|_{4;m}e^{-\rho_\varepsilon t}
\mathcal W_m(s).
\end{equation}

\subsection{Euler--Maruyama}
The Euler--Maruyama discretisation of \eqref{eq:sde} gives the
\gls*{spcdem} scheme.  It differs from classical PCD through the
Gaussian parameter noise prescribed by the continuous-time model;
this noise vanishes as \(N\to\infty\).  For a step size
\(\delta>0\), the update is
\begin{align}
\widehat\theta_{n+1}
&=\widehat\theta_n
+\frac{\delta}{N}
\nabla_\theta E_N(\widehat\theta_n,\widehat Z_n)
+\sqrt{\frac{2\delta}{N}}\,\widehat W_n^\theta,
\nonumber\\
\widehat Z_{n+1}
&=\widehat Z_n
-\frac{\delta}{\varepsilon}
\nabla_zE_N(\widehat\theta_n,\widehat Z_n)
+\sqrt{\frac{2\delta}{\varepsilon}}\,\widehat W_n^z.
\label{eq:euler}
\end{align}
where
\(\widehat W_n^\theta=\delta^{-1/2}
(W^\theta_{t_{n+1}}-W^\theta_{t_n})\) and
\(\widehat W_n^z=\delta^{-1/2}
(W^z_{t_{n+1}}-W^z_{t_n})\), with \(t_n=n\delta\).
The next lemma provides the discrete stability needed to turn the
standard local weak error into a \gls*{uit} estimate.

\begin{lemma}[Euler moment stability]
\label{lem:euler-moments}
Suppose that Assumption~\ref{ass:numstability} holds for the
Euler--Maruyama transition operator
$Q_{\delta}^{\mathrm{EM},\varepsilon}$ at order $q\geq2$.
Then, for every \(0<\delta\leq\delta_q\),
\begin{equation}\label{eq:euler-uniform-moments}
\mathbb E_s[\mathcal W_q(\widehat S_n)]
\leq
\frac{c_q^+}{c_q^-}
(1-c_q\delta)^n\mathcal W_q(s)
+\frac{B_q}{c_qc_q^-},
\qquad n\geq0.
\end{equation}

Moreover, the Euler chain is Lebesgue-irreducible, aperiodic, and
positive Harris recurrent. It admits a unique invariant probability
measure $\widehat\pi_{\mathrm{EM}}^{\varepsilon,\delta}$ satisfying
\begin{equation}\label{eq:euler-invariant-moment}
\int \mathcal W_q(s)\,
\widehat\pi_{\mathrm{EM}}^{\varepsilon,\delta}(\mathrm ds)
\leq
\frac{B_q}{c_qc_q^-}.
\end{equation}
It is geometrically ergodic in the $\mathcal W_q$-weighted total-variation
norm. In particular, for every measurable $h$ with $|h|\leq \mathcal W_q$,
\[
(Q_{\delta}^{\mathrm{EM},\varepsilon})^nh(s)
\longrightarrow
\int h\,\mathrm d\widehat\pi_{\mathrm{EM}}^{\varepsilon,\delta}
\qquad\text{as }n\to\infty.
\]
\end{lemma}

\begin{proof}
Iterating \eqref{eq:numerical-lyapunov} and using
$c_q^-\mathcal W_q\leq\mathcal V_q^{\varepsilon,\delta}
\leq c_q^+\mathcal W_q$ gives \eqref{eq:euler-uniform-moments}.
The Euler transition kernel has a strictly positive continuous
Gaussian density because $\gamma_\varepsilon$ is non-degenerate.
Thus the chain is Lebesgue-irreducible and aperiodic, and compact
sets are small. The Foster--Lyapunov inequality
\eqref{eq:numerical-lyapunov} gives positive Harris recurrence,
uniqueness of the invariant measure, and geometric ergodicity.
Integrating the same inequality against the invariant measure and
using its lower comparison with $\mathcal W_q$ gives
\eqref{eq:euler-invariant-moment}.
\end{proof}

\begin{theorem}[Euler--Maruyama uniform-in-time and invariant-measure error]
\label{thm:euler}
Suppose that Assumptions~\ref{ass:lip},
\ref{ass:poly}, \ref{ass:jointmom}, \ref{ass:jointSES}, and
\ref{ass:numstability} hold, with the latter applied to the
Euler--Maruyama transition operator. Let
\[
S_t=(\theta_t^\varepsilon,Z_t^\varepsilon),
\qquad
\widehat S_n=(\widehat\theta_n,\widehat Z_n),
\qquad
t_n=n\delta.
\]
Then there exists a constant $C_{\mathrm{EM}}>0$, independent of
$n$, $t_n$, and $\delta$, such that, for every $\phi\in C_m^4$ and every
\[
0<\delta\leq\delta_{4m}(N,\varepsilon),
\qquad
\rho_\varepsilon\delta\leq1,
\]
we have
\begin{align*}
\left|
\mathbb{E}\!\left[\phi(\widehat S_n)\right]
-
\mathbb{E}\!\left[\phi(S_{t_n})\right]
\right|
\leq
\frac{C_{\mathrm{EM}}\sqrt{C_{\mathrm{der}}}}
{\rho_\varepsilon}
\|\phi\|_{4;m}
\left(1+\|S_0\|^{4m}\right)\delta .
\end{align*}
Here $\delta_{4m}(N,\varepsilon)$ is the admissible step size from
Lemma~\ref{lem:euler-moments}. Moreover,
\begin{equation}
\label{eq:euler-invariant-bias}
\left|
\int\phi\,\mathrm d\widehat\pi_{\mathrm{EM}}^{\varepsilon,\delta}
-
\int\phi\,\mathrm d\pi^\varepsilon
\right|
\leq
\frac{C_{\mathrm{EM}}\sqrt{C_{\mathrm{der}}}}
{\rho_\varepsilon}
\|\phi\|_{4;m}\delta .
\end{equation}
\end{theorem}
\begin{proof}
Assumption~\ref{ass:jointSES} gives the exponential derivative
estimate required in Theorem~3.2 of \cite{Crisan2021-wg}. Indeed, by
the Cauchy--Schwarz inequality,
\[
\sum_{j=1}^3
\|\nabla^j\mathcal P_t^\varepsilon\phi(s)\|_F
\leq
\sqrt{3C_{\mathrm{der}}}\,
\|\phi\|_{4;m}e^{-\rho_\varepsilon t}
\mathcal W_m(s).
\]
Moreover, Lemma~\ref{lem:euler-moments}, applied with the polynomial
orders required by the weak-error expansion, provides the uniform
moment bounds for the Euler--Maruyama approximation.  Assumption~\ref{ass:poly} gives the required smoothness and
polynomial-growth bounds on the coefficients, while
Assumption~\ref{ass:lip} gives their global Lipschitz continuity.
Therefore, for
$0<\delta\leq\delta_{4m}(N,\varepsilon)$, the hypotheses of
Theorem~3.2 in \cite{Crisan2021-wg} are satisfied. Applying that theorem and using
\[
\int_0^\infty e^{-\rho_\varepsilon t}\,\mathrm{d}t
=
\frac1{\rho_\varepsilon}
\]
gives
\[
\left|
\mathbb{E}\phi(\widehat S_n)
-
\mathbb{E}\phi(S_{t_n})
\right|
\leq
\frac{C_{\mathrm{EM}}\sqrt{C_{\mathrm{der}}}}
{\rho_\varepsilon}
\|\phi\|_{4;m}
(1+\|S_0\|^{4m})\delta ,
\]
uniformly in $n$, which proves the result.

Finally, start the Euler chain and the joint diffusion from \(S_0=0\).
Lemma~\ref{lem:euler-moments} gives convergence of the chain to
\(\widehat\pi_{\mathrm{EM}}^{\varepsilon,\delta}\), while
Assumption~\ref{ass:jointmom} gives convergence of the diffusion to
\(\pi^\varepsilon\).  Letting \(n\to\infty\) in the \gls*{uit}
estimate proves \eqref{eq:euler-invariant-bias}.
\end{proof}

\subsection{S--ROCK}
S--ROCK methods are explicit stabilized integrators designed for
stiff SDEs.  In the additive-noise setting they can attain strong
order one, while possessing a substantially larger linear
mean-square stability interval than Euler--Maruyama
\cite{S-ROCK}.  Given a step size \(\delta>0\), a stage count
\(m_{\mathrm{st}}\geq3\), and current state
\(\widehat\theta_n\in\mathbb R^{d_\theta}\),
\(\widehat Z_n=(\widehat X_n^1,\ldots,\widehat X_n^N)^\top
\in\mathbb R^{Nd_x}\), the S--ROCK update used for \gls*{spcd} is
\begin{align}
&\text{\makebox[.5\linewidth][c]{$\theta$\textit{-dynamics under \gls*{spcd},}}}\nonumber\\
K_0^\theta =& \widehat{\theta}_n\nonumber\\
K_1^\theta = & K_0^\theta + \frac{\delta}{m_{\mathrm{st}}^2N} \nabla_\theta E_N(K_0^\theta, K_0^z)\nonumber\\
K_l^\theta = & \frac{2\delta}{m_{\mathrm{st}}^2N} \nabla_\theta E_N(K_{l-1}^\theta, K_{l-1}^z) + 2 K_{l-1}^\theta - K_{l-2}^\theta\nonumber\\
K_{m_{\mathrm{st}}-1}^\theta =& \frac{2\delta}{m_{\mathrm{st}}^2N} \nabla_\theta E_N(K_{m_{\mathrm{st}}-2}^\theta, K_{m_{\mathrm{st}}-2}^z) + 2 K_{m_{\mathrm{st}}-2}^\theta - K_{m_{\mathrm{st}}-3}^\theta + \sqrt{\frac{\delta}{2N}} \widehat{W}^\theta_n\nonumber\\
\widehat{\theta}_{n+1} = K^\theta_{m_{\mathrm{st}}} =& \frac{2\delta}{m_{\mathrm{st}}^2N} \nabla_\theta E_N(K_{m_{\mathrm{st}}-1}^\theta, K_{m_{\mathrm{st}}-1}^z) + 2 K_{m_{\mathrm{st}}-1}^\theta - K_{m_{\mathrm{st}}-2}^\theta, \label{eq:srock}\\
&\text{\makebox[.5\linewidth][c]{\textit{Particle dynamics under \gls*{spcd},}}}\nonumber\\
K_0^z =& \widehat{Z}_n\nonumber\\
K_1^z = & K_0^z - \frac{\delta}{m_{\mathrm{st}}^2\varepsilon} \nabla_z E_N(K_0^\theta, K_0^z)\nonumber\\
K_l^z = & -\frac{2\delta}{m_{\mathrm{st}}^2\varepsilon} \nabla_z E_N(K_{l-1}^\theta, K_{l-1}^z) + 2 K_{l-1}^z - K_{l-2}^z\nonumber\\
K_{m_{\mathrm{st}}-1}^z =& -\frac{2\delta}{m_{\mathrm{st}}^2\varepsilon} \nabla_z E_N(K_{m_{\mathrm{st}}-2}^\theta, K_{m_{\mathrm{st}}-2}^z) + 2 K_{m_{\mathrm{st}}-2}^z - K_{m_{\mathrm{st}}-3}^z + \sqrt{\frac{\delta}{2\varepsilon}} \widehat{W}^z_n\nonumber\\
\widehat{Z}_{n+1} = K^z_{m_{\mathrm{st}}} =& -\frac{2\delta}{m_{\mathrm{st}}^2\varepsilon} \nabla_z E_N(K_{m_{\mathrm{st}}-1}^\theta, K_{m_{\mathrm{st}}-1}^z) + 2 K_{m_{\mathrm{st}}-1}^z - K_{m_{\mathrm{st}}-2}^z.\label{eq:srock_}
\end{align}
The index \(l\) ranges from \(2\) to \(m_{\mathrm{st}}-2\) in the
two generic stage recursions. We write the joint diffusion as
\[
\mathrm dS_t=f_\varepsilon(S_t)\,\mathrm dt
+\gamma_\varepsilon^{1/2}\,\mathrm dW_t,
\qquad
\gamma_\varepsilon
=
\operatorname{diag}\left(
\frac2N I_{d_\theta},
\frac2\varepsilon I_{Nd_x}
\right),
\]
and denote by \(Q_{\delta,m_{\mathrm{st}}}^\varepsilon\) the
one-step transition operator of
\eqref{eq:srock}--\eqref{eq:srock_}.  Thus
\[
Q_{\delta,m_{\mathrm{st}}}^\varepsilon h(s)
=
\mathbb E\!\left[
h\!\left(\Psi_{\delta,m_{\mathrm{st}}}^\varepsilon(s,\xi)\right)
\right],
\]
where \(\xi\) is a standard Gaussian vector.

The result below is deliberately a small-step result for fixed
\(m_{\mathrm{st}}\).  The Chebyshev recursion is motivated by its
enlarged stability interval on suitable stiff test equations, but note that the
assumptions used here do not by themselves imply that the admissible
step size for the nonlinear joint drift grows like
\(m_{\mathrm{st}}^2\). 

\begin{lemma}[One-step S--ROCK defect]
\label{lem:srock-onestep}
Suppose that Assumptions~\ref{ass:poly} and \ref{ass:lip} hold, and
let \(L_{N,\varepsilon}\) be a global Lipschitz constant of
\(f_\varepsilon\).  Fix \(m_{\mathrm{st}}\geq3\).  For every
polynomial order \(a\geq0\), there exist an order \(q\geq a\),
constants \(C_{\mathrm{loc}}<\infty\) and
\(\delta_{\mathrm{loc}}>0\), depending on
\(a,m_{\mathrm{st}},N,\varepsilon\) and the coefficient bounds but
not on \(s\) or \(\delta\), such that
\begin{equation}
\label{eq:srock-local-weak}
\left|
\left(
Q_{\delta,m_{\mathrm{st}}}^\varepsilon
-\mathcal P_\delta^\varepsilon
\right)u(s)
\right|
\leq
C_{\mathrm{loc}}\delta^2
(1+\|s\|^q)
\sum_{j=1}^3\|\nabla^j u\|_{0;a}
\end{equation}
for \(0<\delta\leq\delta_{\mathrm{loc}}\) and every
\(u\in C_a^4\).

The numerical one-step map and Euler--Maruyama map can use the same
Gaussian vector, given as
\begin{equation}
\label{eq:srock-euler-perturbation}
\Psi_{\delta,m_{\mathrm{st}}}^\varepsilon(s,\xi)
=
s+\delta f_\varepsilon(s)
+\sqrt\delta\,\gamma_\varepsilon^{1/2}\xi
+R_{\delta,m_{\mathrm{st}}}(s,\xi),
\end{equation}
where there is a constant \(C_{\mathrm{mean}}<\infty\) and, for
every \(p\geq2\), a constant \(C_p<\infty\), such that
\begin{align}
\left(
\mathbb E\|R_{\delta,m_{\mathrm{st}}}(s,\xi)\|^p
\right)^{1/p}
&\leq
C_p\delta^{3/2}(1+\|s\|),
\label{eq:srock-remainder-lp}\\
\left\|
\mathbb E R_{\delta,m_{\mathrm{st}}}(s,\xi)
\right\|
&\leq
C_{\mathrm{mean}}\delta^2(1+\|s\|^q).
\label{eq:srock-remainder-mean}
\end{align}
\end{lemma}

\begin{proof}
Put \(c=m_{\mathrm{st}}^{-2}\).  First run the recursion without its
Gaussian increment, and denote the resulting stages by
\(\overline K_\ell\).  An induction over the fixed number of stages,
using global Lipschitz continuity and
\(\|f_\varepsilon(s)\|\leq
L_{N,\varepsilon}\|s\|+\|f_\varepsilon(0)\|\), gives
\[
\max_{0\leq\ell\leq m_{\mathrm{st}}}
\|\overline K_\ell-s\|
\leq
C\delta(1+\|s\|)
\]
and
\[
\overline K_{m_{\mathrm{st}}}
=
s+\delta f_\varepsilon(s)
+\delta^2r_{\mathrm d}(s),
\qquad
\|r_{\mathrm d}(s)\|\leq C(1+\|s\|).
\]
The coefficient of the first-order drift is one because the
Chebyshev recursion gives
\(\overline K_\ell=s+\ell^2c\delta f_\varepsilon(s)+O(\delta^2)\).

In the penultimate stage set
\[
\eta
=
\frac{\sqrt\delta}{2}
\gamma_\varepsilon^{1/2}\xi.
\]
Only the last drift evaluation sees this random displacement.
Consequently the final stage has the exact representation
\begin{align}
\Psi_{\delta,m_{\mathrm{st}}}^\varepsilon(s,\xi)
=
\overline K_{m_{\mathrm{st}}}
+2\eta
+
2c\delta
\left[
f_\varepsilon(\overline K_{m_{\mathrm{st}}-1}+\eta)
-f_\varepsilon(\overline K_{m_{\mathrm{st}}-1})
\right].
\label{eq:srock-exact-final-map}
\end{align}
Since \(2\eta=\sqrt\delta\,\gamma_\varepsilon^{1/2}\xi\),
\eqref{eq:srock-euler-perturbation} follows.  The Lipschitz bound and
Gaussian moment estimates immediately give
\eqref{eq:srock-remainder-lp}.  Taylor expansion of the last
difference in \eqref{eq:srock-exact-final-map} shows that its leading
random term is
\[
c\delta^{3/2}
\nabla f_\varepsilon(\overline K_{m_{\mathrm{st}}-1})
\gamma_\varepsilon^{1/2}\xi,
\]
which is centred.  The deterministic stage defect and the Taylor
remainder are of order
\(\delta^2(1+\|s\|^q)\).  Here the possibly enlarged order \(q\)
absorbs the polynomial growth of the higher derivatives of
\(f_\varepsilon\) allowed by Assumption~\ref{ass:poly}.  This proves
\eqref{eq:srock-remainder-mean}.

It remains to prove \eqref{eq:srock-local-weak}.  Let
\(Q_\delta^{\mathrm{EM},\varepsilon}\) be the Euler--Maruyama
transition operator.  The standard additive-noise It\^o--Taylor
calculation gives
\[
\left|
\left(
Q_\delta^{\mathrm{EM},\varepsilon}
-\mathcal P_\delta^\varepsilon
\right)u(s)
\right|
\leq
C\delta^2(1+\|s\|^q)
\sum_{j=1}^3\|\nabla^j u\|_{0;a}.
\]
Although fourth derivatives occur in the separate second-order
expansions, the pure Gaussian fourth-derivative terms are identical
and cancel in their difference.  This is the same local defect used
in the additive-noise Euler estimate of
Theorem~\ref{thm:euler}.

For the difference
\(Q_{\delta,m_{\mathrm{st}}}^\varepsilon
-Q_\delta^{\mathrm{EM},\varepsilon}\), insert
\eqref{eq:srock-euler-perturbation} into Taylor's formula.
The deterministic and mean remainders are controlled by
\eqref{eq:srock-remainder-lp}--\eqref{eq:srock-remainder-mean}.
For the centred term of order \(\delta^{3/2}\), Gaussian integration
by parts supplies one further factor \(\sqrt\delta\).  The resulting
terms contain derivatives of \(u\) only up to order three and are
bounded by
\[
C\delta^2(1+\|s\|^q)
\sum_{j=1}^3\|\nabla^j u\|_{0;a}.
\]
Combining the two transition-operator differences proves
\eqref{eq:srock-local-weak}.
\end{proof}

\begin{lemma}[S--ROCK moment and invariant-measure stability]
\label{lem:srock-moments}
Fix \(m_{\mathrm{st}}\geq3\), and suppose that
Assumption~\ref{ass:numstability} holds for
\(Q_{\delta,m_{\mathrm{st}}}^\varepsilon\) at an integer order
\(q\geq2\). Then, for every \(0<\delta\leq\delta_q\),
\begin{align}
\mathbb E_s\mathcal W_q(\widehat S_n)
\leq
\frac{c_q^+}{c_q^-}
(1-c_q\delta)^n\mathcal W_q(s)
+
\frac{B_q}{c_qc_q^-}
\left[1-(1-c_q\delta)^n\right],
\qquad n\geq0.
\label{eq:srock-uniform-moments}
\end{align}

Suppose additionally that Assumptions~\ref{ass:lip} and
\ref{ass:poly} hold and that
\[
\frac{\delta L_{N,\varepsilon}}{m_{\mathrm{st}}^2}<1.
\]
Then the S--ROCK chain is Lebesgue-irreducible, aperiodic, positive
Harris recurrent, and \(\mathcal W_q\)-geometrically ergodic. It has a
unique invariant probability measure
\(\widehat\pi_{\mathrm{SR}}^{\varepsilon,\delta,m_{\mathrm{st}}}\)
satisfying
\begin{equation}
\label{eq:srock-invariant-moment}
\int \mathcal W_q(s)\,
\widehat\pi_{\mathrm{SR}}^{\varepsilon,\delta,m_{\mathrm{st}}}
(\mathrm ds)
\leq
\frac{B_q}{c_qc_q^-}.
\end{equation}
\end{lemma}

\begin{proof}
Assumption~\ref{ass:numstability}, applied to
\(Q_{\delta,m_{\mathrm{st}}}^\varepsilon\), gives
\begin{equation}\label{eq:srock-lyapunov}
Q_{\delta,m_{\mathrm{st}}}^\varepsilon
\mathcal V_q^{\varepsilon,\delta}(s)
\leq
(1-c_q\delta)\mathcal V_q^{\varepsilon,\delta}(s)
+B_q\delta.
\end{equation}
Iteration of this inequality gives
\[
\mathbb E_s\mathcal V_q^{\varepsilon,\delta}(\widehat S_n)
\leq
(1-c_q\delta)^n\mathcal V_q^{\varepsilon,\delta}(s)
+\frac{B_q}{c_q}
\left[1-(1-c_q\delta)^n\right].
\]
The comparison
\[
c_q^-\mathcal W_q
\leq
\mathcal V_q^{\varepsilon,\delta}
\leq
c_q^+\mathcal W_q
\]
now proves \eqref{eq:srock-uniform-moments}.

Under the additional hypotheses, it remains to verify irreducibility.
Conditional on \(s\), let
\[
y
=
\overline K_{m_{\mathrm{st}}-1}
+\frac{\sqrt\delta}{2}\gamma_\varepsilon^{1/2}\xi.
\]
Equation~\eqref{eq:srock-exact-final-map} shows that the final state
is an affine translation of
\[
2\left(y+\frac{\delta}{m_{\mathrm{st}}^2}
f_\varepsilon(y)\right).
\]
If \(\delta L_{N,\varepsilon}/m_{\mathrm{st}}^2<1\), the map
\[
y\longmapsto
y+\frac{\delta}{m_{\mathrm{st}}^2}f_\varepsilon(y)
\]
is a global \(C^1\)-diffeomorphism. Since \(y\) has a
non-degenerate Gaussian density, the S--ROCK transition kernel has a
strictly positive continuous density. It is therefore
Lebesgue-irreducible and aperiodic, and compact sets are small.
The Foster--Lyapunov inequality \eqref{eq:srock-lyapunov} gives
positive Harris recurrence, uniqueness of the invariant measure,
and geometric ergodicity. Integrating the inequality against the
invariant measure and using
\(\mathcal V_q^{\varepsilon,\delta}\geq c_q^-\mathcal W_q\) proves
\eqref{eq:srock-invariant-moment}.
\end{proof}

\begin{theorem}[S--ROCK uniform-in-time and invariant-measure error]
\label{thm:srockasymp}
Suppose that Assumptions~\ref{ass:jointmom},
\ref{ass:jointSES}, \ref{ass:numstability}, \ref{ass:lip}, and
\ref{ass:poly} hold, with
\ref{ass:numstability} applied to the S--ROCK transition operator, and
fix \(m_{\mathrm{st}}\geq3\).  Let \(a\geq0\) and let \(q\) be large
enough to dominate the polynomial order in
Lemma~\ref{lem:srock-onestep}.  There exist constants
\(C_{\mathrm{SR}}<\infty\) and
\(\delta_{\mathrm{SR}}=\delta_{\mathrm{SR}}
(a,m_{\mathrm{st}},N,\varepsilon)>0\), independent of \(n\), such
that, for every \(\phi\in C_a^4\),
\begin{align}
\left|
\mathbb E_s\phi(\widehat S_n)
-\mathbb E_s\phi(S_{n\delta})
\right|\leq
\frac{
C_{\mathrm{SR}}\sqrt{C_{\mathrm{der}}}
}{\rho_\varepsilon}
\|\phi\|_{4;a}
(1+\|S_0\|^q)\delta
\label{eq:srock-uit}
\end{align}
for
\[
0<\delta\leq\delta_{\mathrm{SR}},
\qquad
\rho_\varepsilon\delta\leq1,
\qquad
\frac{\delta L_{N,\varepsilon}}{m_{\mathrm{st}}^2}<1.
\]
Here \(C_{\mathrm{der}}\) is the constant in
Assumption~\ref{ass:jointSES}. Moreover,
\begin{equation}
\label{eq:srock-invariant-bias}
\left|
\int\phi\,
\mathrm d\widehat\pi_{\mathrm{SR}}^{\varepsilon,\delta,m_{\mathrm{st}}}
-
\int\phi\,\mathrm d\pi^\varepsilon
\right|
\leq
\frac{
C_{\mathrm{SR}}\sqrt{C_{\mathrm{der}}}
}{\rho_\varepsilon}
\|\phi\|_{4;a}\delta .
\end{equation}
\end{theorem}

\begin{proof}
For brevity write \(Q=Q_{\delta,m_{\mathrm{st}}}^\varepsilon\) and
\(P_t=\mathcal P_t^\varepsilon\).  The telescoping identity gives
\[
(Q_{\delta,m_{\mathrm{st}}}^\varepsilon)^n\phi-\mathcal{P}^\varepsilon_{n\delta}\phi
=
\sum_{j=0}^{n-1}
(Q_{\delta,m_{\mathrm{st}}}^\varepsilon)^j(Q_{\delta,m_{\mathrm{st}}}^\varepsilon-\mathcal{P}^\varepsilon_\delta)
\mathcal{P}^\varepsilon_{(n-j-1)\delta}\phi.
\]
Apply Lemma~\ref{lem:srock-onestep} to
\[
u_j=\mathcal{P}^\varepsilon_{(n-j-1)\delta}\phi.
\]
By \eqref{eq:joint-derivative-linear},
\[
\sum_{\ell=1}^3
\|\nabla^\ell u_j\|_{0;a}
\leq
\sqrt{3C_{\mathrm{der}}}\,
\|\phi\|_{4;a}
e^{-\rho_\varepsilon(n-j-1)\delta}.
\]
Lemma~\ref{lem:srock-moments}, at the polynomial order required by
the local defect, therefore yields
\begin{align*}
\left|
(Q_{\delta,m_{\mathrm{st}}}^\varepsilon)^n\phi(s)-\mathcal{P}^\varepsilon_{n\delta}\phi(s)
\right|
&\leq
C_{\mathrm{SR}}\sqrt{C_{\mathrm{der}}}\,
\|\phi\|_{4;a}(1+\|s\|^q)\delta^2
\sum_{j=0}^{n-1}
e^{-\rho_\varepsilon(n-j-1)\delta}\\
&\leq
C_{\mathrm{SR}}\sqrt{C_{\mathrm{der}}}\,
\|\phi\|_{4;a}(1+\|s\|^q)
\frac{\delta^2}{1-e^{-\rho_\varepsilon\delta}}.
\end{align*}
For \(\rho_\varepsilon\delta\leq1\),
\[
1-e^{-\rho_\varepsilon\delta}
\geq\frac{\rho_\varepsilon\delta}{2}.
\]
Absorbing the numerical factor into
\(C_{\mathrm{SR}}\) proves \eqref{eq:srock-uit}.

Lemma~\ref{lem:srock-moments} gives geometric convergence of the
chain to
\(\widehat\pi_{\mathrm{SR}}^{\varepsilon,\delta,m_{\mathrm{st}}}\),
while Assumption~\ref{ass:jointmom} gives convergence of the joint
diffusion to \(\pi^\varepsilon\).
Letting \(n\to\infty\) in \eqref{eq:srock-uit} proves
\eqref{eq:srock-invariant-bias}.
\end{proof}

\begin{example}[Numerical stability of the linearly-tilted family]
\label{ex:softplus-numerical}
We return to the model of Example~\ref{ex:softplus} and verify
Assumption~\ref{ass:numstability} for both Euler--Maruyama and
S--ROCK.  Recall the continuous-time Lyapunov function
\(\mathscr V_\varepsilon\) from
\eqref{eq:softplus-joint-lyapunov}.  For every required \(q\geq2\),
set
\[
\mathcal V_q^{\varepsilon,\delta}
:=
\mathcal V_q^\varepsilon
=
\mathscr V_\varepsilon^{q/2}.
\]
For fixed \(\varepsilon\), this function is smooth and comparable to
\(\mathcal W_q\).  The calculation leading to
\eqref{eq:softplus-joint-drift}, applied to its powers, gives
\[
\mathcal G_\varepsilon\mathcal V_q^\varepsilon
\leq
-a_{q,\varepsilon}\mathcal V_q^\varepsilon
+b_{q,\varepsilon}
\]
for suitable positive constants $a_{q,\varepsilon}$ and $b_{q,\varepsilon}$.

Because \(\mathscr V_\varepsilon\geq1\) is quadratic, each derivative
of \(\mathcal V_q^\varepsilon=\mathscr V_\varepsilon^{q/2}\) required
below is bounded by a constant times
\(\mathcal V_q^\varepsilon\).  More precisely, after increasing the
constant for the finitely many low orders,
\[
\|\nabla^j\mathcal V_q^\varepsilon(s)\|_F
\leq
C_{q,N,\varepsilon}\mathcal V_q^\varepsilon(s),
\qquad 0\leq j\leq4.
\]
Moreover, global Lipschitz continuity of \(f_\varepsilon\) and
Gaussian moment estimates imply that every intermediate point in
the one-step Taylor formula has moments bounded by
\(C_{q,N,\varepsilon}\mathcal V_q^\varepsilon(s)\), uniformly for
sufficiently small \(\delta\).

For Euler--Maruyama, the standard additive-noise one-step expansion
gives
\[
Q_\delta^{\mathrm{EM},\varepsilon}\mathcal V_q^\varepsilon
=
\mathcal V_q^\varepsilon
+\delta\mathcal G_\varepsilon\mathcal V_q^\varepsilon
+\mathcal R_{\delta,q}^{\mathrm{EM},\varepsilon},
\qquad
\left|
\mathcal R_{\delta,q}^{\mathrm{EM},\varepsilon}
\right|
\leq
C_{q,N,\varepsilon}^{\mathrm{EM}}\delta^2
\mathcal V_q^\varepsilon.
\]
For S--ROCK, the exact one-step representation
\eqref{eq:srock-exact-final-map} and the estimates in
Lemma~\ref{lem:srock-onestep} give, for fixed
\(m_{\mathrm{st}}\geq3\),
\[
Q_{\delta,m_{\mathrm{st}}}^{\varepsilon}
\mathcal V_q^\varepsilon
=
\mathcal V_q^\varepsilon
+\delta\mathcal G_\varepsilon\mathcal V_q^\varepsilon
+\mathcal R_{\delta,q}^{\mathrm{SR},\varepsilon},
\qquad
\left|
\mathcal R_{\delta,q}^{\mathrm{SR},\varepsilon}
\right|
\leq
C_{q,N,\varepsilon,m_{\mathrm{st}}}^{\mathrm{SR}}\delta^2
\mathcal V_q^\varepsilon.
\]
For Euler--Maruyama this follows by Taylor expansion of
\(\mathcal V_q^\varepsilon\) around \(s\), using the centred Gaussian
moments through order four.  For S--ROCK, substitute
\eqref{eq:srock-euler-perturbation} into the same Taylor formula.
The \(L^p\) and mean estimates
\eqref{eq:srock-remainder-lp}--\eqref{eq:srock-remainder-mean}
control the additional terms, while the polynomial order in those
estimates is absorbed by the displayed derivative bounds for
\(\mathcal V_q^\varepsilon\).  This proves the two stated remainder
estimates rather than treating them as additional assumptions.

The constants need not be uniform in \(N\) or \(\varepsilon\), in
agreement with Assumption~\ref{ass:numstability}.  Consequently,
after reducing the method-specific step-size bounds, either
transition operator satisfies
\[
Q_\delta^\varepsilon\mathcal V_q^\varepsilon
\leq
\left(1-\frac{a_{q,\varepsilon}}2\delta\right)
\mathcal V_q^\varepsilon
+b_{q,\varepsilon}\delta.
\]
This is Assumption~\ref{ass:numstability}.  In the S--ROCK case the
step size is additionally chosen so that the stage-map invertibility
condition in Lemma~\ref{lem:srock-moments} holds.  Combined with the
continuous-time verification in Example~\ref{ex:softplus}, this
shows that the model is covered by
Theorems~\ref{thm:euler}, \ref{thm:srockasymp}, and
\ref{thm:numerical-mle} under their stated step-size restrictions.
\end{example}

\subsection{Combined invariant-measure error}

We can now combine the method-specific numerical bias with the
averaging and concentration estimates to obtain bias estimates for the proposed algorithms.  For
\(\mathsf M\in\{\mathrm{EM},\mathrm{SR}\}\), define
\[
\left(
\widehat\pi_{\mathsf M}^{\varepsilon,\delta},
C_{\mathsf M}
\right)
=
\begin{cases}
\left(
\widehat\pi_{\mathrm{EM}}^{\varepsilon,\delta},
C_{\mathrm{EM}}
\right),
&\mathsf M=\mathrm{EM},\\[1mm]
\left(
\widehat\pi_{\mathrm{SR}}^{\varepsilon,\delta,m_{\mathrm{st}}},
C_{\mathrm{SR}}
\right),
&\mathsf M=\mathrm{SR},
\end{cases}
\]
where \(m_{\mathrm{st}}\geq3\) is fixed in the S--ROCK case.

\begin{theorem}[Numerical invariant-measure approximation of the MLE]
\label{thm:numerical-mle}
Suppose that the assumptions of
Theorems~\ref{thm:statmeasure} and \ref{thm:averror} hold.
Here \(m_\theta\) denotes the polynomial regularity order from
Assumption~\ref{ass:poly}, while \(m_V\) denotes the local growth
exponent in Theorem~\ref{thm:statmeasure}. For
\(\mathsf M=\mathrm{EM}\), suppose additionally that the assumptions
of Theorem~\ref{thm:euler} hold with \(m=m_\theta\), and let
\[
0<\delta\leq\delta_{4m_\theta}(N,\varepsilon),
\qquad
\rho_\varepsilon\delta\leq1.
\]
For \(\mathsf M=\mathrm{SR}\), suppose instead that the assumptions
of Theorem~\ref{thm:srockasymp} hold with \(a=m_\theta\), and let
\[
0<\delta\leq\delta_{\mathrm{SR}},
\qquad
\rho_\varepsilon\delta\leq1,
\qquad
\frac{\delta L_{N,\varepsilon}}{m_{\mathrm{st}}^2}<1.
\]
If \(N\geq N_0\), then, for every
\(\phi\in C_{m_\theta}^4(\mathbb R^{d_\theta})\),
\begin{align}
\left|
\int\phi(\theta)\,
\widehat\pi_{\mathsf M}^{\varepsilon,\delta}
(\mathrm d\theta,\mathrm dz)
-\phi(\theta_\star)
\right|
&\leq
\underbrace{
\varepsilon C_{\mathrm{av}}
\|\phi\|_{4;m_\theta}
}_{\textnormal{averaging error}}
+
\underbrace{
\frac{
C_{\mathsf M}\sqrt{C_{\mathrm{der}}}
}{\rho_\varepsilon}
\|\phi\|_{4;m_\theta}\delta
}_{\textnormal{numerical error}}
\nonumber\\
&\qquad\quad+
\underbrace{
\sqrt3\,C_{\mathrm{conc}}\,
\|\nabla_\theta\phi\|_{0;m_\theta}
\left(
1+\bar\gamma_{2m_\theta}^{1/2}
+\|\theta_\star\|^{m_\theta}
\right)
N^{-1/m_V}
}_{\textnormal{MLE concentration error}}.
\label{eq:numerical-invariant-error}
\end{align}
Here \(C_{\mathrm{av}}\) and \(C_{\mathrm{conc}}\) are the constants
in Theorems~\ref{thm:averror} and \ref{thm:statmeasure},
respectively, and \(C_{\mathrm{der}}\) is the constant in
Assumption~\ref{ass:jointSES}.
\end{theorem}

\begin{proof}
Extend \(\phi\) to the joint state space by setting
\(\widetilde\phi(\theta,z)=\phi(\theta)\).  Insert the invariant
measures of the joint and averaged diffusions:
\begin{align*}
\left|
\int\widetilde\phi\,
\mathrm d\widehat\pi_{\mathsf M}^{\varepsilon,\delta}
-\phi(\theta_\star)
\right|\leq
\left|
\int\widetilde\phi\,
\mathrm d\widehat\pi_{\mathsf M}^{\varepsilon,\delta}
-
\int\widetilde\phi\,\mathrm d\pi^\varepsilon
\right|
+
\left|
\int\widetilde\phi\,\mathrm d\pi^\varepsilon
-
\int\phi\,\mathrm d\pi_N^{\mathrm{av}}
\right|+
\left|
\int\phi\,\mathrm d\pi_N^{\mathrm{av}}
-\phi(\theta_\star)
\right|.
\end{align*}
The first term is bounded by \eqref{eq:euler-invariant-bias} when
\(\mathsf M=\mathrm{EM}\), and by
\eqref{eq:srock-invariant-bias} when \(\mathsf M=\mathrm{SR}\).

For the second term, apply Theorem~\ref{thm:averror} at
\((\theta,z)=(0,0)\) and let \(t\to\infty\).  Since
the polynomial orders can be increased without weakening the
estimates, we take \(q_\theta,q_z>0\). Hence
\(\mathcal W_{q_\theta,q_z}(0,0)=1\). Assumption~\ref{ass:jointmom}
and Theorem~\ref{thm:avconv} give
\[
\left|
\int\widetilde\phi\,\mathrm d\pi^\varepsilon
-
\int\phi\,\mathrm d\pi_N^{\mathrm{av}}
\right|
\leq
\varepsilon C_{\mathrm{av}}
\|\phi\|_{4;m_\theta}.
\]

Finally, Lemma~\ref{lem:Equiv}, the invariant moment bound from
Lemma~\ref{lem:boundedmomav}, and Theorem~\ref{thm:statmeasure}
give
\begin{align*}
\left|
\int\phi\,\mathrm d\pi_N^{\mathrm{av}}
-\phi(\theta_\star)
\right|
\leq
\sqrt3\,C_{\mathrm{conc}}\,
\|\nabla_\theta\phi\|_{0;m_\theta}
\left(
1+\bar\gamma_{2m_\theta}^{1/2}
+\|\theta_\star\|^{m_\theta}
    \right)N^{-1/m_V}.
\end{align*}
Combining the three estimates proves
\eqref{eq:numerical-invariant-error}.
\end{proof}

\section{Experiments}\label{sec:experiments}
To verify the efficacy of the proposed algorithms we conduct a series of numerical simulations to compare the proposed multiscale system \eqref{eq:sde}, implemented via the Euler--Maruyama integrator, denoted as \gls*{spcdem}, and via the \gls*{srock} integrator, denoted by \gls*{spcd}. We begin by making these comparisons on a two-dimensional sampling problem from a banana density, followed by the more complex problem of generating handwritten digits from the MNIST dataset.

For the experiments only, we regard the data and particle state space
as a fixed compact set \(\mathsf X\subset\mathbb R^{d_x}\).  Thus the
experimental EBM is interpreted as
\[
p_\theta^{\mathsf X}(x)
=
\frac{e^{-E(\theta,x)}}{Z_\theta^{\mathsf X}}
\mathbbm 1_{\mathsf X}(x),
\qquad
Z_\theta^{\mathsf X}
=
\int_{\mathsf X}e^{-E(\theta,x)}\,\mathrm dx.
\]
Since the neural-network energies used below are continuous,
\(Z_\theta^{\mathsf X}<\infty\) for every fixed \(\theta\).  This
compact-domain convention avoids imposing a separate coercive base
potential on the experimental neural networks.  The theoretical
results above remain formulated for the unconstrained dynamics on
\(\mathbb R^{d_x}\).

\begin{figure}[b!]
    \centering
    \includegraphics[width=\linewidth]{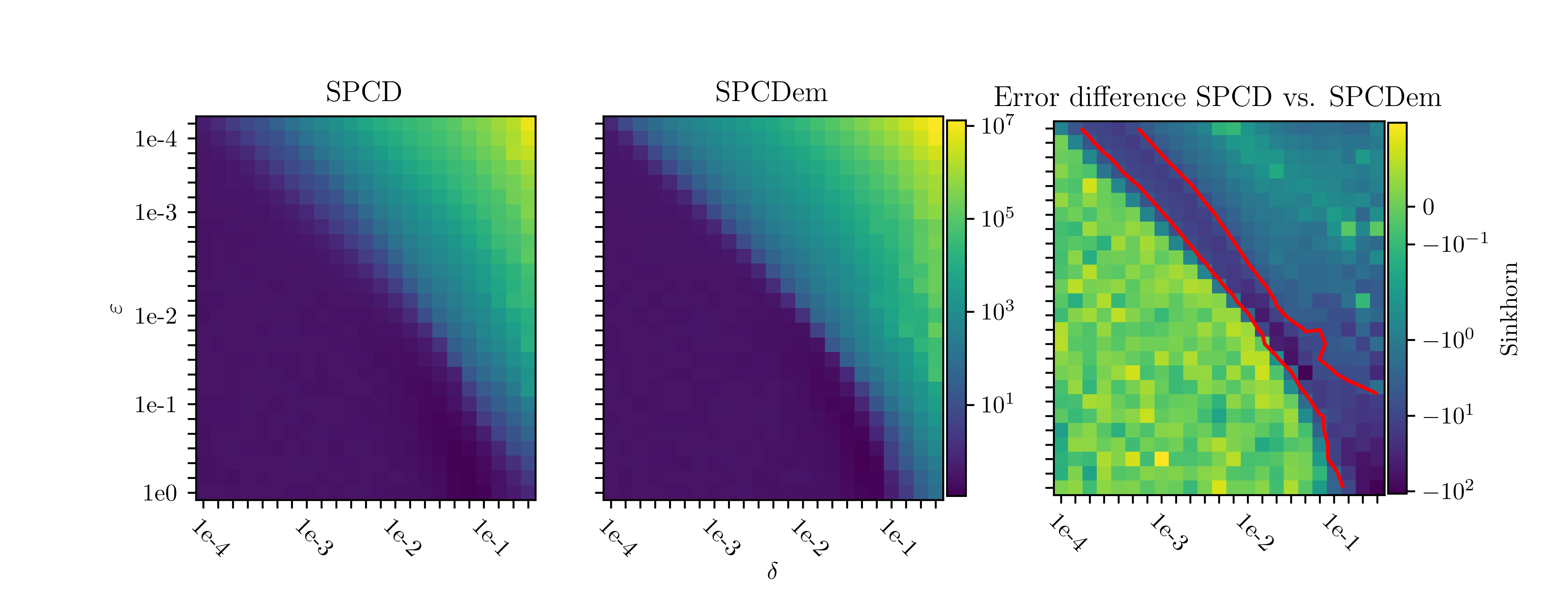}
    \caption{A comparison between the performance of \gls*{spcd} and \gls*{spcdem} over the hyper-parameters $\delta$ and $\varepsilon$. The comparison for \gls*{spcd} and \gls*{spcdem} is the difference between the Sinkhorn error obtained by \gls*{spcd} and the error obtained by \gls*{spcdem}, normalised by the \gls*{spcd} error (negative sign means \gls*{spcd} error is smaller than \gls*{spcdem} error). We highlight in red the region in which \gls*{spcd} attains an error smaller than 20\% of the \gls*{spcdem} error. For full experimental details see the \hyperref[appn]{Appendix}.}
    \label{fig:boxplot}
\end{figure}

\subsection{Synthetic dataset}\label{sec:banana}
We begin by considering a simple distribution on a fixed compact box
\(\mathsf X_{\mathrm{syn}}\subset\mathbb R^2\), chosen large enough to
contain the samples used in the experiment.  Consider the restriction
to \(\mathsf X_{\mathrm{syn}}\) of a variation on the classical banana
density, where \(x=(x_1,x_2)\),
\begin{equation*}
p(\mathrm{d}x) \propto
\exp\left(-\frac{1}{2}
\left(x_1^2+(2x_2-x_1^2)^2\right)\right)
\mathbbm 1_{\mathsf X_{\mathrm{syn}}}(x)\,\mathrm{d}x.
\end{equation*}
We choose this variant because proposals can be sampled directly via a
simple transformation of Gaussian variables: \(X_1=Y_1\) and
\(X_2=(Y_2+Y_1^2)/2\), where \(Y_1,Y_2\) are independent standard
Gaussians.  Retaining the proposals that lie in
\(\mathsf X_{\mathrm{syn}}\) gives exact samples from the restricted
target. Our goal in this setting will be to learn the underlying distribution which we model with a neural-network $E(\theta, x)$ (more details are given in the \hyperref[appn]{Appendix}). As we have access to the true distribution and accurate samples, we will use the Sinkhorn distance to evaluate relative performance, as it enables reliable and scalable numerical implementation of an optimal transport metric \cite{feydy2019}, by using entropic regularisation. Indeed it is shown in \cite{feydy2019} that this loss is non-negative, definite and metrises the convergence in law.

\begin{figure}[b!]
    \centering
    \begin{subfigure}[t]{.25\linewidth}
    \centering
    \includegraphics[width=\linewidth]{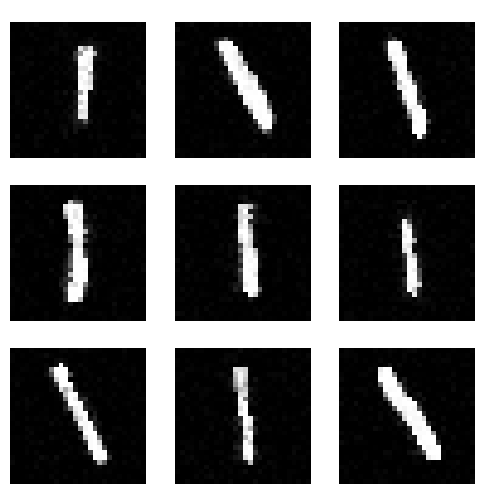}
    \end{subfigure}
    \begin{subfigure}[t]{.25\linewidth}
    \includegraphics[width=\linewidth]{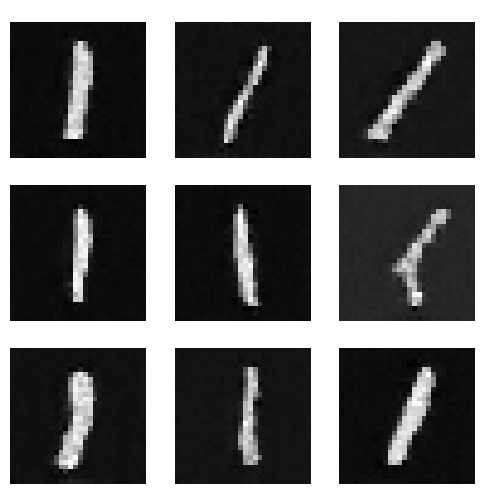}
    \end{subfigure}\\
    \begin{subfigure}[t]{.25\linewidth}
    \centering
    \includegraphics[width=\linewidth]{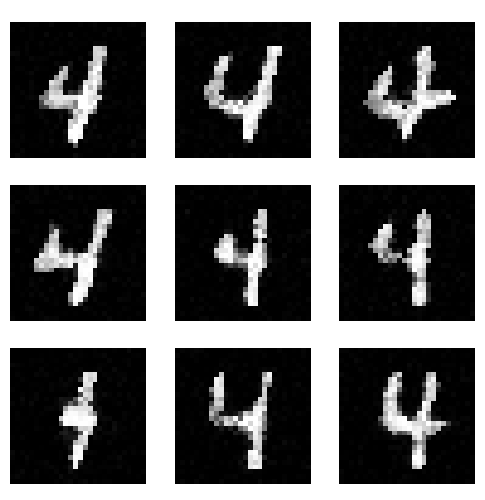}
    \caption{Samples via \gls*{spcd}}
    \end{subfigure}
    \begin{subfigure}[t]{.25\linewidth}
    \includegraphics[width=\linewidth]{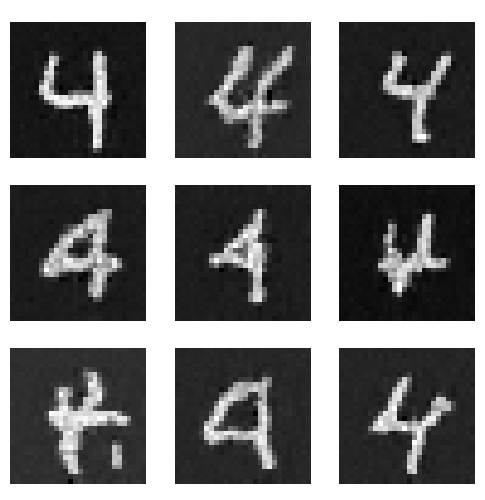}
    \caption{Samples via \gls*{spcdem}}
    \end{subfigure}
    \caption{The samples obtained after 60 outer training epochs with
    matched gradient-evaluation budgets: for each \gls*{spcd} update
    with \(m_{\mathrm{st}}=3\), three \gls*{spcdem} updates are used
    (full details are given in the \hyperref[appn]{Appendix}). In the
    top row the algorithms are trained on images of ones, while in the
    second row they are trained on images of the digit 4. The displayed
    samples are selected randomly from the generated samples.}
    \label{fig:MNIST}
\end{figure}

In Fig.~\ref{fig:boxplot} we explore the stability domain of the
proposed algorithms in \(\varepsilon\) and \(\delta\) by running
\gls*{spcd} and \gls*{spcdem} for a fixed compute time and evaluating
the Sinkhorn distance.  For fixed \(\delta\), the error initially
improves as \(\varepsilon\) is decreased, until numerical instability
from the increasingly stiff fast dynamics dominates.  Decreasing
\(\delta\) moves this transition to smaller values of
\(\varepsilon\).  This behaviour is consistent with the structure of
Theorem~\ref{thm:numerical-mle}, which contains an
\(O(\varepsilon)\) averaging term and an \(O(\delta)\) numerical
term whose constant depends on \(N\), \(\varepsilon\), and the chosen
integrator.  The present estimates do not track these constants
sharply enough to assert a universal \(1/\varepsilon^2\) scaling or
an \(m_{\mathrm{st}}^2\) enlargement of the nonlinear stability
domain.  The enlarged stability region observed for S--ROCK in this
experiment is therefore an empirical property of the tested
problems, consistent with the Chebyshev stability mechanism that
motivates the method.

Regarding the relative performance of \gls*{spcd} and \gls*{spcdem}, we observe in these experiments that in the transition region, outlined in red in Fig.~\ref{fig:boxplot}, \gls*{spcd} achieves a five-fold smaller Sinkhorn error than \gls*{spcdem}. This indicates the relative strength of the \gls*{spcd} algorithm when the continuous-time diffusion \eqref{eq:sde} is ``stiffer''. Each time step of the \gls*{spcd} routine requires \(m_{\mathrm{st}}\) gradient evaluations; this dampens the effective stiffness of the drift, and indeed we observe an extended stability domain as \(m_{\mathrm{st}}\) is increased. However, we identify \(m_{\mathrm{st}}=3\) as a good choice balancing cost and accuracy when compared over a fixed wall-clock time (see the \hyperref[appn]{Appendix} for the influence of \(m_{\mathrm{st}}\) on stability).

\subsection{MNIST generation}

For a more challenging demonstration of the efficacy of the proposed algorithm, we will consider the problem of generating image samples; specifically, handwritten digits based on the MNIST dataset. In this case a convolutional neural network (CNN) is used to model $E(\theta, x)$ and the particles lie in the compact pixel cube
\[
\mathsf X_{\mathrm{MNIST}}=[0,1]^{28\times28}.
\]
For simplicity we will focus on identifying the \gls*{mle} $\theta_\star$ for $\{y_i\}_{i=1}^M$ sampled from characters depicting ones and fours. Note further, that for computational efficiency and added stability, we will batch the MNIST dataset and iterate through the batches for each of the time increments evaluated by the numerical integrator. 

For computational tractability in the high-dimensional image setting, we use $\varepsilon=1$. The focus is on demonstrating the stability advantage of \gls*{spcd} rather than the averaging accuracy; the $O(\varepsilon)$ averaging error is non-negligible but the experiment is designed to show that \gls*{srock}'s improved stability translates to practical improvements. Under the matched gradient-evaluation budget described in the Appendix, the \gls*{spcdem} algorithm appears unable to produce artefact-free samples as consistently as the \gls*{spcd} scheme. Figure~\ref{fig:MNIST} shows samples obtained after training both routines with the same model.

\section{Discussion}
In this paper we introduced a continuous-time, diffusion-based
framework for the analysis of \gls*{pcd}-like schemes. Through this
lens, we established weak \gls*{uit} bounds for Langevin-based
schemes, building on \cite{Crisan_Ottobre_2024}. The resulting error
decomposition compares a numerical invariant measure with the joint
multiscale invariant measure, the averaged Gibbs measure
\(\pi_N^{\mathrm{av}}\), and finally the empirical \gls*{mle} point
\(\theta_\star\). We also introduced an \gls*{srock}
discretisation and established uniform-in-time and invariant-measure
error bounds under the stated semigroup and numerical Lyapunov
assumptions.

The experiments indicate empirical stability beyond the regime currently covered by the theory, thus, future work will seek  weaker weighted derivative-decay conditions. The numerical experiments in Sec.~\ref{sec:experiments} provide
empirical support for the theoretical framework. In particular,
Fig.~\ref{fig:boxplot} is consistent with the two-regime structure
described by Theorem~\ref{thm:numerical-mle}. We further demonstrate
in Figs.~\ref{fig:boxplot} and \ref{fig:MNIST} that the stiffness
attenuation of the \gls*{spcd} algorithm translates to practical
improvements in high-dimensional settings.

This paper builds on a growing body of works which exploit multiscale dynamics for sampling and optimisation, particularly relevant to developing novel approaches in machine learning and computational statistics. We believe that the use of stabilised numerical integrators, as presented in this paper, further extend the applicability of such approaches, and hope that this framework will continue to motivate the exploration of such schemes.

\begin{appendix}\label{appn}

\section*{Appendix}

We present the details of the experiments run in Sec.~\ref{sec:experiments}.

\subsection*{Synthetic dataset model architecture}
For the synthetic data experiment we use a neural network architecture for the energy function $E(\theta, x)$. We use five fully connected layers with latent dimension 128 and tanh activations, with no activation on the scalar output. 

The energy is interpreted on the compact domain
\(\mathsf X_{\mathrm{syn}}\) introduced in
Sec.~\ref{sec:banana}.  Consequently, although a finite tanh network
need not be coercive as a function on \(\mathbb R^2\), its
normalising constant over \(\mathsf X_{\mathrm{syn}}\) is finite.

The particles are initialised as $\mathcal{N}(0, I)$, with $\theta$ as the matching parametrisation and we set $M=N=2000$. As discussed in Sec.~\ref{sec:banana}, we set \(m_{\mathrm{st}}=3\), which strikes a good balance between computational cost and stability. We compare different stage counts in Fig.~\ref{fig:m}, where we observe only marginal gains in stability beyond \(m_{\mathrm{st}}=3\), confirming our choice.

\begin{figure}[t!]
    \centering
    \includegraphics[width=0.4\linewidth]{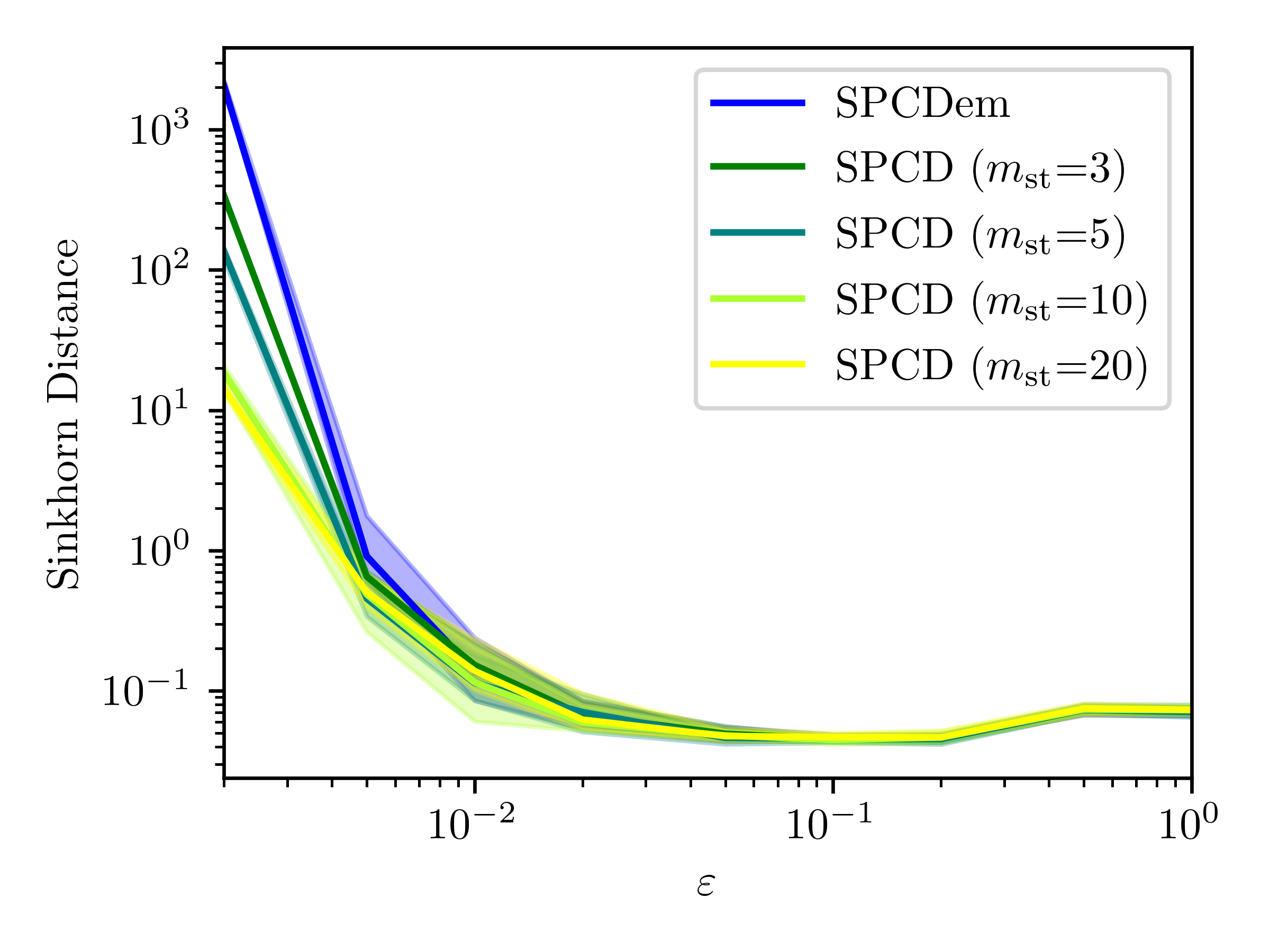}
    \caption{A comparison of the Sinkhorn errors of \gls*{spcd} for \(m_{\mathrm{st}}=3,5,10,20\) and \gls*{spcdem} against $\varepsilon$ on the synthetic dataset experiment from Sec.~\ref{sec:banana}. For this experiment we set $\delta=0.001$ and fix the number of gradient evaluations to 10000. We display the average and interquartile range of 100 runs.}
    \label{fig:m}
\end{figure}

\subsection*{MNIST generation model architecture}
To parametrise the \gls*{ebm}'s potential function for the MNIST dataset, we use a Convolutional Neural Network (CNN). This model processes greyscale images in the compact pixel cube
\[
\mathsf X_{\mathrm{MNIST}}=[0,1]^{28\times28}
\]
through a series of convolutional and fully connected layers, with Swish activation functions and spectral normalisation. We give the exact model architecture in Fig.~\ref{fig:model}.  Compactness of the pixel cube guarantees that the associated experimental normalising constant is finite without requiring the CNN energy to be coercive on all of \(\mathbb R^{28\times28}\).

\begin{figure}[b!]
    \centering
    \includegraphics[width=.5\linewidth]{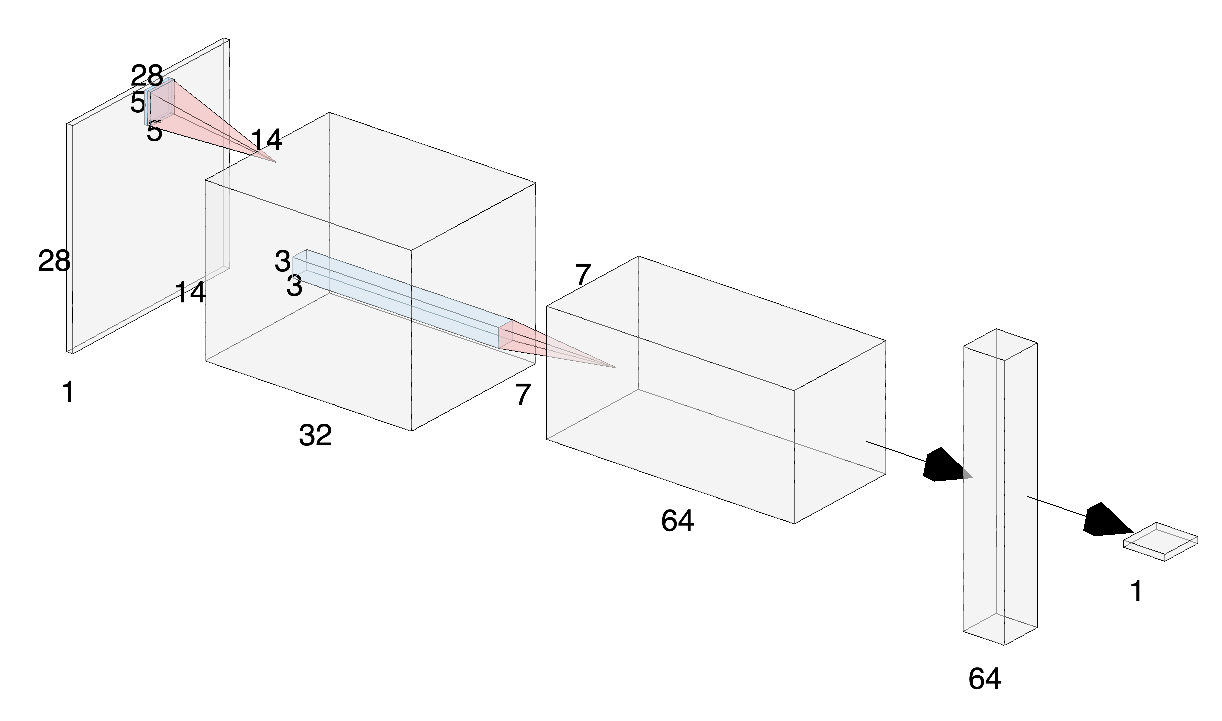}
    \caption{The model structure of $E(\theta, x)$ for the MNIST experiment, where the pyramids represent convolutions and the vectors represent fully connected linear layers. On the left we have a realisation of $x$ and on the right the scalar output of $E(\theta, x)$. We note that between convolutions we apply spectral normalisation and Swish activations (the Swish activation is given as $x\mapsto x\sigma(x)$, with $\sigma$ corresponding to the sigmoid activation). For the linear transformations we similarly normalise and apply Swish activations, except for the last layer.}
    \label{fig:model}
\end{figure}

We note that the learning of this model is performed via the \gls*{spcd} and \gls*{spcdem} algorithms, where $\varepsilon=1$, $\delta = 10^{-4}$, with batch-wise updates using 64 data points and 64 particles. With this partition of the dataset, there are 92 batches per outer epoch, and the experiment is run for 60 outer epochs. The \gls*{spcd} algorithm uses \(m_{\mathrm{st}}=3\); consequently, for every \gls*{spcd} update, we perform three \gls*{spcdem} updates. This equalises the number of gradient evaluations across the two methods, but not the number of numerical updates.

\end{appendix}

\section*{Acknowledgements and Funding}
The authors would like to thank Iain Souttar for his insightful comments and encouragement.

The first author is supported by the EPSRC through the Modern Statistics and Statistical Machine Learning (StatML)
CDT programme, grant no. EP/S023151/1.

\bibliographystyle{plain}
\bibliography{ref}

\end{document}